\documentclass[lettersize,journal]{IEEEtran}
\usepackage{amsmath,amsfonts}
\usepackage{algorithmic}
\usepackage{algorithm}
\usepackage{hyperref}
\usepackage{url}
\usepackage{array}
\usepackage{paralist}
\usepackage{amssymb}
\usepackage{makecell}
\usepackage{subcaption}
\usepackage{amsthm}
\newtheorem{innercustomthm}{Proposition}
\newenvironment{customthm}[1]
{\renewcommand\theinnercustomthm{#1}\innercustomthm}
{\endinnercustomthm}
\newtheorem{manualtheoreminner}{Assumption}
\newenvironment{manualtheorem}[1]{%
  \IfBlankTF{#1}
    {\renewcommand{\themanualtheoreminner}{\unskip}}
    {\renewcommand\themanualtheoreminner{#1}}%
  \manualtheoreminner
}{\endmanualtheoreminner}

\newtheorem{lemma}{Lemma}[section]
\newtheorem{assumption}{Assumption}[section]
\newtheorem{proposition}{Proposition}[section]
\usepackage{mathtools}
\usepackage{cases}
\usepackage{multirow}
\newcommand*{\transpose}{\bgroup\@transpose}
\usepackage[misc]{ifsym}
\newcolumntype{?}{!{\vrule width 1pt}}
\usepackage{epsfig} 
\usepackage{epstopdf} 
\usepackage{textcomp}
\usepackage{stfloats}
\usepackage{url}
\usepackage{verbatim}
\usepackage{graphicx}
\usepackage{cite}
\newcommand{\wass}{\mathcal{W}}

\newcommand{\rd}{\mathrm{d}}
\newcommand{\vertiii}[1]{{\left\vert\kern-0.15ex\left\vert\kern-0.15ex\left\vert #1 
		\right\vert\kern-0.15ex\right\vert\kern-0.15ex\right\vert}}

\begin{document}
	
	\title{Normalizing Flow-based Differentiable Particle Filters}
	
	\author{Xiongjie Chen and Yunpeng Li
		\thanks{X. Chen and Y. Li are with the Center for Oral, Clinical \& Translational Sciences, Faculty of Dentistry, Oral \& Craniofacial Sciences, King's College London, London SE1 9RT, U.K. (email: xiongjie.chen@kcl.ac.uk; yunpeng.li@kcl.ac.uk). }
	}
	
	\markboth{IEEE TRANSACTIONS ON SIGNAL PROCESSING}%
	{Chen \MakeLowercase{\textit{et al.}}: Normalizing Flow-based Differentiable Particle Filters}
	
	
	\maketitle
	
	\begin{abstract}
		Recently, there has been a surge of interest in incorporating neural networks into particle filters, e.g. differentiable particle filters, to perform joint sequential state estimation and model learning for non-linear non-Gaussian state-space models in complex environments. Existing differentiable particle filters are mostly constructed with vanilla neural networks that do not allow density estimation. As a result, they are either restricted to a bootstrap particle filtering framework or employ predefined distribution families (e.g. Gaussian distributions), limiting their performance in more complex real-world scenarios. In this paper we present a differentiable particle filtering framework that uses (conditional) normalizing flows to build its dynamic model, proposal distribution, and measurement model. This not only enables valid probability densities but also allows the proposed method to adaptively learn these modules in a flexible way, without being restricted to predefined distribution families. We derive the theoretical properties of the proposed filters and evaluate the proposed normalizing flow-based differentiable particle filters' performance through a series of numerical experiments. 
	\end{abstract}
	
	\begin{IEEEkeywords}
		Sequential Monte Carlo, Differentiable Particle Filters, Normalizing Flows, Parameter Estimation, Machine Learning.
	\end{IEEEkeywords}
	\section{Introduction}
	
	Particle filters, also known as sequential Monte Carlo (SMC) methods, are a class of importance sampling-based methods developed for performing sequential state estimation tasks in state-space models~\cite{doucet2009tutorial,gordon1993novel,djuric2003particle}. Because particle filters do not assume the linearity or Gaussianity on the considered state-space model and produce consistent estimators~\cite{del2000branching,crisan2002survey,del1998measure,elvira2017adapting,elvira2021performance}, they are particularly suitable for solving non-linear non-Gaussian filtering problems and have been widely adopted in various domains~\cite{giremus2007particle,zhang2017multi,creal2012survey,ma2019discriminative}. 
 
    In cases where the state-space model of interest is known, particle filters can provide reliable approximations to posterior distributions of latent states. Since the celebrated bootstrap particle filter (BPF) was proposed~\cite{gordon1993novel}, a series of particle filtering algorithms have been developed. For instance, the auxiliary particle filter (APF) improves its sampling efficiency by employing an auxiliary variable, such that the particles that are more compatible with the next observation have higher chances of survival~\cite{pitt1999filtering,elvira2019elucidating,branchini2021optimized}. The variance of the Monte Carlo estimates is reduced in the Rao-Blackwellized particle filter (RBPF) by marginalizing out some latent states analytically~\cite{doucet2000rao,de2002rao}. We refer readers to~\cite{doucet2001sequential,godsill2019particle,doucet2009tutorial} for more detailed discussions of the above and several other variants of particle filtering methods.
	
    In many real-world applications, parameters in the state-space model of interest are often unknown. Several techniques have been proposed to estimate the parameters in the state-space model~\cite{kantas2015particle,kantas2009overview}, among which maximum likelihood estimation methods and Bayesian estimation methods are the two main approaches in this direction. In maximum likelihood methods~\cite{hurzeler2001approximating,ionides2006inference,malik2011particle}, estimates of the unknown parameters are obtained by searching for parameters that maximize the likelihood of observations given the estimate. In contrast, Bayesian estimation methods aim to estimate the posterior of parameters given observations through a specified prior distribution on the estimated parameters and the conditional likelihood of observations given parameters~\cite{olsson2008sequential,poyiadjis2011particle}. In addition, depending on whether the observations are from fixed datasets or streaming dataflows, both Bayesian estimation and maximum likelihood estimation methods can be further divided into off-line and on-line methods~\cite{andrieu2010particle,chopin2013smc2,perez2018probabilistic,crisan2002survey,lindsten2014particle}. 
	
    These parameter estimation methods have shown their effectiveness in certain scenarios, e.g. where the structure or a part of the parameters of the state-space model is known. Complex real-world cases require the learning of a full, complex state-space model from data. Recently, an emerging class of particle filters, often named differentiable particle filters (DPFs)~\cite{chen2023overview,jonschkowski18,karkus2018particle,corenflos2021differentiable,scibior2021differentiable,kloss2021train,zhu2020towards}, has received a surge of interest. Compared with classical parameter estimation techniques developed for particle filters, differentiable particle filters often make much less restrictive assumptions about the considered state-space model. Such flexibility makes them a promising tool for solving filtering tasks with complex high-dimensional environments, where the observations can be high-dimensional unstructured data such as images~\cite{jonschkowski18,karkus2018particle,corenflos2021differentiable}.
	
    Components of differentiable particle filters, including dynamic models, measurement models, and proposal distributions, are mostly constructed with neural networks and optimized by minimizing a loss function via gradient descent~\cite{jonschkowski18,karkus2018particle,le2018auto,naesseth2018variational,maddison2017filtering}. 
    Due to the discrete nature of the standard multinomial resampling, different resampling strategies have been developed to estimate the gradient of neural network parameters w.r.t. different loss functions~\cite{jonschkowski18, karkus2018particle,scibior2021differentiable,zhu2020towards,gama2023unsupervised,revach2022unsupervised}.  It was investigated in~\cite{kloss2021train} the impact of different design choices of dynamic models, measurement models, noise models, loss functions, and resampling schemes on the performance of differentiable particle filters. A detailed discussion of previous work that is most relevant to the proposed work is presented in Section~\ref{subsec:related_work}.
	
    One limitation of existing differentiable particle filters is that most of them only employ vanilla neural networks to construct their components~\cite{jonschkowski18,karkus2018particle,corenflos2021differentiable}. However, as vanilla neural networks do not allow density estimation, i.e. we do not know the probability density of their outputs, differentiable particle filters built with vanilla neural networks often include multiple levels of approximations such that desired statistical properties of standard particle filtering frameworks do not apply to these methods. For example, the transition density of particles is either modeled by simple distributions such as Gaussian distributions~\cite{karkus2018particle,corenflos2021differentiable} or ignored~\cite{jonschkowski18}. As a result, for all but a few trivial low-dimensional examples, existing differentiable particle filters are often restricted to the bootstrap particle filtering framework, making them susceptible to the weight degeneracy issue~\cite{bickel2008sharp}. 
	
	To address these issues, in this paper, we present a normalizing flow-based differentiable particle filtering framework. By leveraging (conditional) normalizing flows, the proposed method provides a flexible mechanism to model complex dynamics of latent states and design valid and effective proposal distributions. In addition, we use conditional normalizing flows to construct measurement models that admit valid probability densities.
	The contributions of this paper are as follows:
	\begin{itemize}
		\item We propose a normalizing flow-based differentiable particle filter (NF-DPF), which provides a flexible mechanism for modeling complex state-space models and admits valid probability densities for each component of the proposed method; 
		\item We establish convergence results for the proposed method, proving that the approximation error of the resulting Monte Carlo estimates vanishes when the number of particles approaches infinity; To the best of our knowledge, this is the first work that established such convergence properties for both predictive and posterior approximations in differentiable particle filters.
		\item We report that the proposed method leads to improved performance over state-of-the-art differentiable particle filters on a variety of benchmark datasets in this field.
	\end{itemize}
    Compared with existing differentiable particle filtering methods, the proposed method offers three main advantages. Firstly, since normalizing flows are universal approximators~\cite{papamakarios2021normalizing}, the proposed NF-DPF can theoretically approximate any dynamic models, proposal distributions, and measurement models arbitrarily well. Secondly, the proposed NF-DPF is not restricted to the bootstrap particle filtering framework and thus can alleviate the weight degeneracy issue. Lastly, in contrast to differentiable particle filters built with vanilla neural networks, all the components of the NF-DPF yield valid probability density. This enables maximum likelihood training of the NF-DPF, making it suitable for scenarios where ground-truth latent states are not accessible.
    
    Some of our initial explorations on building differentiable particle filters with normalizing flows were reported in abbreviated forms in conference papers~\cite{chen2021differentiable,chen2022conditional}. In this paper, more details of the proposed method are presented and discussed. Additionally, in this work, we establish convergence results for the proposed method and validate the effectiveness of the proposed method on a more extensive set of numerical experiments.
	
	The rest of this paper is organized as follows. A detailed review of previous work relevant to our work is presented in Section~\ref{subsec:related_work}. The problem statement is presented in Section~\ref{sec:problem_statement}. In Section~\ref{sec:preliminaries}, we provide the necessary background knowledge for introducing the proposed method. The proposed normalizing flow-based differentiable particle filter is detailed in Section~\ref{sec:nfdpf}. Convergence results of the proposed method are established in Section~\ref{sec:theoretical_results}. The performance of the proposed method is evaluated and compared with the other differentiable particle filters in Section~\ref{sec:experiments}. We conclude this paper in Section~\ref{sec:conclusion}.

	\subsection{Related Work}
	\label{subsec:related_work}
	Parameter estimation for particle filters has long been an active research area, and various techniques have been proposed to address this task in several different directions~\cite{kantas2009overview,kantas2015particle}. One type of such parameter estimation methods is the maximum likelihood (ML) methods~\cite{malik2011particle,hurzeler2001approximating,ionides2006inference,ionides2011iterated,zhao2013parameter,wills2008parameter,olsson2008sequential,andrieu2005line,legland1997recursive}. In~\cite{malik2011particle} and~\cite{hurzeler2001approximating}, importance sampling and common random numbers methods are respectively used to create a continuous version of the resampling step and learn parameters of particle filters by maximizing the marginal observation likelihood with gradient descent. To obtain a low-variance estimate of the gradient of log-likelihood, different variance reduction techniques have been proposed~\cite{ionides2006inference,ionides2011iterated}. An alternative approach that can maximize the log-likelihood in a numerically more stable way is the expectation-maximization (EM) algorithm~\cite{zhao2013parameter,wills2008parameter,olsson2008sequential}. 
    Another main category of techniques developed for estimating parameters of state-space model is the Bayesian parameter estimation methods, where the parameters to be estimated are assigned with a prior distribution and the estimate is characterized by the posterior distribution of the parameters given the observations~\cite{perez2018probabilistic,crisan2018nested,perez2021nested,andrieu2010particle,lindsten2014particle,chopin2013smc2,crisan2002survey,polson2008practical,fernandez2007estimating}. One typical example of Bayesian parameter estimation methods is the particle Markov chain Monte Carlo (PMCMC) method and its variants~\cite{andrieu2010particle,fernandez2007estimating,lindsten2014particle}, which use Markov chain Monte Carlo (MCMC) samplers to generate Monte Carlo estimates of the parameter posterior. 
    
    One common assumption in the aforementioned methods is that the structures or part of parameters of the dynamic and measurement models are known, which often cannot be satisfied in real-world applications. To alleviate this limitation, several methods resort to combining particle filtering methods with machine learning tools such as neural networks and gradient descent. We refer to these methods as differentiable particle filters~\cite{chen2023overview,rosato2022efficient,jonschkowski18,karkus2018particle,corenflos2021differentiable,scibior2021differentiable,kloss2021train,ma2020particle}.
    
    For dynamic models in differentiable particle filters, it was proposed in~\cite{jonschkowski18,karkus2018particle} to parametrize dynamic models using fully-connected neural networks with previous states and given actions as inputs. In~\cite{corenflos2021differentiable, scibior2021differentiable,ma2020particle}, recurrent neural networks such as long short-term memory (LSTM) networks and gated recurrent unit (GRU) networks were applied in differentiable particle filters to model the transition of latent states. Dynamic models with known functional forms are considered in~\cite{kloss2021train}, while the covariance matrices for dynamic noise variables need to be learned. To the best of our knowledge, existing dynamic models of differentiable particle filters are limited to pre-specified distribution families, such as Gaussian distributions. This limitation highlights the need for a more general and flexible framework to enable differentiable particle filters to construct complex latent dynamics.

    For measurement models in existing differentiable particle filters, they are either only able to produce unnormalized probability densities or limited to pre-defined simple distribution families like Gaussian distributions. For example, in robot localization tasks reported in~\cite{jonschkowski18,karkus2018particle,ma2020particle}, the unnormalized conditional likelihood of observations given states is estimated by vanilla neural networks with observations and particles as inputs. The conditional likelihood function in~\cite{corenflos2021differentiable,kloss2021train,scibior2021differentiable} is defined as probability density functions (PDFs) of known distributions with parameters determined by state features.
    
    In several trivially simple cases, hand-crafted proposal distributions have shown to be effective~\cite{naesseth2018variational,le2018auto}. However, most existing differentiable particle filters are built with vanilla neural networks, while in general, it is not feasible to compute the density of vanilla neural networks' output. There is a lack of more general mechanisms for constructing proposal distributions in these differentiable particle filters.
	
    A key ingredient to achieving fully differentiable particle filters is a differentiable resampling scheme~\cite{rosato2022efficient}. Several approaches have been developed towards this direction~\cite{karkus2018particle,zhu2020towards,corenflos2021differentiable,rosato2022efficient,scibior2021differentiable}. One class of differentiable resampling schemes designs the particle weights after resampling as a differentiable function of particle weights before resampling such that the gradients backpropagated through resampling steps are non-zero~\cite{karkus2018particle,rosato2022efficient}. A truly differentiable resampling scheme was proposed in~\cite{corenflos2021differentiable}, where the deterministic and differentiable resampling output is obtained by solving an entropy-regularized optimal transport problem with the Sinkhorn algorithm~\cite{cuturi2013sinkhorn,feydy2019interpolating,peyre2019computational}. It was shown that the resampling scheme proposed in~\cite{corenflos2021differentiable} leads to biased but asymptotically consistent estimates of the log-likelihood. A particle transformer was introduced in~\cite{zhu2020towards} based on a set transformer architecture~\cite{lee2019set,vaswani2017attention}, which needs to be trained from collected data beforehand and therefore can hardly be adapted to new tasks. Because these resamplers do not conform to the standard multinomial resampling framework, establishing convergence results for differentiable particle filters is challenging.  In this paper, we adopted the entropy-regularized optimal transport resampler and, for the first time, established convergence results for both predictive and posterior approximations in the proposed method.
	
    Loss functions that are often used in training differentiable particle filters can be grouped into two main classes. The first class is the likelihood-based loss functions~\cite{naesseth2018variational,le2018auto,maddison2017filtering}. In~\cite{naesseth2018variational,le2018auto,maddison2017filtering}, an evidence lower bound (ELBO) of observation log-likelihood was derived within a general particle filtering framework and maximized to learn system models and proposal distributions~\cite{corenflos2021differentiable,le2018auto}. The other type of loss functions involves task-specific objectives, e.g. root mean square error (RMSE) between estimates of states and ground-truth states~\cite{jonschkowski18,karkus2018particle,corenflos2021differentiable,ma2020particle}, among others~\cite{ma2019discriminative,ma2020particle,dupty2021pf,chen2021deep}. It was reported in~\cite{ma2020particle} that combining task-specific loss functions with log-likelihood objectives gives the best empirical performance in numerical simulations.

    \section{Problem Statement}
    \label{sec:problem_statement}
    We consider filtering problems in state-space models (SSMs). State-space models refer to a class of sequential models that consist of two discrete-time variables, the latent state variable ${x}_t$, ${t\geq0}$ defined on $\mathcal{X}\subseteq\mathbb{R}^{d_\mathcal{X}}$, and the observed measurement variable ${y}_t$, ${t\geq0}$ defined on $\mathcal{Y}\subseteq\mathbb{R}^{d_\mathcal{Y}}$~\cite{doucet2001introduction}. The latent state ${x}_t$, ${t\geq0}$  is characterized by a Markov process with an initial distribution $\pi(x_0)$ and a transition kernel $p(x_{t} | {x}_{t-1};\theta)$ for $t\geq 1$. The observation ${y}_t$ is conditionally independent given the current latent state $x_t$:
	\begin{align}
		\label{eq:initial_dist}
		&{x}_{0} \sim \pi(x_0; \theta)\,, \\
            \label{eq:ssm_dynamic_model}
		&{x}_{t}| x_{t-1} \sim p(x_{t}| {x}_{t-1}; \theta) \text { for } t \geq 1 \,,\\
		&{y}_{t}| x_t \sim p(y_t | {x}_{t} ; \theta) \text { for } t \geq 0\,,
	\end{align}
    where $\theta\in\Theta$ is the parameter set of interest. 
    Denoted by ${x}_{0:t}:=\{{x}_0,\,\,\cdots,\,\,{x}_t\}$ and ${y}_{0:t}:=\{{y}_0,\,\,\cdots,\,\,{y}_t\}$ the sequences of latent states and observations up to time step $t$ respectively. In this work, our goal is to jointly estimate the joint posterior distribution $p\left(x_{0: t}| y_{0: t};\theta\right)$ or the marginal posterior distribution $p\left(x_{t}| y_{0: t};\theta\right)$ and the parameter set $\theta$. 

	\section{Preliminaries}
	\label{sec:preliminaries}
    
    \subsection{Particle Filtering}
    Except for a limited class of state-space models such as linear Gaussian models~\cite{kalman1960new}, analytical solutions for the posterior distribution $p(x_{0:t}| y_{0:t};\theta)$ are unavailable since they involve complex high-dimensional integrations over $\mathcal{X}^{t+1}$. Particle filters are an alternative solution to the above problem. In particular, particle filters approximate intractable joint posteriors with empirical distributions consisting of sets of weighted samples $\{\tilde{w}_t^i,x_{0:t}^i\}_{i\in [N]}$:
	\begin{equation}
		\label{eq:pf_approx}
		p(x_{0:t}| y_{0:t};\theta)\approx\sum_{i=1}^{N}\tilde{w}_t^i\, \delta_{x_{0:t}^i}(x_{0:t})\,,
	\end{equation} 
	where $[N]:=\{1,\cdots,N\}$, $N$ is the number of particles, $\delta_{x_{0:t}^i}(\cdot)$ denotes the Dirac delta measure located in $x_{0:t}^i$, and $\tilde{w}_t^i\geq 0$ with $\sum_{i=1}^{N}\tilde{w}_t^i=1$ is the normalized importance weight of the $i$-th particle at the $t$-th time step. Particles with higher importance weights are believed to be closer to the true state than those with lower importance weights. The particles $\{x_{0:t}^i\}_{i\in [N]}$ are sampled from proposal distributions $q(x_{0}| y_0;\phi)$ when $t=0$ and $q(x_{t}| y_t, x_{t-1};\phi)$ for $t\geq1$. Denote by $w_t^i$ unnormalized importance weights of particles, importance weights of particles are updated recursively through:
	\begin{equation}
		\label{eq:weight_update}
		w_t^i= w_{t-1}^i \frac{p(y_t| x_t^i;\theta)p(x_t^i| x_{t-1}^i;\theta)}{q(x_{t}^i| y_t, x_{t-1}^i;\phi)}\,,
	\end{equation}
	with $w_{0}^i=\frac{p(y_0| x_0^i;\theta)\pi(x_0^i;\theta)}{q(x_0^i| y_0;\phi)}$ and normalized as $\tilde{w}_t^i=\frac{w_t^i}{\sum_{j=1}^{N}w_t^j}$.
    Particle resampling is triggered when a predefined condition is satisfied~\cite{douc2005comparison,li2015resampling}.
	
	\subsection{Differentiable particle filters}
	\label{subsec:dpfs}
	In differentiable particle filters, both the evolution of latent state $x_t$ and the relationship between the observation $y_t$ and the latent state $x_t$ are modeled by neural networks. Particularly, differentiable particle filters describe the transition of the state $x_t$ using a parametrized function $g_\theta(\cdot): \mathcal{X}\times \mathbb{R}^{d_\varsigma}\rightarrow \mathcal{X}$:
	\begin{gather}
		\label{eq:dpf_dynamic_model}
		{x}_t=g_\theta( {x}_{t-1}, \varsigma_t)\sim p({x}_t| {x}_{t-1};\theta)\,,
	\end{gather}
	where $\varsigma_t\in\mathbb{R}^{d_\varsigma}$ is the noise sample used to simulate the dynamic noise, and $g_\theta(\cdot)$ is differentiable w.r.t. $x_{t}$ and $\varsigma_t$.
	For measurement models, one commonly adopted construction is through a parametrized function $l_\theta(\cdot):\mathcal{Y}\times\mathcal{X}\rightarrow\mathbb{R}$:
	\begin{gather}
		\label{eq:dpf_measurement_model}
		p( {y}_t| {x}_t;\theta)\propto l_\theta( {y}_t,  {x}_t)\,,
	\end{gather}
    $l_\theta(\cdot)$ measures compatibilities between $y_t$ and $x_t$ and needs to be differentiable w.r.t. both $y_t$ and $x_t$.
	Similarly, the proposal distribution can also be constructed through a parametrized function $f_\phi(\cdot):\mathcal{X}\times \mathcal{Y}\times\mathbb{R}^{d_\upsilon}\rightarrow\mathcal{X}$:
	\begin{equation}
		\label{eq:dpf_proposal}
		{x}_t=f_\phi( {x}_{t-1},  {y}_{t}, \upsilon_t)\sim q( {x}_t| {x}_{t-1},  {y}_{t};\phi)\,,
	\end{equation}
	where $\upsilon_t\in\mathbb{R}^{d_\upsilon}$ refers to the sampling noise.
	
	While it has been widely documented that the resampling step is non-differentiable~\cite{karkus2018particle,rosato2022efficient}, several approaches have been developed to solve this problem~\cite{karkus2018particle,zhu2020towards,corenflos2021differentiable,rosato2022efficient,scibior2021differentiable}. With the differentiable components discussed above, differentiable particle filters are optimized by minimizing a loss function through gradient descent.
	
	\subsection{Normalizing flows}
 	\label{sec:nfs}
    Consider a $D$-dimensional variable $z\sim p_Z(z)$, where $p_Z(\cdot)$ is a known simple distribution, e.g. Gaussian, defined on $\mathcal{Z}\subseteq\mathbb{R}^{D}$. We define a variable $s$ on $\mathcal{S}\subseteq\mathbb{R}^{D}$ through a transformation $s=\mathcal{T}_\vartheta(z)$, $\mathcal{T}_\vartheta(\cdot):\mathcal{Z}\rightarrow\mathcal{S}$, where $\vartheta$ is the parameter of the transformation. The transformation $\mathcal{T}_\vartheta(\cdot)$ is called a normalizing flow if it is invertible w.r.t. $z$ and differentiable w.r.t. $\vartheta$ and $z$~\cite{kobyzev2020normalizing,papamakarios2021normalizing}. Under some mild assumptions, $s=\mathcal{T}_\vartheta(z)$ can represent arbitrarily complex distributions, even if the distribution of $z$ is as simple as a standard Gaussian~\cite{papamakarios2021normalizing}.
	
	Recent developments in normalizing flows focus on constructing invertible transformations with neural networks~\cite{dinh2017density,balunovic2021fair,kingma2018glow,rezende2015variational,karami2019invertible,huang2018neural}. Compared with vanilla neural networks which cannot produce valid probability densities, the density of normalizing flows' output $s=\mathcal{T}_\vartheta(z)$ can be obtained by applying the change of variable formula. Since the composition of a series of invertible and differentiable transformations is still invertible and differentiable, we can stack $K$ simple invertible transformations $\{\mathcal{T}_{\vartheta_k}(\cdot)\}_{k=1}^K$ together and yield a more expressive normalizing flow $s=\mathcal{T}_{\vartheta_K}\circ\mathcal{T}_{\vartheta_{K-1}}\circ\cdots\circ\mathcal{T}_{\vartheta_1}(z)$. Correspondingly, the density of $s$ can be computed by successively applying the change of variable formula.
 
    One simple example of normalizing flows is the planar flow~\cite{rezende2015variational}. Denote by $z\in \mathbb{R}^D$ the input of a planar flow, the output of the planar flow is computed as follows:
	\begin{gather}
		\label{eq:planar_flow}
		\mathcal{T}_\vartheta(z)=z+v{h}(w^\intercal z+b)\,,
	\end{gather}
	where $\mathcal{T}_\vartheta(\cdot)$ is parametrized by $\vartheta:=\{w\in\mathbb{R}^D, v\in\mathbb{R}^D, b\in\mathbb{R}\}$, and $h(\cdot):\mathbb{R}\rightarrow\mathbb{R}$ is a smooth non-linear function. It has been proved in~\cite{rezende2015variational} that Eq.~\eqref{eq:planar_flow} is invertible when some mild conditions on $w$, $u$, and $h(\cdot)$ are satisfied. The Jacobian determinant of planar flow can be computed in $\mathcal{O}(D)$ time.
	
	Another variant of normalizing flows, the Real-NVP model~\cite{dinh2017density}, constructs invertible transformations through coupling layers. In standard coupling layers the input $z$ is split into two parts $z=[z_1, z_2]$, where $z_1=\underset{1:d}{z}$ refers to the first $d$ dimensions of $z$, and $z_2=\underset{d+1:D}{z}$ refers to the last $D-d$ dimensions of $z$. The partition is uniquely determined by an index $d<D$, and the output $s\in\mathbb{R}^D$ of the coupling layer is given by:
	\begin{gather}
		\label{eq:coupling_layers}
		\underset{1:d}{s}=\underset{1:d}{z}\,,\\
		\underset{d+1:D}{s}=\underset{d+1:D}{z}\odot e^ {\gamma_\vartheta(\underset{1:d}{z})}+\eta_\vartheta(\underset{1:d}{z})\,,
	\end{gather}
	where $\gamma_\vartheta(\cdot): \mathbb{R}^{d}\rightarrow \mathbb{R}^{D-d}$ and $\eta_\vartheta(\cdot): \mathbb{R}^{d}\rightarrow \mathbb{R}^{D-d}$ stand for the scale function and the translation function, respectively, $\odot$ refers to element-wise products, and the exponential of the vector $\gamma_\vartheta(\underset{1:d}{z})$ is also applied element-wise. With the special structure defined by Eq.~\eqref{eq:coupling_layers}, coupling layers are invertible by design, and Jacobian matrices of coupling layers are lower or upper triangular matrices, such that the Jacobian determinants can be efficiently computed. In addition to coupling layers, more tricks to build expressive invertible transformations such as multi-scale structure, masked convolution, and batch normalization can be found in~\cite{dinh2017density}. 
    
    To model the conditional probability density of $s$ conditioned on $u\in\mathbb{R}^{d_u}$, i.e. $p(s|u;\varphi)$, one can use another type of normalizing flows, called conditional normalizing flow~\cite{winkler2019learning,lu2020structured}. Both planar flows and Real-NVP models have their conditional counterparts. Given a condition $u\in \mathbb{R}^{d_u}$, a planar flow can be made conditional with the following modification:
	\begin{gather}
		\label{eq:conditional_planar_flow}
		\mathcal{F}_\varphi(z; u)=z+v{h}(w^\intercal z+b\odot s(u))\,,
	\end{gather}
	where $s(\cdot):\mathbb{R}^{d_u}\rightarrow\mathbb{R}^D$ can be any linear or non-linear functions and does not impair the invertibility of the flow. Conditional Real-NVP models were proposed in~\cite{winkler2019learning} by replacing standard coupling layers with conditional coupling layers. Compared to standard coupling layers, the scale and translation functions in conditional coupling layers are functions of concatenations of the base variable $z_1=\underset{1:d}{z}$ and a given condition $u\in \mathbb{R}^{d_u}$, i.e. the input of the translation and scale functions now becomes $[\underset{1:d}{z}, u]$. Specifically, a conditional coupling layer can be formulated as:
	\begin{gather}
		\label{eq:conditional_coupling_layers_1}
		\underset{1:d}{s}=\underset{1:d}{z}\,,\\
		\label{eq:conditional_coupling_layers_2}
		\underset{d+1:D}{s}=\underset{d+1:D}{z}\odot e^{\tilde{\gamma}_\varphi(\underset{1:d}{z},u)}+\tilde{\eta}_\varphi(\underset{1:d}{z},u)\,,
	\end{gather}
	where $\tilde{\gamma}_\varphi(\cdot):\mathbb{R}^{d+d_u}\rightarrow\mathbb{R}^{D-d}$ and $\tilde{\eta}_\varphi(\cdot):\mathbb{R}^{d+d_u}\rightarrow\mathbb{R}^{D-d}$ stand for the conditional scale function and the conditional translation function, respectively. 
	
	\section{Normalizing flow-based differentiable particle filters}
	\label{sec:nfdpf}
\begin{center}
    \begin{figure}
        \centering
        \includegraphics[width=\linewidth]{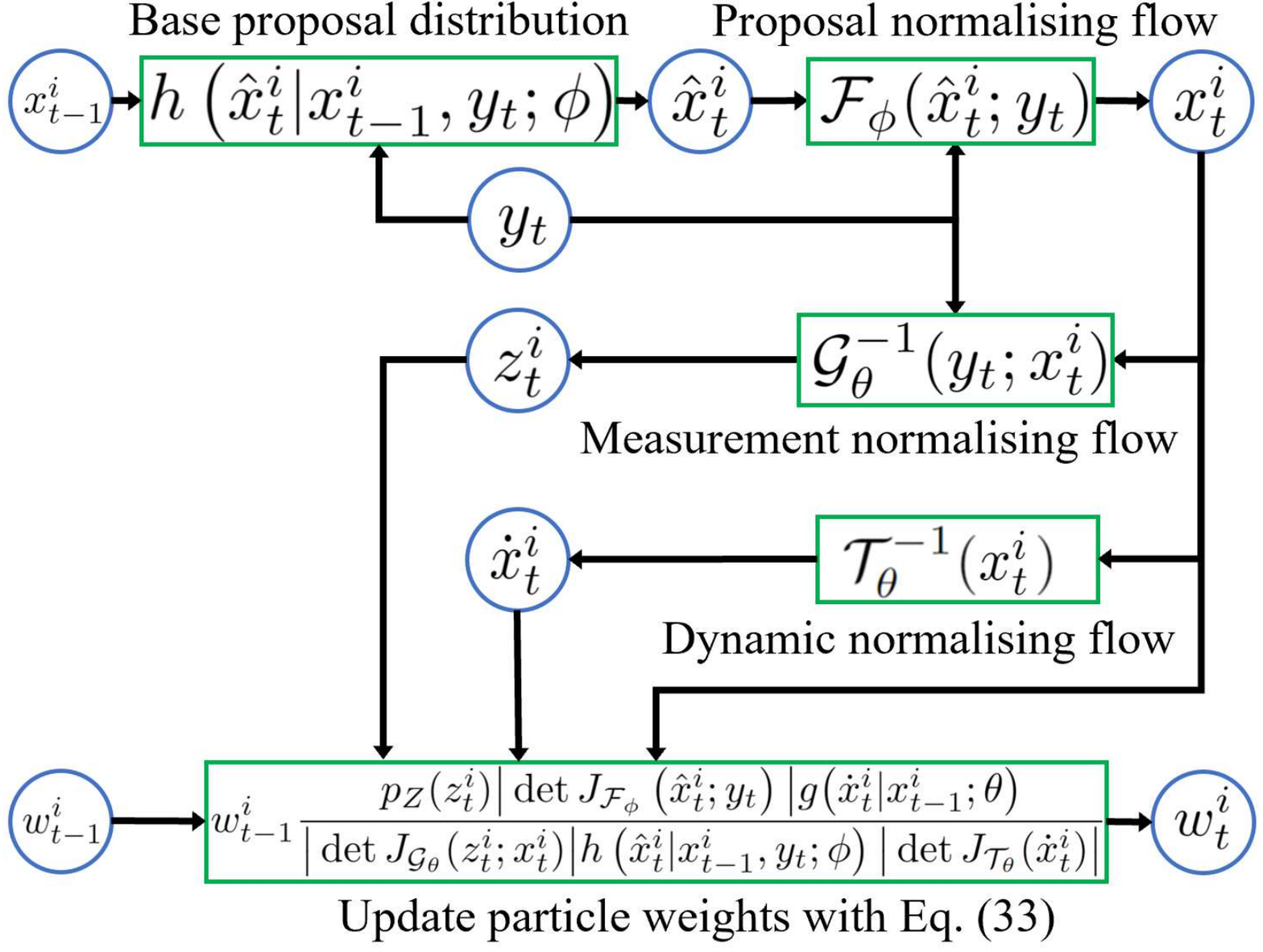}
        \caption{A diagram that shows the overall structure of the proposed NF-DPF, illustrating how to generate new particles and update particle weights in NF-DPFs. Blue circles refer to random variables. Green rectangles refer to operations such as drawing samples or evaluating certain functions.}
        \label{fig:NFDPF_diagram}
    \end{figure}
\end{center}
	In this section, we present details of the proposed normalizing flow-based differentiable particle filter (NF-DPF), including its dynamic model, measurement model, and proposal distribution. Specifically, we first show that normalizing flows can provide a flexible mechanism for learning complex dynamics of latent states. In addition, the construction of proposal distributions with tractable proposal densities can also be achieved using conditional normalizing flows. Lastly, we elaborate on how to construct measurement models with valid probability densities using conditional normalizing flows. An illustration of the structure of the proposed NF-DPF is presented in Fig.~\ref{fig:NFDPF_diagram}. Note that the parameter sets $\vartheta$ and $\varphi$ in Section~\ref{sec:nfs} are respectively subsets of $\theta$ and $\phi$ and the normalizing flows in this section are used to construct components in state-space models, so we use $\mathcal{T}_\theta(\cdot)$ and $\mathcal{G}_\phi(\cdot)$ instead of $\mathcal{T}_\vartheta(\cdot)$ and $\mathcal{G}_\varphi(\cdot)$ to denote the normalizing flows in the following contents. 
	\subsection{Dynamic models with normalizing flows}
	\label{subsec:dynamic_models_nf}
	We first show how to use normalizing flows to construct flexible dynamic models. Here we consider a base distribution $g(\cdot|x_{t-1}; \theta)$, e.g. Gaussian distribution, from which we can draw samples and obtain tractable probability density, and a normalizing flow $\mathcal{T}_\theta(\cdot): \mathcal{X} \rightarrow \mathcal{X}$ parametrized by $\theta$. To draw samples from the proposed dynamic model, a set of particles $\{\dot{x}_t^i\}_{i=1}^{N}$ are first drawn from $g(\cdot|x_{t-1}; \theta)$. 
	Thereafter, $\{\dot{x}_t^i\}_{i=1}^{N}$ are further transformed by the normalizing flow $\mathcal{T}_\theta(\cdot)$ and considered as samples from $p(x_t|x_{t-1};\theta)$:
	\begin{gather}
		\dot{x}_t^i \sim g(\dot{x}_t|x_{t-1}; \theta)\,,\\
		{x}_t^i=\mathcal{T}_\theta(\dot{x}_t^i)\sim p({x}_t|x_{t-1};\theta)\,.
	\end{gather}
	By applying the change of variable formula, the probability density function of the proposed dynamic model can be formulated as:
	\begin{gather}
		\label{eq:dynamic_density}
		p({x}_t|x_{t-1};\theta)=g\big(\dot{x}_t|x_{t-1};\theta\big)\Big|\operatorname{det}J_{\mathcal{T}_\theta}(\dot{x}_t)\Big|^{-1}\,,\\
		\dot{x}_t=\mathcal{T}_\theta^{-1}({x}_t)\sim g\big(\dot{x}_t|x_{t-1};\theta\big)\,,
	\end{gather}
	where $\operatorname{det}J_{\mathcal{T}_\theta}(\dot{x}_t)$ is the Jacobian determinant of $\mathcal{T}_\theta(\cdot)$ evaluated at $\dot{x}_t=\mathcal{T}_\theta^{-1}({x}_t)$.
	\subsection{Proposal distributions with conditional normalizing flows}
	\label{subsec:proposal_cnf}
	We propose to incorporate information from observations to construct proposal distributions by using conditional normalizing flows. We use $\mathcal{F}_\phi(\cdot):\mathcal{X}\times \mathcal{Y}\rightarrow \mathcal{X}$ to denote a conditional normalizing flow defined on $\mathcal{X}\times \mathcal{Y}$, where $\mathcal{X}$ and $\mathcal{Y}$ are the ranges of state $x_t$ and observation $y_t$, respectively. In the proposed method, particles sampled from the proposal distributions are obtained by transforming samples from a base proposal distribution $h(\cdot|x_{t-1}, y_t;\phi)$, $t\geq 1$, and $h_0(\cdot|y_0;\phi)$, $t=0$, with the conditional normalizing flow $\mathcal{F}_\phi(\cdot)$:
	\begin{gather}
		\label{eq:base_proposal} 
		\hat{x}_{0}^i\sim h_0(\hat{x}_{0}|y_0;\phi)\,,\\
            \hat{x}_{t}^i\sim h(\hat{x}_{t}|x_{t-1}, y_t;\phi)\,,\\
		\label{eq:proposal_particles_1}
		x_0^i=\mathcal{F}_\phi(\hat{x}_{0}^i;y_0)\sim q( {x}_0| {y}_{0};\phi)\,,\\
            \label{eq:proposal_particles_2}
		x_t^i=\mathcal{F}_\phi(\hat{x}_{t}^i;y_t)\sim q( {x}_t| {x}_{t-1},  {y}_{t};\phi)\,.
	\end{gather}
    The base proposal distribution $h(\cdot|x_{t-1}, y_t;\phi)$ is a distribution with a tractable probability density which we can draw samples from, e.g. a Gaussian distribution. The conditional normalizing flow $\mathcal{F}_\phi(\cdot)$ is an invertible function of particles $\hat{x}_t^i$ given the observation $y_t$. Since the information from observations is taken into account, the conditional normalizing flow $\mathcal{F}_\phi(\cdot)$ provides the capability to migrate particles to regions that are closer to the true posterior distributions. 
 
    The proposal density can be obtained by applying the change of variable formula:
	\begin{gather}
		\label{eq:proposal_density}
		q(x_{t}|x_{t-1}, y_{t};\phi)=h\left(\hat{x}_{t}| x_{t-1}, y_t; \phi\right)\bigg|\operatorname{det}J_{\mathcal{F}_{\phi}}\left(\hat{x}_{t}; y_t\right)\bigg|^{-1}\,,\\ \hat{x}_{t}=\mathcal{F}_\phi^{-1}(x_t; y_t)\sim h\left(\hat{x}_{t}| x_{t-1}, y_t; \phi\right)\,,
	\end{gather}
	where $\operatorname{det}J_{\mathcal{F_\phi}}(\hat{x}_{t}, y_t)$ refers to the determinant of the Jacobian matrix $J_{\mathcal{F}_\phi}(\hat{x}_t; y_t)=\frac{\partial \mathcal{F}_\phi(\hat{x}_t; y_t)}{\partial \hat{x}_t}$ evaluated at $\hat{x}_t$.
	
	\subsection{Conditional normalizing flow measurement models}
	\label{subsec:cnf_measurement}
	Given an observation $y_t$ and a state value $x_t$, we model the relationship between $y_t$ and $x_t$ through a conditional normalizing flow $\mathcal{G}_\theta(\cdot):\mathbb{R}^{d_\mathcal{Y}}\times\mathcal{X}:\rightarrow \mathcal{Y}$:
	\begin{gather}
		\label{eq:observation_generation}
		y_t=\mathcal{G}_\theta(z_t; x_t)\,,
	\end{gather}
	where $d_\mathcal{Y}$ is the dimension of the observation $y_t$, $z_t=\mathcal{G}^{-1}_\theta(y_t; x_t)$ is the base variable which follows an independent marginal distribution $p_Z(z_t)$ defined on $\mathbb{R}^{d_\mathcal{Y}}$, and the state $x_t$ is the condition variable. Note that while the base distribution $p_Z(z_t)$ can be an arbitrary distribution with tractable densities, standard Gaussian distributions are employed in this work for the sake of simplicity. The main reason that conditional normalizing flows are used for our measurement models is that standard normalizing flows require their input and output to have the same dimensionality, while $y_t$ and $x_t$ often differ in their dimensionalities. Note that the invertible transformation $\mathcal{G}_\theta(\cdot)$ used here is a new construction of conditional normalizing flow that is different from $\mathcal{F}_\phi(\cdot)$, as they are used to model different conditional probabilities. 
 
    With the conditional generative process of $y_t$ defined by Eq.~\eqref{eq:observation_generation}, we can evaluate the likelihood of the observation $y_t$ given $x_t$ through:
	\begin{gather}
		\label{eq:base}
		z_t=\mathcal{G}^{-1}_\theta(y_t; x_t)\,,\\
		p(y_t|x_t;\theta)=p_Z(z_t)\bigg|\operatorname{det} J_{\mathcal{G}_\theta}(z_t; x_t)\bigg|^{-1}\,.
        \label{eq:measurement_raw}
	\end{gather}
	
	In scenarios where observations are high-dimensional such as images, evaluating $p(y_t|x_t;\theta)$ with Eq.~\eqref{eq:measurement_raw} using raw observations ${y}_t$ can be computationally expensive. As an alternative solution, we propose to map the observation ${y}_t$ to a lower-dimensional space via ${e}_t=U_\theta({y}_t)\in\mathbb{R}^{d_{e}}$, where $U_\theta$ is a parametrized encoder function $U_\theta(\cdot):\mathbb{R}^{d_\mathcal{Y}}\rightarrow\mathbb{R}^{d_{e}}$. To ensure that the feature $e_t$ maintains key features contained in $y_t$, we introduce a decoder $D_\theta(\cdot):\mathbb{R}^{d_{e}}\rightarrow\mathbb{R}^{d_\mathcal{Y}}$ to reconstruct $y_t$, and include the following autoencoder (AE) loss into the training objective:
	\begin{gather}
		\label{eq:loss_ae}
		\mathcal{L}_\text{AE}(\theta)=\frac{1}{T}\sum_{t=0}^{T}||D_\theta(U_\theta({y}_t))-{y}_t||_2^2\,,
	\end{gather}
	where $T$ is the trajectory length. We then assume that the conditional probability density $p({e}_t|{x}_t;\theta)$ of observation features given state is an approximation of the actual measurement likelihood $p({y}_t|{x}_t;\theta)$~\cite{corenflos2021differentiable}:
	\begin{align}
            \label{eq:obs_feature}
		{e}_t&=U_\theta({y}_t)\,,\\
		\label{eq:feature_base}
		z_t&=\mathcal{G}^{-1}_\theta(e_t; x_t)\,,\\
		p(y_t|x_t;\theta)&\approx p(e_t|x_t;\theta)\\
		&=p_Z(z_t)\bigg|\operatorname{det} J_{\mathcal{G}_\theta}(z_t; x_t)\bigg|^{-1}\,.
		\label{eq:measurement_cnf}
	\end{align}

    \subsection{Importance weights update}
    Combining Eqs.~\eqref{eq:dynamic_density},~\eqref{eq:proposal_density}, and~\eqref{eq:measurement_cnf}, in the proposed normalizing flow-based differentiable particle filters, importance weights of $x_t^i\sim q(x_t^i|x_{t-1}^i,y_t;\phi)$ for $t\geq 1$ are updated as follows:
	\begin{align}
		&w_t^i\propto w_{t-1}^i \frac{p(y_t|x_t^i;\theta)p(x_t^i|x_{t-1}^i;\theta)}{q(x_t^i|x_{t-1}^i,y_t;\phi)}\\
		\label{eq:weights_nf_dpfs}
		&=w_{t-1}^i \frac{p_Z(z_t^i)\big|\operatorname{det}J_{\mathcal{F}_{\phi}}\left(\hat{x}_{t}^i; y_t\right)\big|g\big(\dot{x}_t^i|x_{t-1}^i;\theta\big)}{\big|\operatorname{det} J_{\mathcal{G}_\theta}(z_t^i;x_t^i)\big| h\left(\hat{x}_{t}^i| x_{t-1}^i, y_t; \phi\right)\big|\operatorname{det}J_{\mathcal{T}_\theta}(\dot{x}_t^i)\big|}\,,
	\end{align}
	where $z_t^i$ is computed by either Eq.~\eqref{eq:feature_base} or Eq.~\eqref{eq:base}, $\hat{x}_{t}^i=\mathcal{F}_\phi^{-1}(x_t; y_t)$, $\dot{x}_{t}^i=\mathcal{T}_\theta^{-1}({x}_t^i)$, and $w_0^i$ is obtained by:
	\begin{gather}
		w_0^i=\frac{p_Z(z_0^i)\big|\operatorname{det}J_{\mathcal{F}_{\phi}}\left(\hat{x}_{0}^i; y_0\right)\big|\pi({x}_0^i;\theta)}{h_0(\hat{x}_{0}^i|y_0;\phi)\big|\operatorname{det} J_{\mathcal{G}_\theta}(z_0^i;x_0^i)\big|}\,.
		\label{eq:weights_t0}
	\end{gather}

    We provide in Algorithm~\ref{alg:nf_dpfs} a detailed description of the proposed normalizing flow-based differentiable particle filters (NF-DPFs), where we use the entropy-regularized optimal transport resampler $\mathcal{R}_\epsilon(\{x_{t}^i\}_{i\in[N]},\{\tilde{w}_{t}^i\}_{i\in[N]}):\mathbb{R}^{N\times d_\mathcal X}\times \mathbb{R}^{N}\rightarrow\mathbb{R}^{N\times d_\mathcal X}$ in resampling steps~\cite{corenflos2021differentiable}, and denote the entropy regularization coefficient by $\epsilon$, the base distribution $p_Z(\cdot)$ is set to be a standard Gaussian distribution.
	\begin{algorithm}[ht!]
		\begin{algorithmic}[1]
			\caption{Normalizing flow-based differentiable particle filters}\label{alg:nf_dpfs}
			\STATE \textbf{Notations}:\\ 	
			\begin{footnotesize}	
				\hspace{-2em}\begin{tabular}[t]{l @{\hspace{.1em}} l}%
					$T$ & Trajectory length\\
					$y_{0:T}$ & Observations \\
					$g(\cdot)$ & Base dynamic model\\
					$\mathcal{F}_\phi(\cdot)$ & Proposal normalizing flow\\
					$N$ & Number of particles\\
					$\xi$ & Learning rate \\
					$U_\theta(\cdot)$ & Observation encoder \\
					$\epsilon$ & regularization coefficient\\
				\end{tabular}\hspace{-0.5em}%
				\begin{tabular}[t]{l @{\hspace{.1em}} l}
					$x_{0:T}$ & Latent states \\
					$\pi(\cdot;\theta)$ & Initial distribution of latent states\\
					$\mathcal{T}_\theta(\cdot)$ & Dynamic model normalizing flow\\
					$\mathcal{G}_\theta(\cdot)$ & Measurement normalizing flow\\
					$\text{ESS}_{\text{min}}$ & Resampling threshold\\
					$\mathcal{L}(\theta, \phi)$ & Overall loss function\\
					$\mathcal{R}_\epsilon(\cdot)$ & regularized OT resampler~\cite{corenflos2021differentiable}\\
					$p_Z(\cdot)$ & Standard Gaussian distribution\\
				\end{tabular}
			\end{footnotesize}
			
			\STATE \textbf{initialization}: Randomly initialize $\theta$ and $\phi$;\\
			Sample $\hat{x}_0^i\overset{\text{i.i.d}}{\sim} h_0(\hat{x}_{0}|y_0;\phi)$, $\forall i\in[N]:=\{1,\cdots,N\}$ (Eq.~\eqref{eq:base_proposal});
			\STATE Generate proposed particles (Eq.~\eqref{eq:proposal_particles_1}): \\$x_0^i=\mathcal{F}_\phi(\hat{x}_0^i; y_0)\sim q(x_0|y_0;\phi)$, $\forall i\in[N]$;
			\STATE [optional] Encode observation (Eq.~\eqref{eq:obs_feature}): ${y}_0 :=U_\theta({y}_0)$;
			\STATE Compute the base variable (Eq.~\eqref{eq:feature_base}):\\ $z_0^i=\mathcal{G}_\theta^{-1}(y_0;x_0^i)$, $\forall i \in[N]$;
			\STATE Compute weights $w_0^i$ using Eq.~\eqref{eq:weights_t0}, $\forall i\in[N]$;
                \STATE [optional] $\hat{p}(y_{0};\theta)=\sum_{i=1}^{N}w_{0}^i/N$; 
			\FOR {$t=1$ to $T$}
			\STATE Normalize weights $\tilde{w}_{t-1}^i\propto{w_{t-1}^i}$, $\sum_{i=1}^{N}\tilde{w}_{t-1}^i=1$;
			\STATE Compute the effective sample size: \\$\text{ESS}_{t-1}(\{\tilde{w}_{t-1}^i\}_{i\in [N]})=\frac{1}{\sum_{i=1}^{N}(\tilde{w}_{t-1}^i)^2}$;
			\IF {$\text{ESS}_t(\{\tilde{w}_{t-1}^i\}_{i\in [N]}) < \text{ESS}_{\text{min}}$} 
			\STATE $\{\tilde{x}_{t-1}^{i}\}_{i=1}^{N}\leftarrow$ $\mathcal{R}_\epsilon(\{x_{t-1}^i\}_{i\in[N]},\{\tilde{w}_{t-1}^i\}_{i\in[N]})$;
			\STATE ${{w}}_{t-1}^i\leftarrow {1}$, $\forall i \in[N]$;
			\ELSE
			\STATE $\tilde{x}_{t-1}^{i}\leftarrow {x}_{t-1}^{i}$, $\forall i \in[N]$;
			\ENDIF
			\STATE ${x}_{t-1}^{i}\leftarrow \tilde{x}_{t-1}^{i}$;
			\STATE Sample $\hat{x}_t^i\overset{\text{i.i.d}}{\sim}h\left(\hat{x}_{t}| x_{t-1}, y_t; \phi\right)$, $\forall i\in[N]$ (Eq.~\eqref{eq:base_proposal});
			\STATE Generate proposed particles (Eq.~\eqref{eq:proposal_particles_2}):\\
			$x_t^i=\mathcal{F}_\phi(\hat{x}_{t}^i;y_t)\sim q( {x}_t| {x}_{t-1},  {y}_{t};\phi)$, $\forall i\in[N]$;
			\STATE [optional] Encode observation (Eq.~\eqref{eq:obs_feature}): ${y}_t :=U_\theta({y}_t)$;
			\STATE Compute the base variable (Eq.~\eqref{eq:feature_base}):\\$z_t^i=\mathcal{G}_\theta^{-1}(y_t;x_t^i)$, $\forall i \in[N]$;
			\STATE Update weights using Eq.~\eqref{eq:weights_nf_dpfs}:\\
            $w_t^i={w}_{t-1}^i \frac{p(y_{t}| x_{t}^{i}; \theta) p(x_t^i| {x}_{t-1}^i;\theta)}{q(x_t^i| y_{t},{x}_{t-1}^i;\phi)}$, $\forall i\in[N]$;
                \STATE [optional] $\hat{p}(y_{t}|y_{0:t-1};\theta)=\sum_{i=1}^{N}w_{t}^i/\sum_{i=1}^{N}w_{t-1}^i$, $\hat{p}(y_{0:t};\theta)=\hat{p}(y_{t}|y_{0:t-1};\theta) \hat{p}(y_{0:t-1};\theta)$; 
			\ENDFOR
			\STATE Compute the overall loss $\mathcal{L}(\theta, \phi)$ (Examples of $\mathcal{L}(\theta, \phi)$ can be found in Eq.~\eqref{eq:ELBO} and Eq.~\eqref{eq:loss_disk});
			\STATE Update $\theta$ and $\phi$ through gradient descent:
			\begin{gather}
				\theta\leftarrow \theta-\xi\nabla_\theta\mathcal{L}(\theta, \phi)\,,
				\phi\leftarrow \phi-\xi\nabla_\phi\mathcal{L}(\theta, \phi)\,.\nonumber
			\end{gather}
		\end{algorithmic}
	\end{algorithm}
	
	\section{Theoretical analysis}
	\label{sec:theoretical_results}
	In this section, we establish convergence results for particle approximations in normalizing flow-based differentiable particle filters. We assume resampling is performed at each time step using the entropy-regularized optimal transport resampler~\cite{corenflos2021differentiable}. We use the notations below for the following contents:
	\begin{gather}
		\label{eq:notations_main}
		\alpha^{(t)}:=p(x_t|y_{0:t-1};\theta)\,,\;\alpha^{(t)}_{N}(\psi):=\frac{1}{N}\sum_{i=1}^{N}\psi(x_t^i)\,,\\
		\label{eq:notations_empirical_main}
            \beta^{(t)}:=p(x_t|y_{0:t};\theta)
		\,,\; \beta^{(t)}_{N}(\psi):=\sum_{i=1}^{N}\tilde{w}_t^i \psi(x_t^i)\,,
	\end{gather}
        \begin{gather}
            \omega^{(t)}(x_t)=p(y_t|x_t;\theta)\,,
		\label{eq:weight_function}
        \end{gather}
	with $\alpha^{(0)}:=\pi(x_0;\theta)$, $\psi(\cdot): \mathcal{X}\rightarrow \mathbb{R}$ is a function defined on $\mathcal{X}$, $\alpha^{(t)}_{N}$ is an approximation of the predictive distribution $\alpha^{t}$ with $N$ uniformly weighted particles, and $\beta^{(t)}_{N}$ is an approximation of the posterior distribution $\beta^{(t)}$ with $N$ particles weighted by $\tilde{w}_t^i$. For a measure $\alpha$ defined on $\mathcal{X}$ we use $\alpha(\psi)=\int_{\mathcal{X}}\psi(x)\alpha(\mathrm{d}x)$ to denote the expectation of $\psi(\cdot)$ w.r.t. $\alpha$. For the sake of simplicity, we restrict ourselves to the bootstrap particle filtering framework, i.e. particles $x_t^i$ are sampled from $p(x_t|x_{t-1};\theta)$. However, our proof can be modified to adapt to particle filters employing proposal distributions that are distinct from their dynamic models by taking into account the estimation error caused by sampling from $q( {x}_t| {x}_{t-1},  {y}_{t};\phi)$ instead of $p( {x}_t| {x}_{t-1};\theta)$.
	
	To prove the consistency of particle approximations provided by NF-DPFs, we introduce the following assumptions:
	\begin{assumption}\label{ass:compact}
		$\mathcal{X}$ is a compact subset of $\mathbb{R}^{d_\mathcal X}$ with diameter $\mathfrak{d}:=\underset{x,x'\in\mathcal{X}}{\sup}{||x-x'||_2}$, where $||\cdot||_2$ denotes the Euclidean distance.
	\end{assumption}
	\begin{assumption}\label{ass:transport_plan}
		For $\forall t \geq 0$, there exists a unique optimal transport plan between $\alpha^{(t)}$ and $\beta^{(t)}$ featured by a deterministic transport map $\textbf{T}_t(\cdot):\mathcal{X}\rightarrow\mathcal{X}$, and the transport map $\textbf{T}_t(\cdot)$ is $\lambda$-Lipschitz for $\forall t \geq 0$ with $\lambda>0$.
	\end{assumption}
	
	\begin{assumption}\label{ass:transition_lipcontraction}
		Denote by $f(\cdot)$ the transition kernel $p( {x}_t| {x}_{t-1};\theta)$ of NF-DPFs defined in Eq.~\eqref{eq:dynamic_density} and $\psi(\cdot):\mathcal{X}\rightarrow\mathbb{R}$ the considered bounded $k$-Lipschitz function, there exists an $\eta\in \mathbb{R}$ such that for any two probability measures $\mu, \rho$ on $\mathcal{X}$
		$$|\mu f(\psi) - \rho f(\psi)| \leq \eta|\mu(\psi) - \rho(\psi)|\,, s.t.\; \mu(\psi) \neq \rho(\psi)\,.$$
	\end{assumption}
	
	\begin{assumption}\label{ass:weight_lipcontraction}
		There exists a constant $\zeta\in \mathbb{R}$ such that for any continuous probability measure $\mu$ on $\mathcal{X}$ and its empirical approximation $\mu_N$, for weighted probability measures $\mu_{\omega_t}={\omega_t\mu}/{\mu(\omega_t)}$ and $\mu_{N,\omega_t}={\omega\mu_N}/{\mu_N(\omega_t)}$, we have
		$$\wass_2(\mu_{N,\omega_t}, \mu_{\omega_t}) \leq \zeta\wass_2(\mu_N, \mu)\,,$$
		where $\omega_t(\cdot):\mathcal{X}\rightarrow\mathbb{R}$ is defined in Eq.~\eqref{eq:weight_function}, and $\wass_2(\cdot, \cdot)$ refers to the $2$-Wasserstein distance~\cite{villani2008optimal,peyre2019computational}.
	\end{assumption}
	With the above assumptions, we provide the following proposition for the consistency of NF-DPFs:
	\begin{proposition}
		\label{prop:consistency}
		For a bounded weight function $\omega_t(x_t)=p(y_t|x_t;\theta):\mathcal{X}\rightarrow\mathbb{R}$ and a measurable bounded $k$-Lipschitz function $\psi(\cdot):\mathcal{X}\rightarrow\mathbb{R}$, when the regularization coefficient in entropy-regularized optimal transport resampler $\epsilon_N=o(1/\log{N})$, there exist constants $c_t$ and ${c'}_t$ such that for $t\geq0$
		\begin{align}
			\label{eq:predictive_true_exp}
			\mathbb{E}\Bigg[\bigg({\alpha}_N^{(t)}(\psi)-{\beta}^{(t-1)}f(\psi)\bigg)^2\Bigg]\leq c_t\frac{||\psi||_\infty^2}{N^{1/2d_\mathcal{X}}}
		\end{align}
		(replacing ${\beta}^{(t-1)}f$ by the initial distribution $\pi(x_0,\theta)$ at time $t=0$ defined in Eq.~\eqref{eq:initial_dist}) and
		\begin{align}
			\label{eq:posterior_true_exp}
			\mathbb{E}\Biggl[\bigg(\beta_{N}^{(t)}(\psi)-\beta^{(t)}(\psi)\bigg)^2\Biggr]\leq {c'}_t\frac{||\psi||_\infty^2}{N^{1/2d_\mathcal{X}}}\,,
		\end{align}
		where $||\psi||_\infty:=\underset{x\in\mathcal{X}}{\sup}|\psi(x)|$ is the infinity norm of  $\psi(\cdot)$, ${\beta}^{(t)}$ and ${\alpha}_N^{(t)}$ are respectively defined by~Eqs.~\eqref{eq:notations_main} and~\eqref{eq:notations_empirical_main}, and $f(\cdot)$ is the transition kernel defined by Eq.~\eqref{eq:ssm_dynamic_model}.
	\end{proposition}
	The proof of Proposition~\ref{prop:consistency} can be found in Appendix~\ref{appendix:proof}. The results in Proposition~\ref{prop:consistency} show that the particle estimates given by the NF-DPF are consistent estimators if $\epsilon_N=o(1/\log{N})$, i.e. the estimation error vanishes when $N\rightarrow \infty$. The error bounds we derived converge at $\mathcal{O}_P(\frac{1}{N^{1/2d_\mathcal{X}}})$ and are tighter than those derived in~\cite{corenflos2021differentiable}, which converge at $\mathcal{O}_P(\frac{1}{N^{1/8d_\mathcal{X}}})$, where $\mathcal{O}_P$ denotes the "big O in probability" notation. Specifically, in~\cite{corenflos2021differentiable} the upper bound on $\wass_2(\beta^{(t)}, \beta_N^{(t)})$ is adopted as the upper bound on the approximation error. As a result, the upper bound derived in~\cite{corenflos2021differentiable} is loose since the gap between $\wass_2(\beta^{(t)}, \beta_N^{(t)})$ and the approximation error is included in the upper bound. The tighter bound presented in Proposition~\ref{prop:consistency} is achieved by directly deriving an upper bound for $\mathbb{E}\Bigg[\bigg({\alpha}_N^{(t)}(\psi)-{\beta}^{(t-1)}f(\psi)\bigg)^2\Bigg]$ and $\mathbb{E}\Biggl[\bigg(\beta_{N}^{(t)}(\psi)-\beta^{(t)}(\psi)\bigg)^2\Biggr]$ without relying on the upper bound on the Wasserstein distances. However, compared with traditional particle filters~\cite{chopin2020introduction}, the error bound is loose due to the use of entropy-regularized optimal transport resampling, as we revealed in the proof.  We provide background knowledge about optimal transport and related notations in Appendix~\ref{appendix:ot}. 

	\section{Experiments}
	\label{sec:experiments}
	In this section, we present experimental results to compare the performance of the proposed normalizing flow-based differentiable particle filters (NF-DPFs) with other DPF variants\footnote{Code to reproduce experiment results is available at \url{https://github.com/xiongjiechen/Normalizing-Flows-DPFs}.}. We consider in Section~\ref{subsec:lgssm} a one-dimensional linear Gaussian state-space model similar to an example used in~\cite{naesseth2018variational}. In Section~\ref{subsec:multi_lgssm} we evaluate the performance of the proposed method on a multivariate linear Gaussian state-space model similar to the one used in~\cite{corenflos2021differentiable} with varying dimensionalities. In Section~\ref{subsec:disk} we compare the performance of NF-DPFs with other state-of-the-art DPFs in a synthetic visual tracking task following the setup in~\cite{haarnoja2016backprop,kloss2021train,chen2022conditional}. The experimental results on a simulated robot localization task from~\cite{jonschkowski18,corenflos2021differentiable,beattie2016deepmind} are reported in Section~\ref{subsec:exp_robot}. For all experiments presented in this section, the entropy-regularized optimal transport resampler~\cite{corenflos2021differentiable} is applied in the resampling step. We use the Real-NVP and the conditional Real-NVP models as the default normalizing flows and conditional normalizing flows in the NF-DPF, except in the one-dimensional example, because Real-NVP models have to split the latent state dimension-wise.

    We compare NF-DPFs with five baseline methods in our experiments, the deep state-space model (Deep SSM), the autoencoder sequential Monte Carlo (AESMC) and the AESMC-bootstrap proposed in~\cite{naesseth2018variational}, the particle filter recurrent neural networks (PFRNNs)~\cite{ma2020particle}, and the particle filter networks (PFNets)~\cite{karkus2018particle}. The deep state-space model (Deep SSM) employs recurrent neural networks to infer the parameters of state-space models~\cite{rangapuram2018deep}. The AESMC-bootstrap uses Gaussian distributions to construct its dynamic model and measurement model and new particles are generated from the learned dynamic model. The proposal distributions of the AESMC are constructed by Gaussian distributions with parameters determined by observations. The PFRNN uses recurrent neural networks with observations and latent states as inputs to generate new particles. We use PFNets to denote the method proposed in~\cite{karkus2018particle} and its concurrent work~\cite{jonschkowski18}, which are bootstrap differentiable particle filters with dynamic models and measurement models built by vanilla neural networks. 
	\subsection{One-dimensional Linear Gaussian State-Space Models}
	\label{subsec:lgssm}
	\subsubsection{Experiment setup} We first consider a one-dimensional example as in~\cite{naesseth2018variational}, for which the goal is to learn the parameters $\theta^*:=[\theta_1^*, \theta_2^*]$ and a proposal distribution $q(x_{t}|x_{t-1}, y_{t};\phi)$ for the following linear Gaussian state-space model:
	\begin{gather}
		\label{eq:1_d_initial}
		x_0 \sim \mathcal{N}(0, 1)\,,\\
		\label{eq:1_d_transition}
		x_t|x_{t-1}\sim \mathcal{N}(\theta_1^*x_{t-1}, 1) \text{ for } t\geq1\,,\\
		\label{eq:1_d_measurement}
		{y}_{t}| x_t \sim \mathcal{N}(\theta_2^*x_{t}, 0.1) \text{ for } t\geq0\,.
	\end{gather}
	
	We adopt the evidence lower bound (ELBO) $\mathbb{E}[\log{p}(y_{0:T};\theta)]$ of the log marginal likelihood as the training objective as in~\cite{naesseth2018variational,le2018auto,maddison2017filtering}, and we approximate the ELBO through:
    \begin{gather}
    \label{eq:ELBO}
        \mathbb{E}[\log{p}(y_{0:T};\theta)] \approx \frac{1}{K}\sum^{K}_{k=1}\log \hat{p}(y^k_{T};\theta)\,,
    \end{gather}
    where $K=10$ is the number of training sequences, and $\log \hat{p}(y^k_{T};\theta)$ is computed as in Line 23 of Alg.~\ref{alg:nf_dpfs}. 
	We use a fixed learning rate of 0.002 and optimize the model for 500 iterations. At each iteration, we feed the model $K=10$ sequences of observations $y_{0:T}$ generated with $T=50$ and $\theta^*:=[\theta_1^*, \theta_2^*]=[0.9, 0.5]$. We also use 1000 sequences of observations as our validation set. The trained model is then tested with another 1000 observation sequences. We set the number of particles as $N=100$ for training, validation, and testing stages. Since the goal in this experiment is to simultaneously learn the model parameters $\theta:=[\theta_1, \theta_2]$ (initialized as [0.1, 0.1]) and proposal distributions $q(x_{t}|x_{t-1}, y_{t};\phi)$, the state-space model to be optimized is in the same form as the true model:
	\begin{gather}
		\label{eq:1_d_initial_learned}
		x_0 \sim \mathcal{N}(0, 1)\,,\\
		\label{eq:1_d_transition_learned}
		x_t|x_{t-1}\sim \mathcal{N}(\theta_1x_{t-1}, 1) \text{ for } t\geq1\,\\
		\label{eq:1_d_measurement_learned}
		{y}_{t}| x_t \sim \mathcal{N}(\theta_2x_{t}, 0.1) \text{ for } t\geq0\,,
	\end{gather}
	such that we can evaluate the difference between the true model parameters and the learned model parameters.
 
    The performance of trained models is evaluated based on four metrics: 
	\begin{itemize}
		\item The $L^2$-norm $||\theta-\theta^*||_2$ between the learned parameters $\theta:=[\theta_1, \theta_2]$ and true parameters $\theta^*:=[\theta_1^*, \theta_2^*]=[0.9, 0.5]$;
		\item The $L^2$-norm $||\bar{\chi}_T-\bar{\chi}_T^*||_2$ between the estimated posterior means $\bar{\chi}_T$ and the true posterior means $\bar{\chi}^*_T$ computed by Kalman filter, where $\bar{\chi}_T:=[\bar{x}_0, \bar{x}_1, \cdots, \bar{x}_T]$ and $\bar{\chi}_T^*:=[\bar{x}_0^*, \bar{x}_1^*, \cdots, \bar{x}_T^*]$;
		\item The ELBO defined by Eq.~\eqref{eq:ELBO}.
		\item The effective sample size.
	\end{itemize} 
	Lower $||\theta-\theta^*||_2$, $||\bar{\chi}_T-\bar{\chi}_T^*||_2$, and higher ELBO and effective sample size indicate better performance of evaluated models.
	
	\subsubsection{Experimental results}
	We report in Fig.~\ref{fig:1_d} and Table~\ref{tab:1_d} the evaluated metrics for different methods on the validation set and the test set, respectively. We observe that the PFNet can achieve arbitrarily large ELBOs but produces poor tracking performance, so we do not report the performance of the PFNet in this experiment. The main reason for this is that the measurement model of the PFNet is constructed with vanilla neural networks, and the ELBO can be increased by simply amplifying the magnitude of the neural network's output without learning the relationship between observations and states. 
	
	For the ELBO, all methods converge to almost the same validation ELBO as demonstrated in Fig.~\ref{fig:1_d_c}. The AESMC-bootstrap converges the fastest. We speculate that this is because it does not have proposal parameters to learn, so it can focus on learning model parameters. This is also reflected in Fig.~\ref{fig:1_d_a} and the second column of Table~\ref{tab:1_d}, where we can observe that all approaches produce similar parameter estimation errors after their convergence, and the AESMC-bootstrap has the highest convergence rate but exhibits slightly larger parameter estimation error than other methods. The NF-DPF converges faster than the AESMC and the PFRNN has the highest ELBO and the lowest parameter estimation error. 
 
    We report posterior mean errors $||\bar{\chi}_T-\bar{\chi}_T^*||_2$ evaluated on the validation set and the test set for different methods in Fig.~\ref{fig:1_d_b} and Table~\ref{tab:1_d} respectively. The AESMC-bootstrap leads to the highest estimation error as expected. The AESMC produces better results compared to the AESMC-bootstrap. The PFRNN reports the second lowest test posterior mean error after the convergence. Deep SSM has similar performance to PFRNN on both validation and test datasets. The NF-DPF outperforms all the compared baselines regarding both the convergence rate and the validation error when training has converged. The NF-DPF leads to the highest effective sample size among all the evaluated methods.
	
	\begin{figure}[t!]
		\centering
		\begin{subfigure}{0.475\linewidth}
			\centering
			\includegraphics[width=\linewidth]{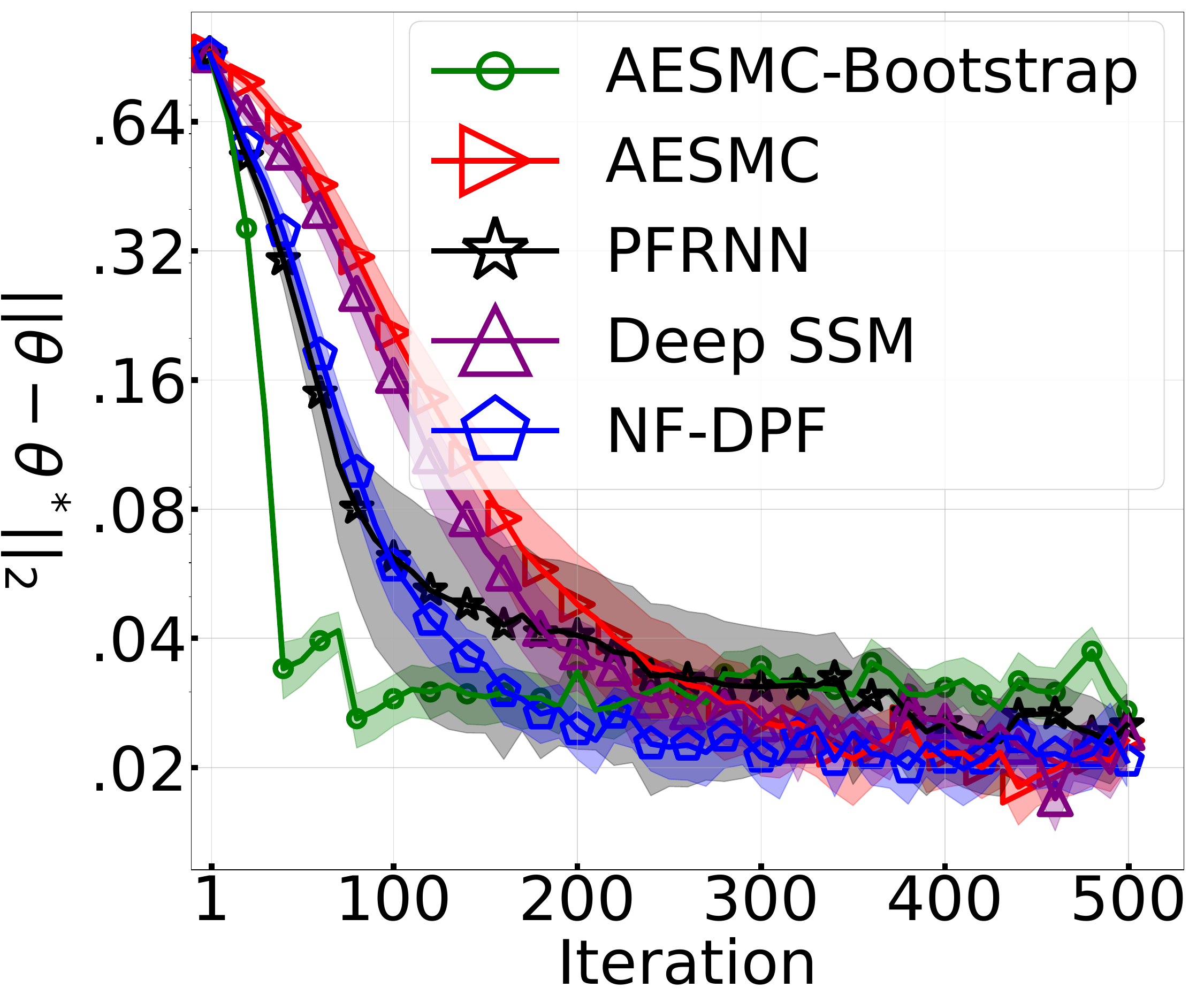}
			\caption{}
			\label{fig:1_d_a}
		\end{subfigure}
		\begin{subfigure}{0.475\linewidth}
			\centering
			\includegraphics[width=\linewidth]{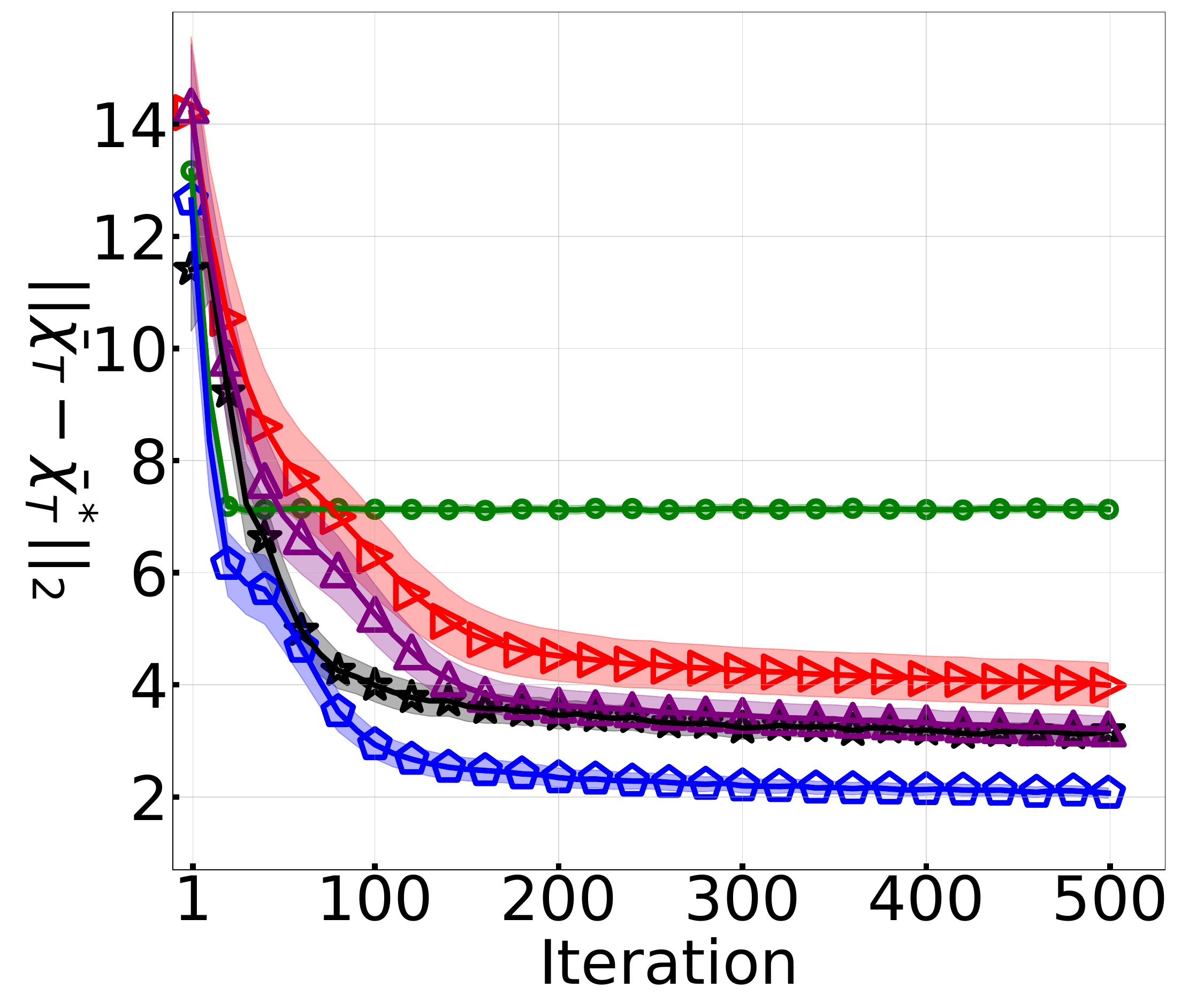}
			\caption{}
			\label{fig:1_d_b}
		\end{subfigure}
		\begin{subfigure}{0.475\linewidth}
			\centering
			\includegraphics[width=\linewidth]{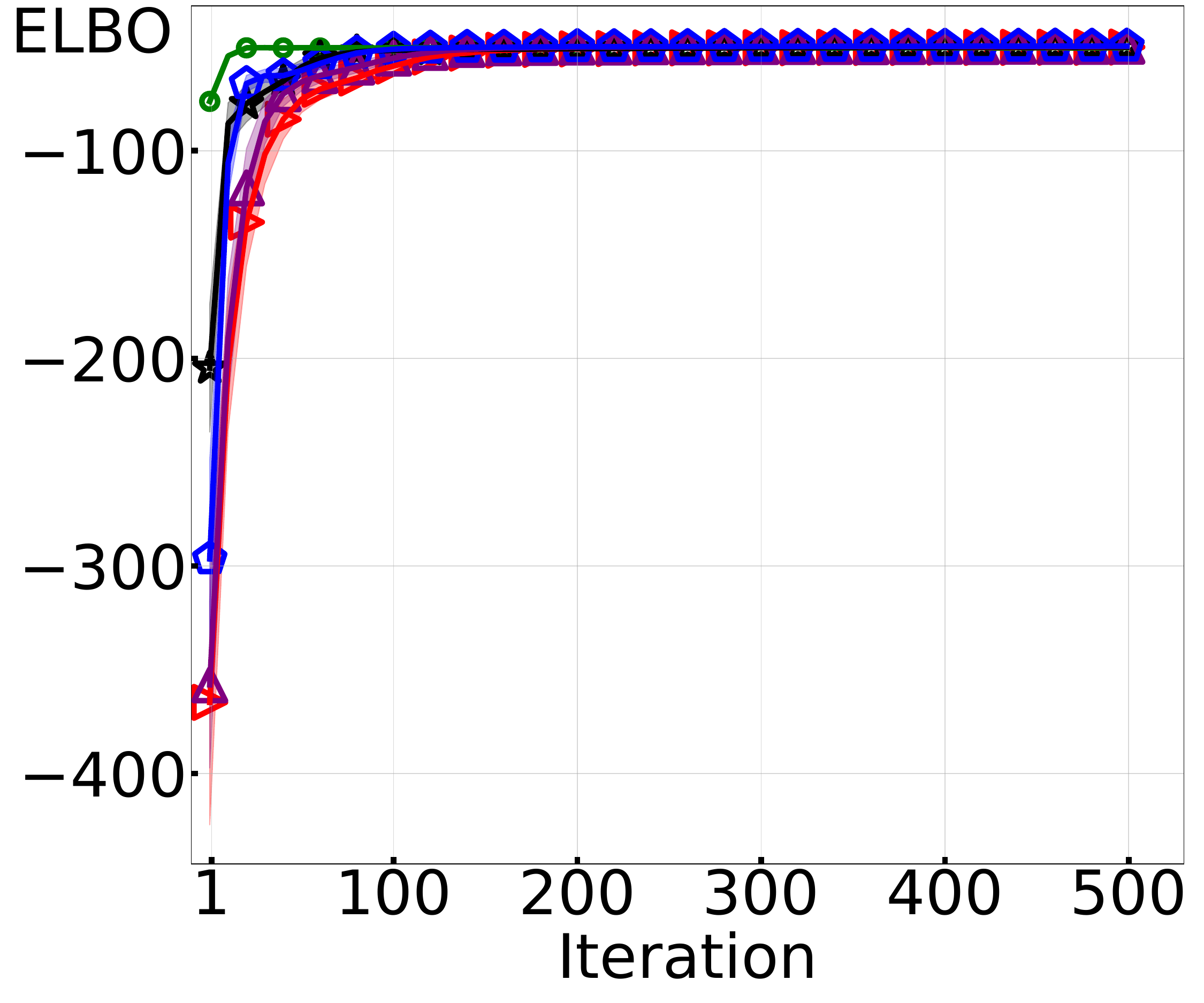}
			\caption{}
			\label{fig:1_d_c}
		\end{subfigure}
		\begin{subfigure}{0.475\linewidth}
			\centering
			\includegraphics[width=\linewidth]{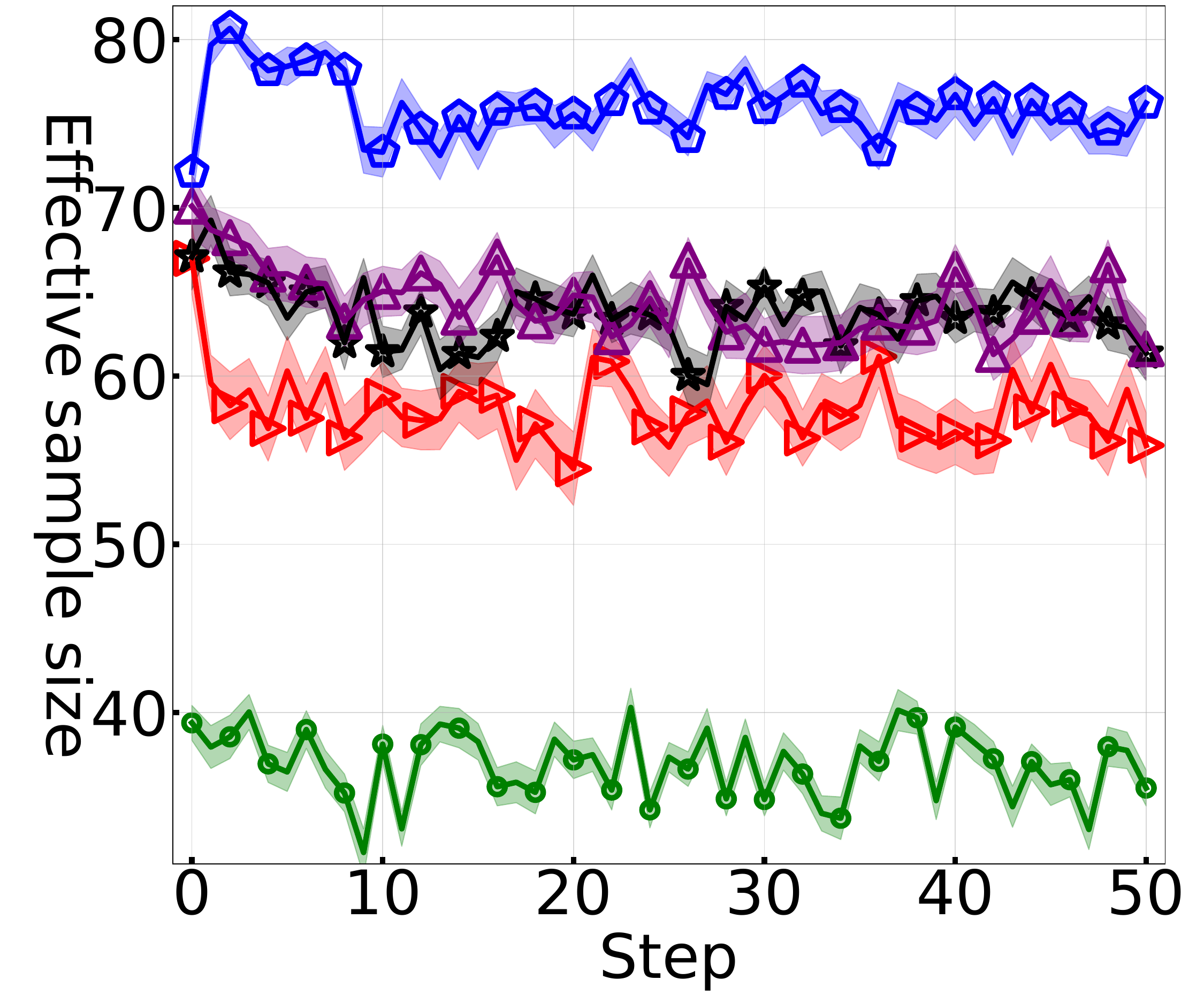}
			\caption{}
			\label{fig:1_d_d}
		\end{subfigure}
		\caption{Evaluation metrics of different methods evaluated on the validation set with 1000 sequences. (a) $L^2$-norm between the true parameter set and the estimated parameter sets. (b) $L^2$-norm of posterior mean error evaluated on validation set. (c) ELBO evaluated on validation set. (d) Effective sample size on validation set. Lower parameter estimation error, posterior mean error, higher effective sample size, and ELBO indicate better performance. The shaded area represents the standard deviation of the presented evaluation metrics among 50 random simulations.}
		\label{fig:1_d}
	\end{figure}
	
	\begin{table*}[htbp!]
		\caption{Evaluation metrics of different methods evaluated on the test set with 1000 sequences. Lower parameter estimation error, posterior mean error, higher effective sample size, and ELBO indicate better performance. The reported parameter estimation error, posterior mean error, and ELBO are computed with the model saved at the last iteration. The average effective sample size is the mean of effective sample sizes at each time step. The reported mean and standard deviation are computed with 50 random runs. The running time per training iteration is computed based on a computer with an 13th Gen Intel(R) Core(TM) i9-13900KF CPU 3000Mhz and 64GB of RAM, and a RTX 4090 graphic card with 24GB memories.}

		\centering
		\begin{small}
			\begin{tabular}{?c??c?c?c?c?c?}
				\Xhline{3\arrayrulewidth}
				{Method} & $||\theta-\theta^*||_2$ $\downarrow$& $||\bar{\chi}_T-\bar{\chi}_T^*||_2$$\downarrow$& \makecell{Average effective\\ sample size} $\uparrow$& ELBO $\uparrow$ & \makecell{Training time\\ (s/it)} \\ 
                \hline \hline
				Deep SSM~\cite{le2018auto} &  $0.0234 \pm 0.0121$ &$3.20 \pm 0.389$ &  $63.7 \pm 8.52$ & $-50.3 \pm 1.68$&0.189\\
                \hline
				AESMC-Bootstrap~\cite{le2018auto} &  $0.0271 \pm 0.0133$&$7.13 \pm 0.233$ &  $36.9 \pm 1.52$ & $-50.4 \pm 1.72$&0.182\\ \hline
				AESMC~\cite{le2018auto} & $0.0231\pm 0.0108$&$3.99\pm 1.316$ &  $58.1 \pm 14.24$ &$-50.1 \pm 1.51$&0.192\\ \hline
				PFRNN~\cite{ma2020particle}  & ${0.0251}\pm {0.015}$&${3.18}\pm {0.456}$ &  ${63.8} \pm {11.11}$ &${-50.1} \pm {1.94}$&0.204\\ \hline
				NF-DPF  & $\textbf{0.0207}\pm \textbf{0.0083}$&$\textbf{2.07}\pm \textbf{0.304}$ &  $\textbf{76.0} \pm \textbf{6.26}$ &$\textbf{-49.6} \pm \textbf{1.61}$&0.231\\  \Xhline{3\arrayrulewidth}
			\end{tabular}
		\end{small}
		\label{tab:1_d}
	\end{table*}

	\subsection{Multivariate Linear Gaussian State-Space Models}
	\label{subsec:multi_lgssm}
	\subsubsection{Experiment setup}
	In this experiment, we extend the one-dimensional example in Section~\ref{subsec:lgssm} to higher dimensional-spaces to evaluate the performance of the NF-DPF. Following the setup in~\cite{corenflos2021differentiable}, we consider a similar multivariate linear Gaussian state-space model as below:
	\begin{gather}
		\label{eq:multi_initial}
		x_0 \sim \mathcal{N}(\mathbf{0}_{d_\mathcal{X}}, \mathbf{I}_{d_\mathcal{X}})\,,\\
		\label{eq:multi_transition}
		x_t|x_{t-1}\sim \mathcal{N}(\boldsymbol{\theta}_1^*x_{t-1},  \mathbf{I}_{d_\mathcal{X}})\text{ for } t\geq1\,,\\
		\label{eq:multi_measurement}
		{y}_{t}| x_t \sim \mathcal{N}(\boldsymbol{\theta}_2^*x_{t}, 0.1\mathbf{I}_{d_\mathcal{X}}) \text{ for } t\geq0\,,
	\end{gather}
	where $\mathbf{0}_{d_\mathcal{X}}$ is a $d_\mathcal{X}\times d_\mathcal{X}$ null matrix, $\mathbf{I}_{d_\mathcal{X}}$ is a $d_\mathcal{X}\times d_\mathcal{X}$ identity matrix, the element of $\boldsymbol{\theta}_1^*$ at the intersection of its $i$-th row and $j$-th column $\boldsymbol{\theta}_1^*(i,j)=(0.42^{|i-j|+1})_{1\leq i,j \leq d_\mathcal{X}}$, $\boldsymbol{\theta}_2^*$ is a $d_\mathcal{Y}\times d_\mathcal{X}$ matrix with 0.5 on the diagonal for the first $d_\mathcal{Y}$ rows and zeros elsewhere. We set $d_\mathcal{X}=d_\mathcal{Y}$ in this experiment. We again want to learn model parameters $\theta^*:=[\theta_1^*, \theta_2^*]$ and proposal distributions by maximizing the evidence lower bound (ELBO) as in Section~\ref{subsec:lgssm}. We also use the same hyperparameter setting in Section~\ref{subsec:lgssm}, and train the compared models with 5000 sequences for 500 iterations (10 sequences for each iteration). Validation and test sets contain 1000 i.i.d sequences each. Model parameters $\theta:=[\theta_1, \theta_2]$ to be optimized are initialized as $[0.1\times\mathbf{I}_{d_\mathcal{X}}, 0.1\times\mathbf{I}_{d_\mathcal{X}}]$. 

	\subsubsection{Experimental results}
	In Fig.~\ref{fig:multi_lgssm} and Table~\ref{tab:multivariate}, we show the test performance of NF-DPFs, the Deep SSM, the PFRNN, and the AESMC in $d_\mathcal{X}$-dimensional spaces for $d_\mathcal{X}\in\{2, 5, 10, 25, 50, 100\}$. The AESMC-bootstrap particle filter is excluded in this experiment because its estimation error is too large to be compared with the other methods in the same figure. The evaluation metrics reported in Section~\ref{subsec:lgssm} are used in this experiment as well. 
 
    For model parameters learning, from Fig.~\ref{fig:multi_a}, Fig.~\ref{fig:multi_c}, and Table~\ref{tab:multivariate}, we found that NF-DPFs produced the highest ELBOs in 4 out of 6 setups, and all the evaluated methods achieved similar parameter estimation errors. We also observed that higher ELBOs and lower parameter estimation errors do not necessarily correspond to better posterior approximation errors as we can see from Fig.~\ref{fig:multi_lgssm} and Table~\ref{tab:multivariate}. 
    Specifically, from Fig.~\ref{fig:multi_b} and Table~\ref{tab:multivariate}, NF-DPFs outperform the compared baselines in terms of posterior approximation errors by consistently leading to the lowest posterior mean errors and the highest effective sample sizes for $d_\mathcal{X}\in\{2, 5, 10, 25, 50, 100\}$. 
	
	\begin{figure}[htbp]
		\centering
		\begin{subfigure}{0.475\linewidth}
			\centering
			\includegraphics[width=\linewidth]{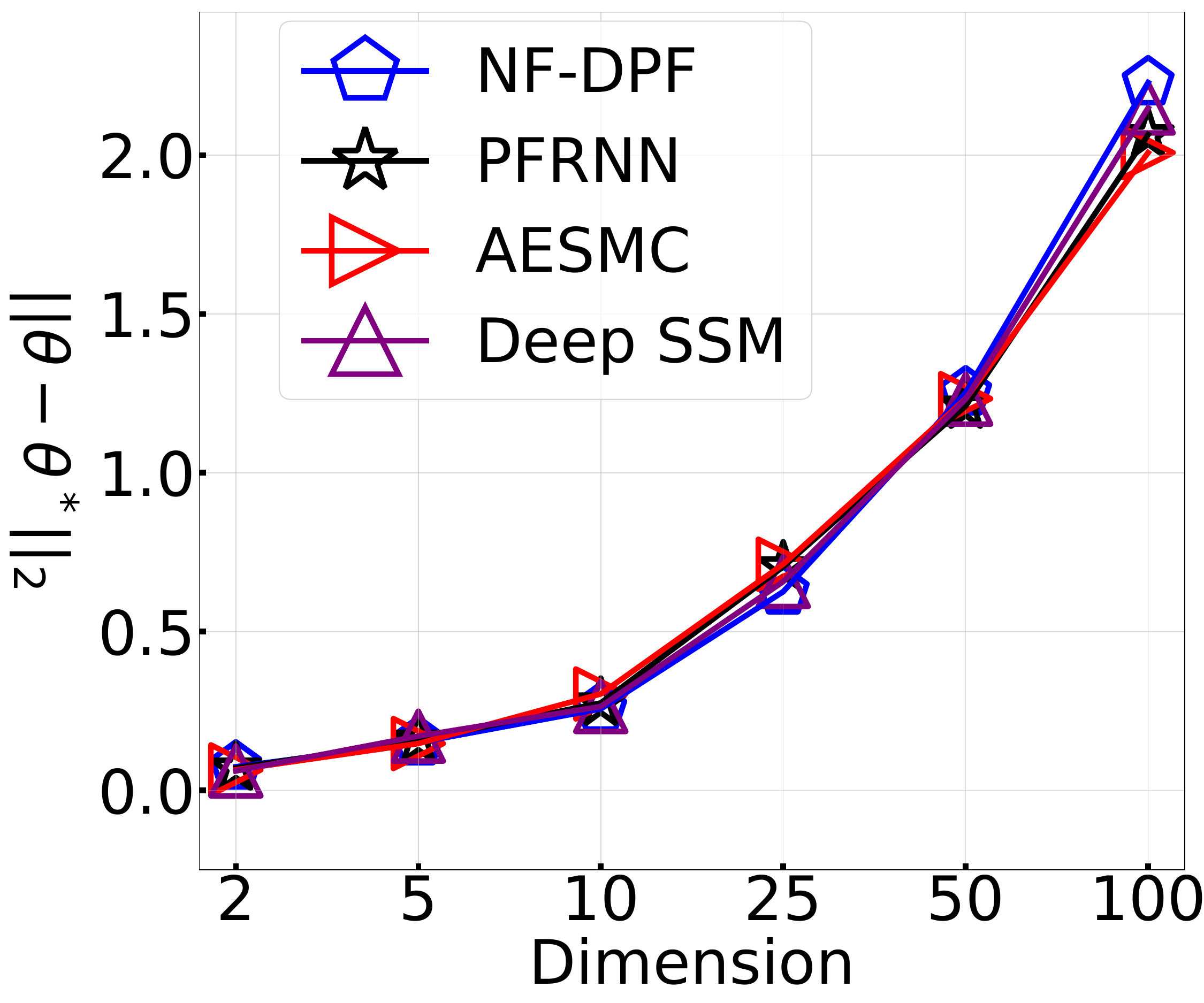}
			\caption{}
			\label{fig:multi_a}
		\end{subfigure}
		\begin{subfigure}{0.475\linewidth}
			\centering
			\includegraphics[width=\linewidth]{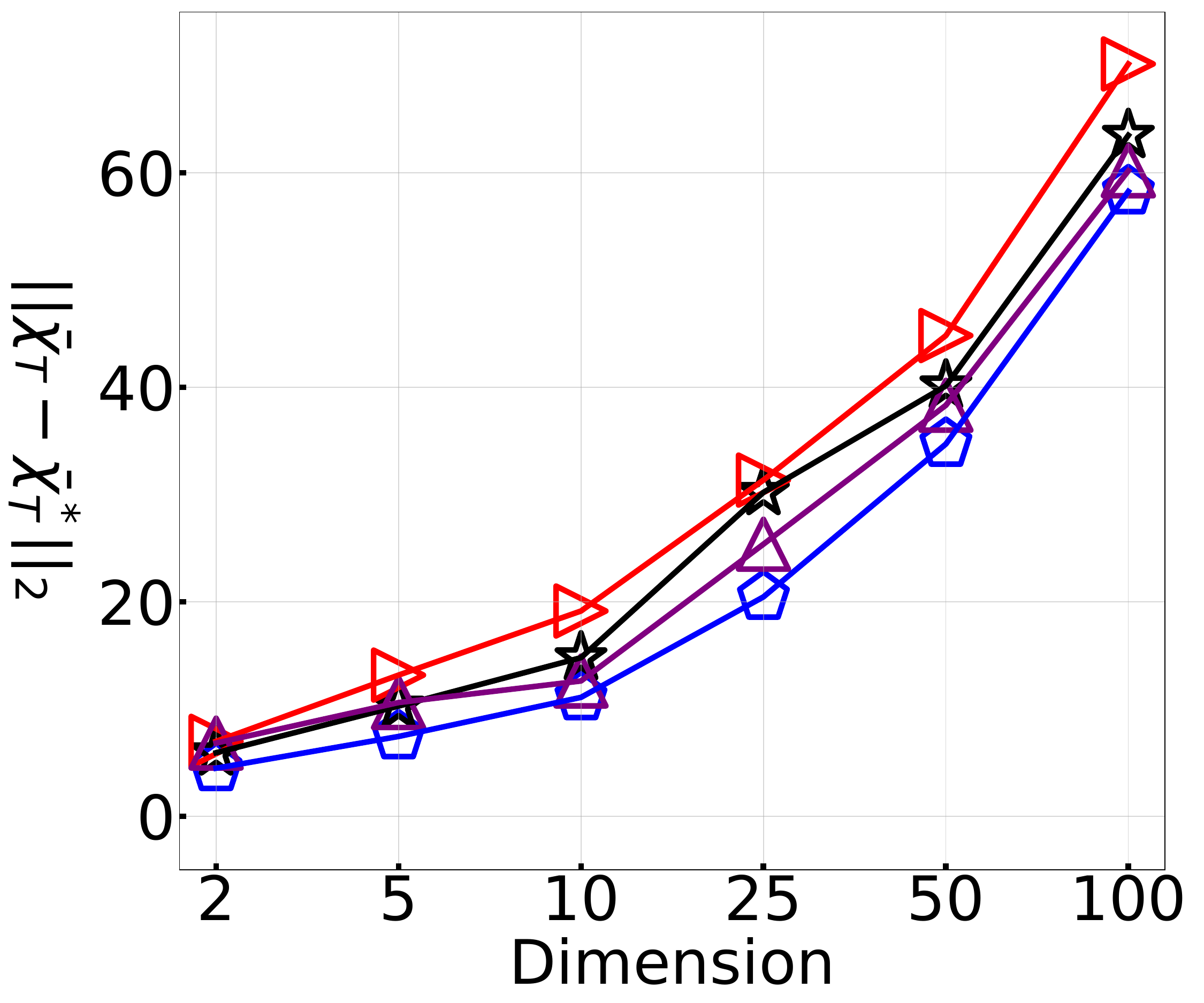}
			\caption{}
			\label{fig:multi_b}
		\end{subfigure}
		\begin{subfigure}{0.475\linewidth}
			\centering
			\includegraphics[width=\linewidth]{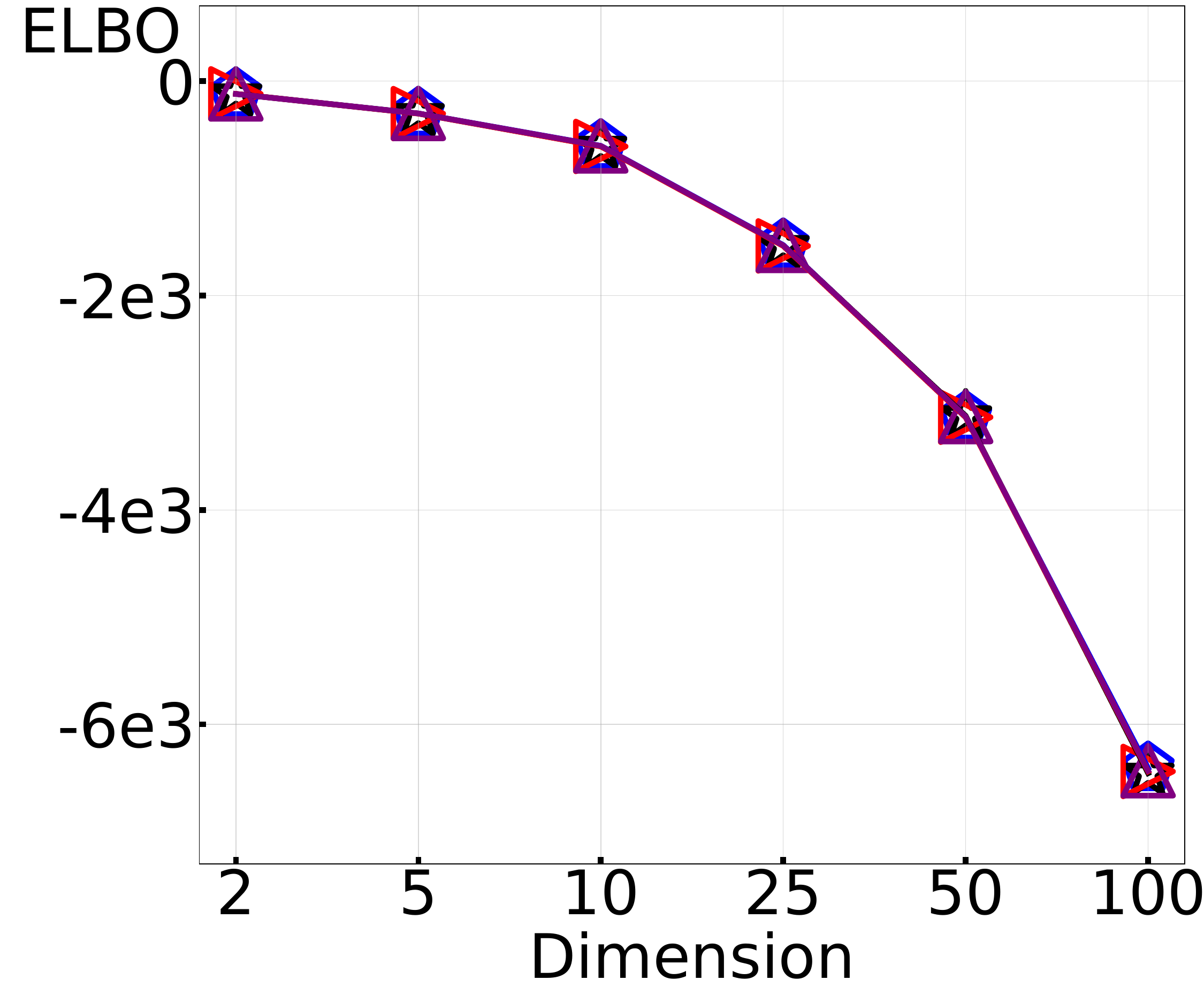}
			\caption{}
			\label{fig:multi_c}
		\end{subfigure}
		\begin{subfigure}{0.475\linewidth}
			\centering
			\includegraphics[width=\linewidth]{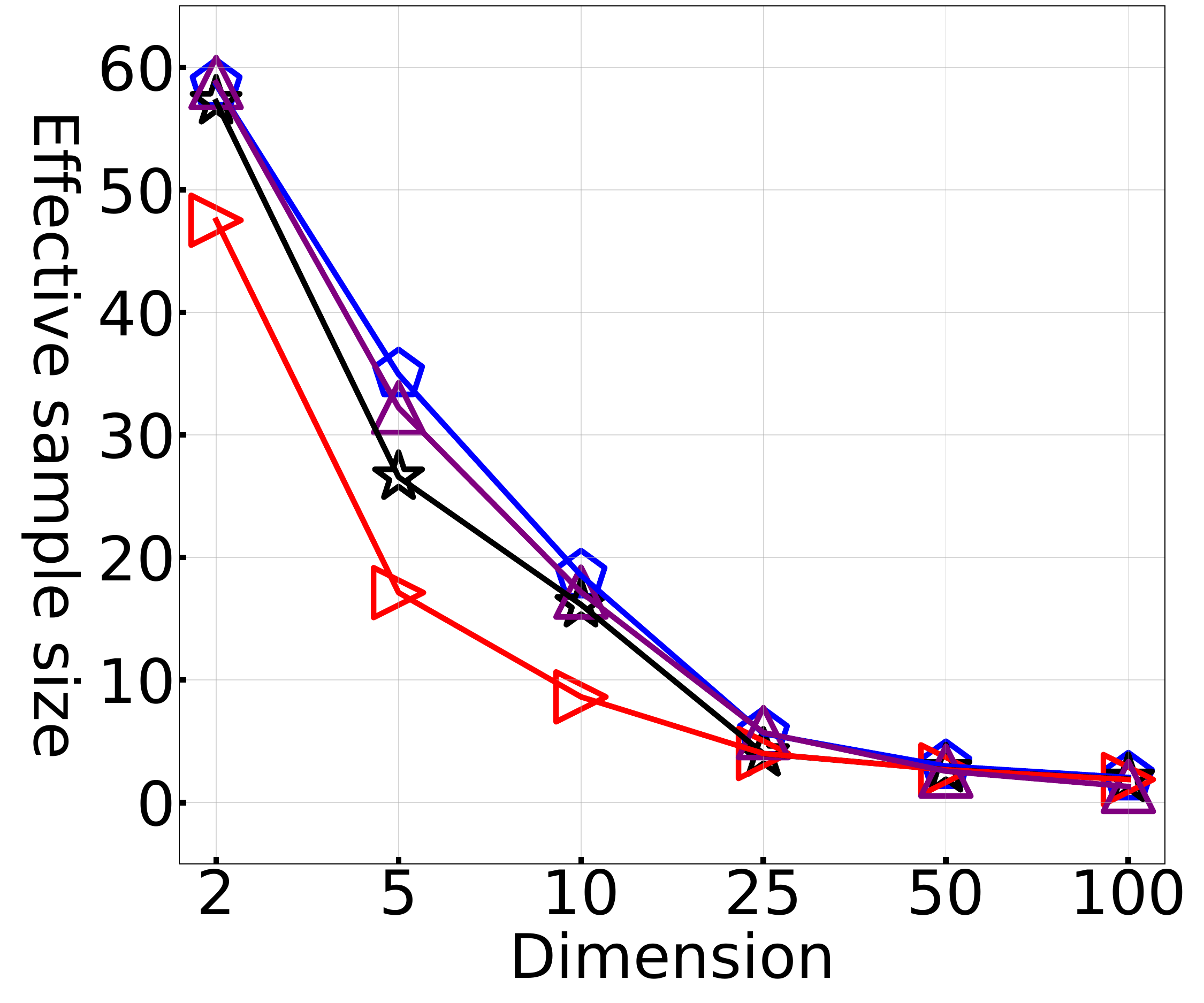}
			\caption{}
			\label{fig:multi_d}
		\end{subfigure}
		\caption{Evaluation metrics of different methods evaluated on a test set with 1000 sequences. (a) $L^2$-norm between the true parameter set and the estimated parameter sets. (b) $L^2$-norm of posterior mean error evaluated on test set. (c) ELBO evaluated on test set. (d) Effective sample size on test set. Lower parameter estimation error, posterior mean error, higher effective sample size, and ELBO indicate better performance. The reported results are the mean of evaluation metrics computed over 50 random simulations.}
		\label{fig:multi_lgssm}
	\end{figure}
	
	\begin{table*}[htbp]
		\caption{Evaluation metrics of different methods evaluated on a test set with 1000 sequences. Lower parameter estimation error, posterior mean error, higher effective sample size, and ELBO indicate better performance. The reported parameter estimation error, posterior mean error, and ELBO are computed with the model saved at the last iteration. The average effective sample size is the mean of effective sample size at each time step. The reported mean and standard deviation is computed with 50 random runs.}
		\centering
		\begin{scriptsize}
			\begin{tabular}{?c?c??c?c?c?c?c?}
				\Xhline{3\arrayrulewidth}
				Dimension&{Method} & $||\theta-\theta^*||_2$ $\downarrow$& $||\bar{\chi}_T-\bar{\chi}_T^*||_2$$\downarrow$& \makecell{Average effective\\ sample size} $\uparrow$& ELBO $\uparrow$& \makecell{Training time\\ (s/it)} \\ \hline \hline 
				\multirow{3}{0.1\linewidth}{$D=2$}&Deep SSM~\cite{rangapuram2018deep}& $\textbf{0.061}\pm \textbf{0.0229}$&$6.83\pm 1.92$ & \textbf{58.8 $\pm$ 6.01}&$-120.3 \pm 1.30$& 0.279\\ \cline{2-7}
                    &AESMC~\cite{le2018auto}& ${0.065}\pm {0.0226}$&$6.85\pm 1.16$ &  $47.5 \pm 12.51$ &$-119.3 \pm 2.80$& 0.272\\ \cline{2-7}
                    &PFRNN~\cite{ma2020particle}& {${0.071}\pm {0.0181}$}& ${5.82} \pm {1.45}$ &  ${57.2} \pm {10.18}$ & ${-119.8} \pm {2.67}$& 0.317\\ \cline{2-7}
                    &NF-DPF & $0.075\pm 0.0120$&$\textbf{4.48}\pm \textbf{0.84}$ &  ${58.6} \pm {6.46}$ &$\textbf{-119.2} \pm \textbf{1.88}$&0.361\\ \Xhline{3\arrayrulewidth}
				
				\multirow{3}{0.1\linewidth}{$D=5$}&Deep SSM~\cite{rangapuram2018deep}& ${0.171}\pm {0.0172}$&$10.68\pm 1.48$ & {32.2 $\pm$ 4.87}&$-301.8 \pm 3.52$& 0.319\\  \cline{2-7}
                    &AESMC~\cite{le2018auto}& $\textbf{0.148}\pm \textbf{0.0120}$&$13.25\pm 1.62$ &  $17.1 \pm 2.83$ &$-303.3 \pm 2.64$&0.309\\ \cline{2-7}
                    &PFRNN~\cite{ma2020particle}& {${0.158}\pm {0.0156}$}& ${10.27} \pm {0.40}$ &  ${26.6} \pm {2.60}$ & ${-302.6} \pm {3.19}$&0.328\\ \cline{2-7}
                    &NF-DPF & $0.151\pm 0.0105$&$\textbf{7.44}\pm \textbf{1.55}$ &  $\textbf{34.9} \pm \textbf{5.29}$ &$\textbf{-301.6} \pm \textbf{3.25}$&0.371\\ \Xhline{3\arrayrulewidth}

				\multirow{3}{0.1\linewidth}{$D=10$}&Deep SSM~\cite{rangapuram2018deep}& ${0.264}\pm {0.0292}$&$12.64\pm 0.82$ & {17.1 $\pm$ 2.98}&$-604.8 \pm 4.85$& 0.351\\  \cline{2-7}
                    &AESMC~\cite{le2018auto}& $0.304\pm 0.0214$&$19.17\pm 1.88$ &  $8.6 \pm 1.73$ &$-609.7 \pm 6.24$&0.327\\ \cline{2-7}
                    &PFRNN~\cite{ma2020particle}& {${0.275}\pm {0.0148}$}& ${14.89}\pm {0.32}$ &  ${16.2} \pm {1.24}$ & ${-608.5} \pm {7.18}$&0.357\\ \cline{2-7}
                    &NF-DPF & $\textbf{0.257}\pm \textbf{0.0116}$&$\textbf{11.18}\pm \textbf{1.28}$ &  $\textbf{18.5} \pm \textbf{2.37}$ &$\textbf{-603.3} \pm \textbf{5.02}$&0.385\\ \Xhline{3\arrayrulewidth}

				\multirow{3}{0.1\linewidth}{$D=25$}&Deep SSM~\cite{rangapuram2018deep}& ${0.657}\pm {0.0372}$&$25.41\pm 1.35$ & {5.7 $\pm$ 0.79}&$-1532.7 \pm 8.62$& 0.351\\  \cline{2-7}
                    &AESMC~\cite{le2018auto}& $0.712\pm 0.0191$&$31.46\pm 2.12$ &  $4.0\pm 0.22$ &$-1537.5 \pm 12.14$&0.348\\ \cline{2-7}
                    &PFRNN~\cite{ma2020particle}& {${0.704}\pm {0.0312}$}& ${30.21} \pm {1.03}$ &  ${4.0} \pm {0.39}$ & ${-1534.2} \pm {6.46}$&0.378\\  \cline{2-7}
                    &NF-DPF & $\textbf{0.626}\pm \textbf{0.0188}$&$\textbf{20.51}\pm \textbf{1.21}$ &  $\textbf{5.6} \pm \textbf{0.45}$ &$\textbf{-1529.0} \pm \textbf{12.53}$&0.412\\
				\Xhline{3\arrayrulewidth}
				
				\multirow{3}{0.1\linewidth}{$D=50$}&Deep SSM~\cite{rangapuram2018deep}& ${1.23}\pm {0.0168}$&$38.33\pm 0.58$ & {2.6 $\pm$ 0.65}&$-3128.6 \pm 16.40$& 0.387\\  \cline{2-7}
                    &AESMC~\cite{le2018auto}& $1.233\pm 0.0260$&$44.83\pm 0.92$ &  $2.7 \pm 0.37$ &$-3135.0 \pm 19.56$&0.382\\ \cline{2-7}
                    &PFRNN~\cite{ma2020particle}& {$\textbf{1.210}\pm \textbf{0.0222}$}& ${40.14} \pm {0.61}$ &  ${2.7} \pm {0.13}$ & $\textbf{-3121.7} \pm \textbf{6.39}$&0.442\\ \cline{2-7}
                    &NF-DPF & $1.252\pm 0.0189$&$\textbf{34.73}\pm \textbf{1.81}$ &  $\textbf{3.0} \pm \textbf{0.11}$ &${-3135.4} \pm {28.53}$&0.482\\ \Xhline{3\arrayrulewidth}

				\multirow{3}{0.1\linewidth}{$D=100$}&Deep SSM~\cite{rangapuram2018deep}& ${2.158}\pm {0.0531}$&$60.19\pm 1.69$ & {1.3 $\pm$ 0.26}&$\textbf{-6431.8} \pm \textbf{69.67}$& 0.501\\  \cline{2-7}
                    &AESMC~\cite{le2018auto}& $\textbf{2.007}\pm \textbf{0.0255}$&$70.16\pm 2.23$ &  $1.9 \pm 0.28$ &${-6438.1} \pm {34.80}$&0.478\\ \cline{2-7}
                    &PFRNN~\cite{ma2020particle}& {${2.064}\pm {0.0272}$}& ${63.51} \pm {2.78}$ &  ${1.9} \pm {0.47}$ & ${-6454.8} \pm {52.32}$&0.501\\ \cline{2-7}
                    &NF-DPF & $2.228\pm 0.0419$&$\textbf{58.27}\pm \textbf{1.45}$ &  $\textbf{2.0} \pm \textbf{0.18}$ &${-6447.4} \pm {80.01}$&0.581\\ \Xhline{3\arrayrulewidth}

			\end{tabular}
		\end{scriptsize}
		\label{tab:multivariate}
	\end{table*}

	\subsection{Disk Localization}
	\label{subsec:disk}
	\subsubsection{Experiment setup}
	We consider in this experiment a disk localization task, where the goal is to locate a moving red disk based on observation images. Specifically, an observation image is a 128$\times$128 RGB image that contains 25 disks, including the red disk and 24 distracting disks with varying sizes and colors, and such an observation image is given at each time step. The colors of distracting disks are uniformly drawn with replacement from the set of \{green, blue, cyan, purple, yellow, white\}, and the radii of them are uniformly sampled with replacement from \{3, 4,$\cdots$, 10\}. The radius of the target, i.e. the red disk, is set to be 7. The initial locations of the 25 disks are uniformly distributed over the observation image as shown in Fig.~\ref{fig:disk_image}.

 Following the setup in~\cite{chen2022conditional,corenflos2021differentiable,chen2021differentiable}, we use a combination of two loss functions as our training objective:
 \begin{gather}
 \label{eq:loss_disk}
     \mathcal{L}(\theta,\phi):=\mathcal{L}_{\text{RMSE}}(\theta,\phi)+\mathcal{L}_{\text{AE}}(\theta)\,,
 \end{gather}
 where $\mathcal{L}_{\text{RMSE}}(\theta,\phi)$ is the root mean square error (RMSE) between the estimated location $\bar{{x}}_t$ and the ground truth location ${x}_t^*$ of the red disk
	\begin{gather}
		\label{eq:loss_rmse}
		\mathcal{L}_{\text{RMSE}}(\theta,\phi):=\sqrt{\frac{1}{T}\sum_{t=0}^{T}||\bar{{x}}_t-{x}_t^*||_2^2}\,,
	\end{gather}
	and $\mathcal{L}_{\text{AE}}(\theta)$ is the autoencoder reconstruction loss of observation images as defined in Eq.~\eqref{eq:loss_ae}. We use an Adam optimizer~\cite{kingma2015adam} with a learning rate of 0.001 to minimize the overall loss function $\mathcal{L}(\theta,\phi)$. 
	
	The dynamic system used for generating training, validation, and test sets for this experiment follows the setup in~\cite{kloss2021train}. The training set we use to optimize DPFs contains 500 trajectories, each with 50 time steps, and both the validation and test sets are composed of 50 trajectories with the same length as the training trajectories.
	
	The performance of different DPFs is evaluated by the RMSE between estimated locations and ground truth locations of the tracking objective. We report both the test RMSE and the validation RMSE to investigate the tracking performance of different DPFs during and after training.
	\subsubsection{Experimental results}
	\begin{figure}[htbp]
		\centering
		\begin{subfigure}{0.475\linewidth}
			\centering
			\includegraphics[width=0.925\linewidth, height=0.8\linewidth]{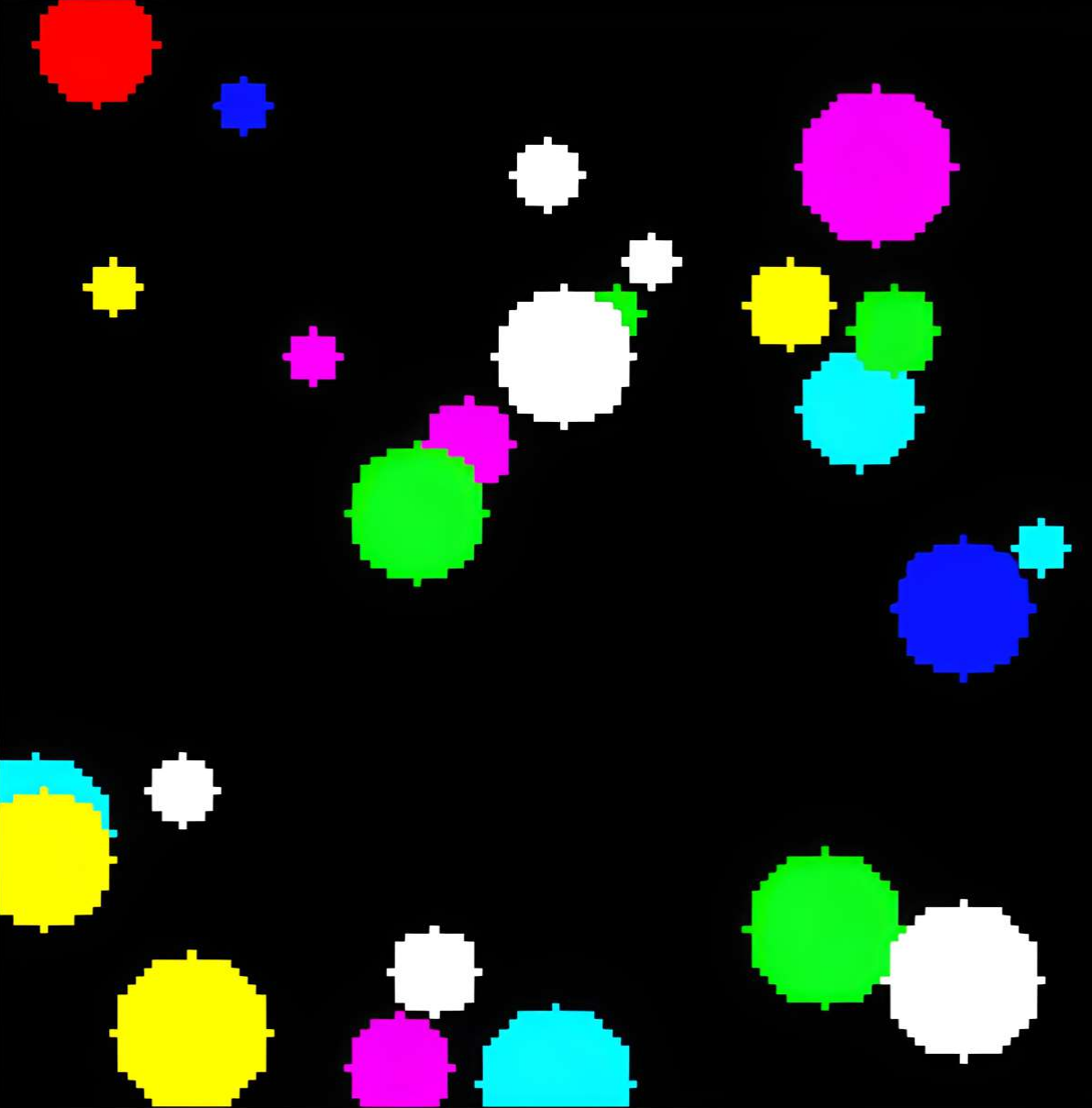}
			\caption{}
			\label{fig:disk_image}
		\end{subfigure}
		\begin{subfigure}{0.475\linewidth}
			\centering
			\includegraphics[width=\linewidth]{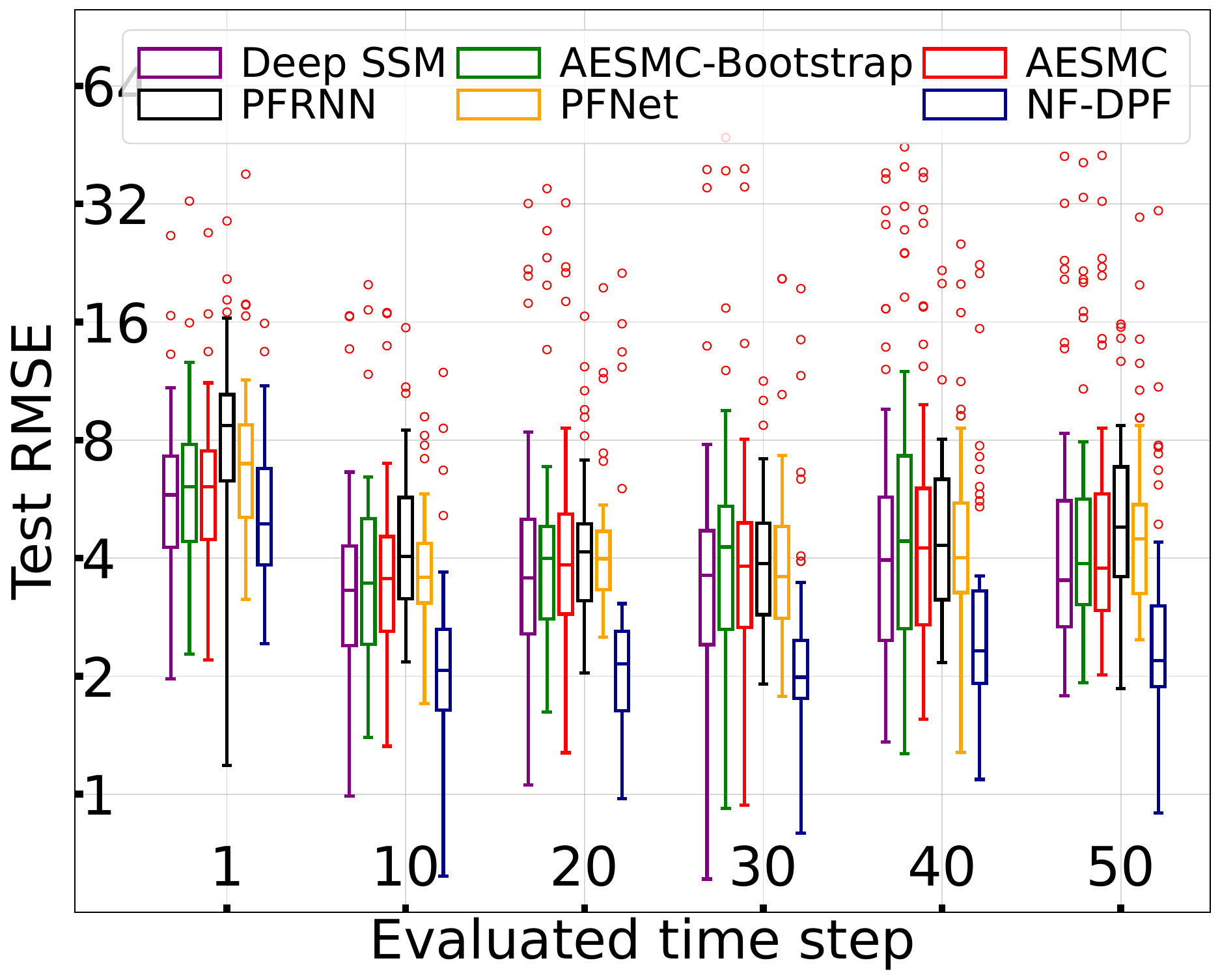}
			\caption{}
			\label{fig:disk_results_1}
		\end{subfigure}
		\begin{subfigure}{0.475\linewidth}
			\centering
			\includegraphics[width=\linewidth]{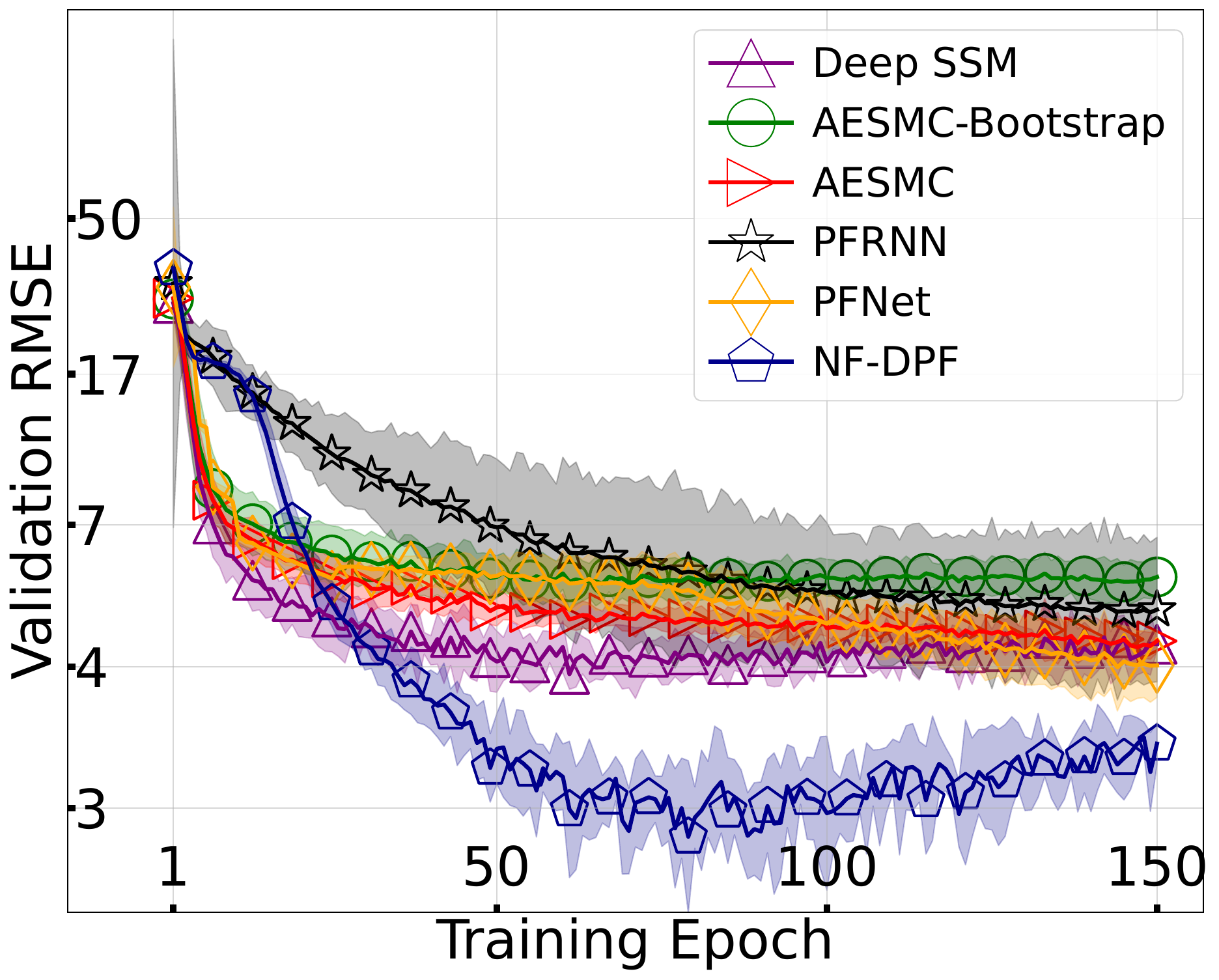}
			\caption{}
			\label{fig:disk_results_3}
		\end{subfigure}
		\caption{(A) An example of observation images. (B) RMSE of different methods evaluated at selected time steps on test set. (C) RMSE of different differentiable particle filters on the validation set during training. Shaded areas represent the standard deviation of the presented evaluation metrics among 5 random simulations.}
		\label{fig:disk_results}
	\end{figure}

    \begin{figure}
        \centering
        \begin{subfigure}{0.475\linewidth}
          \centering
          \includegraphics[width=0.925\linewidth, height=0.8\linewidth]{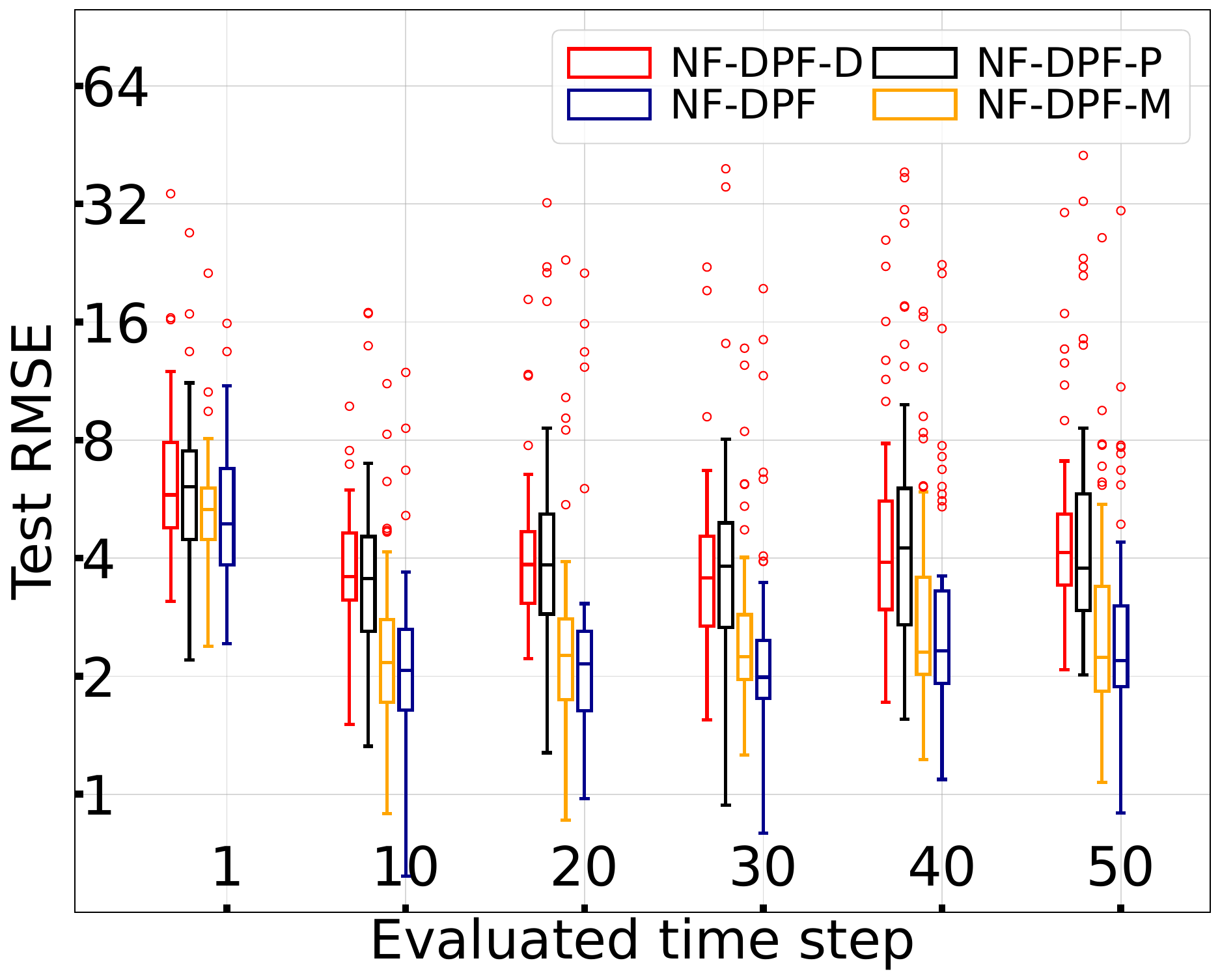}
          \caption{}
          \label{fig:ablation_test}
        \end{subfigure}%
        \begin{subfigure}{0.475\linewidth}
          \centering
          \includegraphics[width=0.925\linewidth, height=0.8\linewidth]{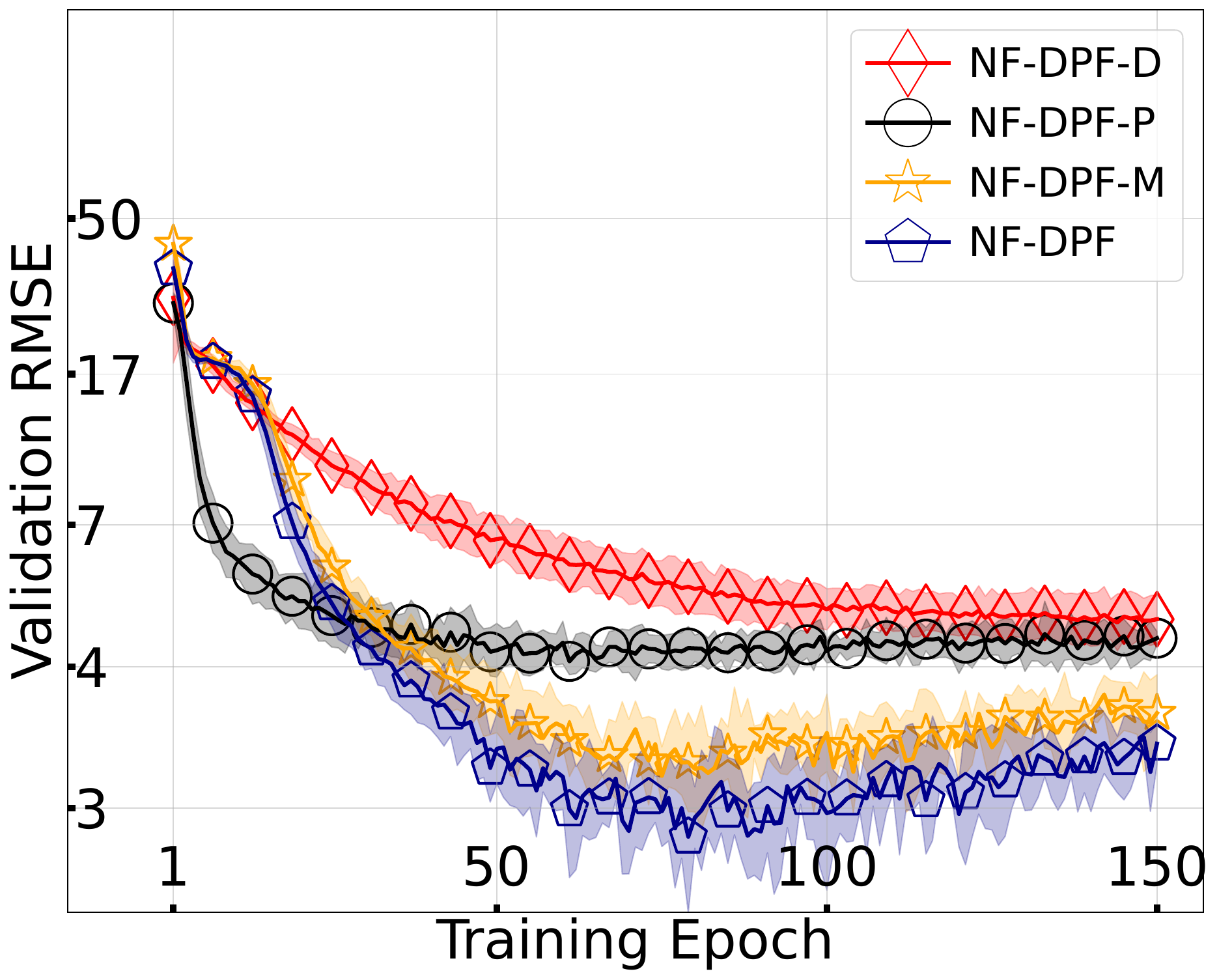}
          \caption{}
          \label{fig:ablation_validation}
        \end{subfigure}
        \caption{Ablation studies conducted in the disk localization experiment to investigate how each individual component of the NF-DPF affects its performance. (a) RMSE of variants of the NF-DPF evaluated at selected time steps on test
set. (b) RMSE of different variants of the NF-DPF on the validation set during training. Shaded areas represent the standard deviation of the presented evaluation metrics among 5 random simulations.}
        \label{fig:ablation}
    \end{figure}

	\setlength{\extrarowheight}{1pt}
	\setlength\tabcolsep{1pt}
	\begin{scriptsize}
		\begin{table}[htbp]
			\caption{Disk tracking RMSE of different differentiable particle filters. The reported RMSE is averaged over 50 time steps for 50 trajectories in the test set, and the standard deviation is computed with 5 simulation runs with different random seeds.}
			\label{tab:disk}
			\centering
			\begin{tabular}{?l?c?c?c?c?c?c?}
				\Xhline{3\arrayrulewidth}
                Method&\makecell{\\Deep\\SSM}&\makecell{\\AESMC\\Bootstrap}&\makecell{AESMC}&\makecell{PFRNN}&\makecell{PFNet}&\makecell{NF-DPF} \\\Xcline{1-7}{3\arrayrulewidth}
				RMSE&5.91$\pm$1.20&6.35$\pm$1.15&5.85$\pm$1.34 &6.12$\pm$1.23&5.34$\pm$1.27&\makecell{\textbf{3.62$\pm$0.98}}\\\Xcline{1-7}{3\arrayrulewidth}
                \makecell{Training \\time (s/it)}&0.578&0.572&0.618&0.687&0.572&1.041
                \\ \Xcline{1-7}{3\arrayrulewidth}
			\end{tabular}
		\end{table}
	\end{scriptsize}
 
    \setlength\tabcolsep{7pt}
	\begin{scriptsize}
		\begin{table}[htbp]
			\caption{Disk tracking RMSE of different variants of the NF-DPF. The reported RMSE is averaged over 50 time steps for 50 trajectories in the test set, and the standard deviation is computed with 5 simulation runs with different random seeds.}
			\label{tab:disk_ablation}
			\centering
			\begin{tabular}{?l?c?c?c?c?}
				\Xhline{3\arrayrulewidth}
                Method&\makecell{NF-DPF-D}&\makecell{NF-DPF-P}&\makecell{NF-DPF-M}&\makecell{NF-DPF}\\\Xcline{1-5}{3\arrayrulewidth}
				RMSE&5.65$\pm$0.95&5.21$\pm$1.12 &3.81$\pm$1.05&3.62$\pm$0.98\\\Xcline{1-5}{3\arrayrulewidth}
			\end{tabular}
		\end{table}
	\end{scriptsize}

        The experimental results shown in~Fig.~\ref{fig:disk_results_3} are the validation RMSEs of different methods evaluated during training. It can be observed that the NF-DPF requires fewer training epochs to converge but in the meantime achieves better tracking performance compared with the other evaluated approaches. For all methods, we saved the best models with the lowest validation error and used them to compute the tracking error on the test set.
        
        We report the test RMSEs of different differentiable particle filters in Table~\ref{tab:disk}. The experimental results in Table~\ref{tab:disk} again demonstrated the benefit of using (conditional) normalizing flows to construct differentiable particle filters. It can be observed from Table~\ref{tab:disk} that among all the tested methods, the proposed NF-DPF produces the lowest mean tracking error. Fig.~\ref{fig:disk_results_1} compares tracking RMSEs from different methods on the test set. From Fig.~\ref{fig:disk_results_1}, we found that, except for the first step $t=1$, the proposed NF-DPFs achieved the lowest tracking RMSE at all evaluated time steps compared with the other evaluated methods.

        To investigate how each component of the NF-DPF influences its performance, we conducted an ablation study in this experiment. The results of the ablation study are presented in Figure~\ref{fig:ablation} and Table~\ref{tab:disk_ablation}, where NF-DPF-D, NF-DPF-P, and NF-DPF-M respectively refer to the method that only uses normalizing flows to construct its dynamic model, proposal distribution, and measurement model. It can be observed that among the three components, the measurement model brought the most significant performance improvement to the NF-DPF. The NF-DPF-P produced lower localization error than all the other baseline methods. The NF-DPF-D outperformed 3 out of 4 baseline methods. We hypothesize that the primary factor contributing to the superior performance of NF-DPF-M compared to other variants is that, in this experiment, the location of the target disk can be directly observed from the observation image. Please note that the conclusions drawn from this ablation study may not be applicable to other environments. More extensive ablation studies are required to gain a better understanding of the role of different components in the NF-DPF.
	
	\subsection{Robot Localization in Maze Environments}
	\label{subsec:exp_robot}
	\subsubsection{Experiment setup}
	In this experiment, we evaluate the performance of NF-DPFs in three environments, namely Maze 1, Maze 2, and Maze 3, simulated in the DeepMind Lab~\cite{beattie2016deepmind} following the setup in~\cite{jonschkowski18,corenflos2021differentiable}. In each of the three maze environments, there exists a simulated robot moving through the maze, and its locations $l_t=(l^{(1)}_t, l^{(2)}_t)$, orientations $\varrho_t$, velocity $\Delta l_t=(\Delta l^{(1)}_t, \Delta l^{(2)}_t, \Delta \varrho_t)$, and camera images $y_t$ are available for model training. The collected dataset is split into training, validation, and test sets containing 900, 100, and 100 robot trajectories, each with a length of 100 time steps, respectively. We set the learning rate to be 0.001, and use the Adam optimizer to train DPFs.
	
	Based on image observations given by robot cameras, the goal in this task is to infer the location and the orientation of the robot at each time step, i.e. the latent state $x_t:=(l^{(1)}_t, l^{(2)}_t, \varrho_t)$. We give an example of observation images in Fig.~\ref{fig:maze_obs}. Particles are uniformly initialized over the maze in the first step. The dynamic model we use in DPFs is as follows:
	\begin{align}
		\label{eq:robot_dynamic}
		x_{t+1}:&=\left[\begin{array}{c}
			l_{t+1}^{(1)} \\
			l_{t+1}^{(2)} \\
			\varrho_{t+1}
		\end{array}\right]
		\nonumber\\&=\left[\begin{array}{c}
			l_t^{(1)}+\Delta l_t^{(1)} \cos \left(\varrho_t\right)+\Delta l_t^{(2)} \sin \left(\varrho_t\right) \\
			l_t^{(2)}+\Delta l_t^{(1)} \sin \left(\varrho_t\right)-\Delta l_t^{(2)} \cos \left(\varrho_t\right) \\
			\varrho_t+\Delta \varrho_t
		\end{array}\right]+\varsigma_t\,,
	\end{align}
	where $\varsigma_t\sim\mathcal{N}(\textbf{0},\Sigma^2)$ is the dynamic noise, and $\Sigma:=\text{diag}(\sigma_l, \sigma_l, \sigma_\varrho)$ with $\sigma_l=10$ and $\sigma_\varrho=0.1$. 
    
    The loss function $\mathcal{L}(\theta,\phi):=\mathcal{L}_\text{RMSE}(\theta,\phi)+\mathcal{L}_\text{AE}(\theta)$ used in this experiment consists of a root mean square error loss $\mathcal{L}_\text{RMSE}(\theta,\phi)$ and an autoencoder loss $\mathcal{L}_\text{AE}(\theta)$ as in Section~\ref{subsec:disk}. 
    The evaluation metric we use to compare the performances of different DPFs is the RMSE error between estimated robot locations and true robot locations on validation and test sets.
	
	\subsubsection{Experimental results}	
	\begin{figure}[htbp]
		\centering
		\begin{subfigure}{0.475\linewidth}
			\centering
			\includegraphics[width=0.925\linewidth, height=0.8\linewidth]{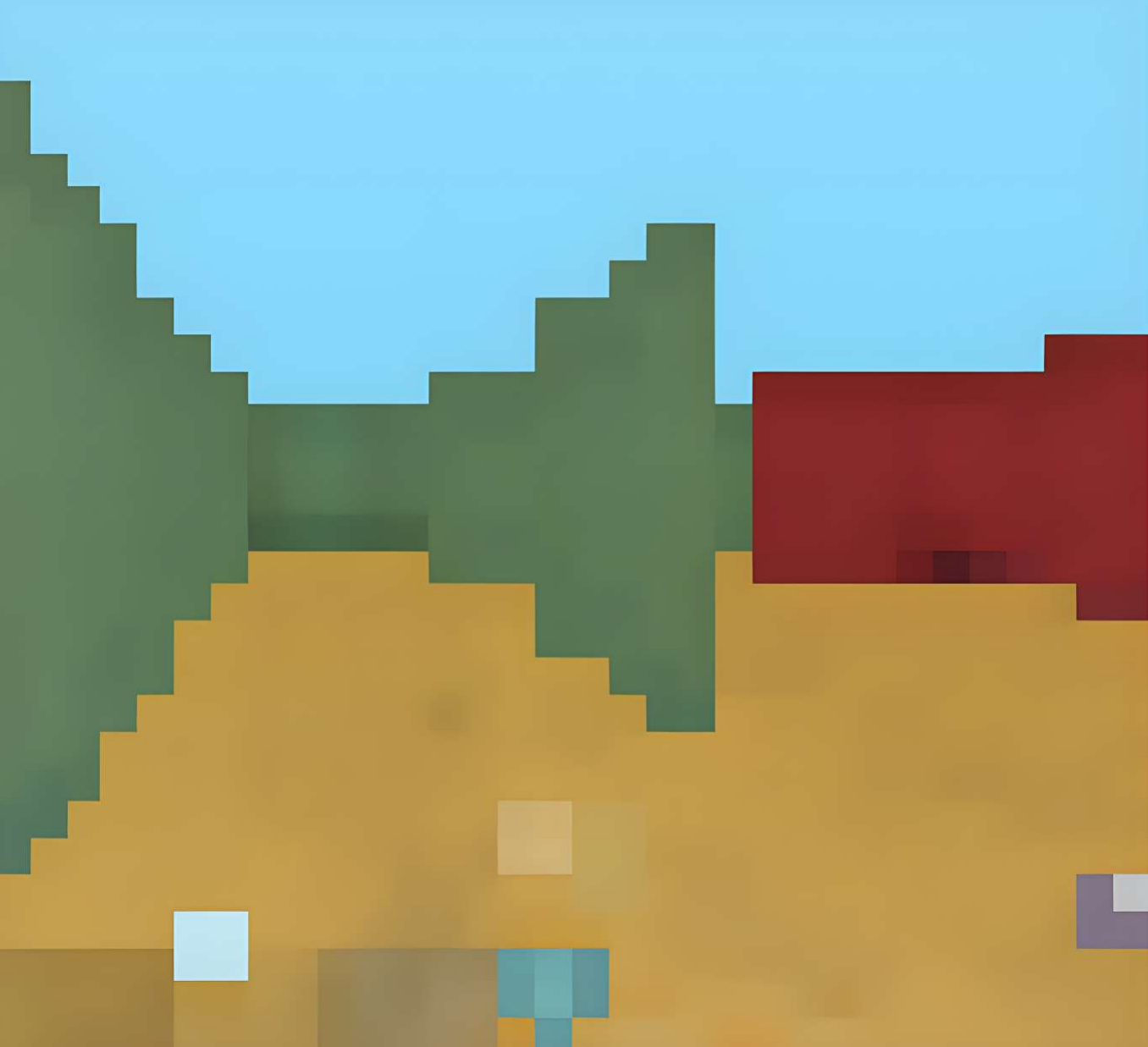}
			\caption{}
			\label{fig:maze_obs}
		\end{subfigure}
		\begin{subfigure}{0.475\linewidth}
			\centering
			\includegraphics[width=\linewidth]{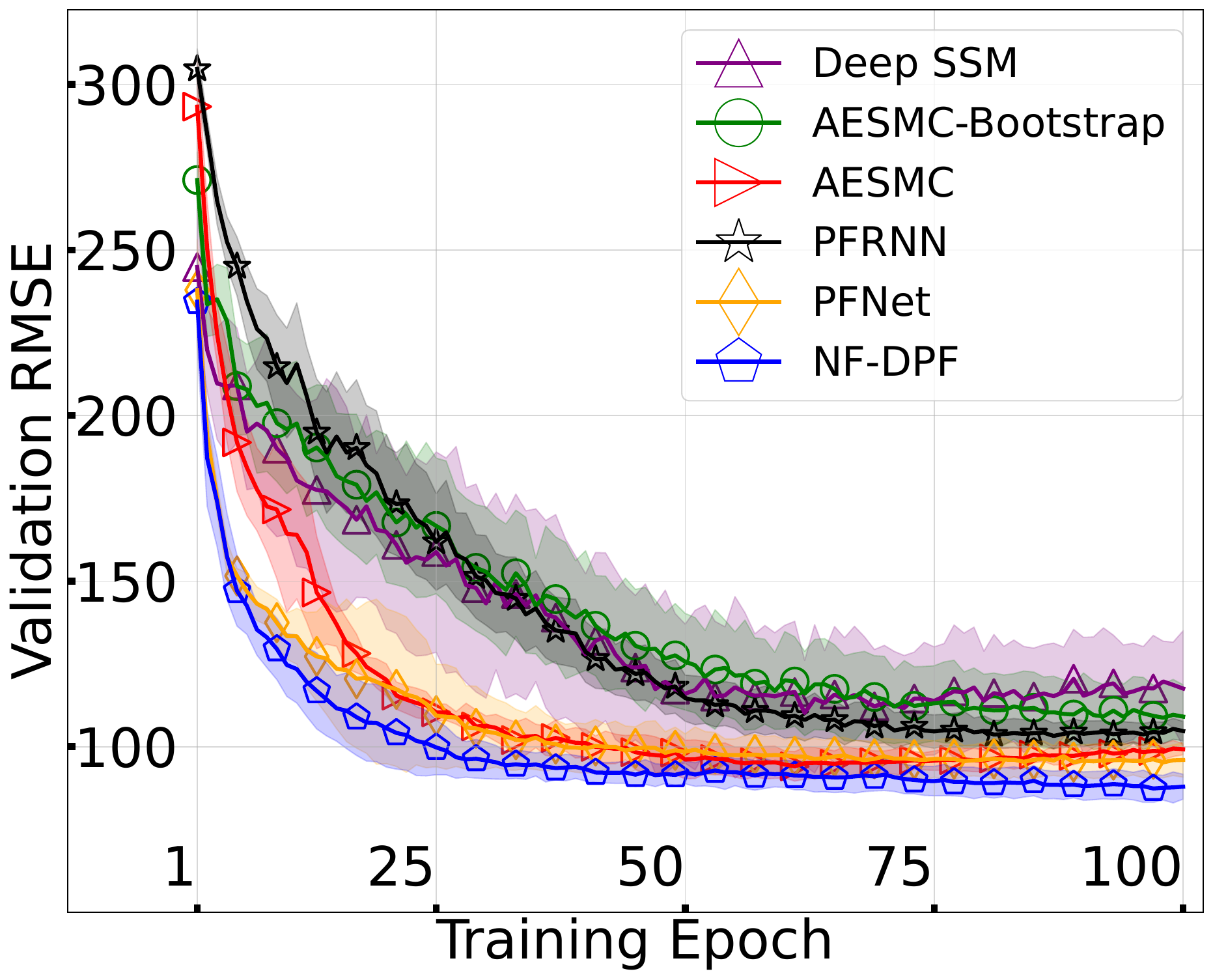}
			\caption{}
			\label{fig:maze_results_1}
		\end{subfigure}
		\begin{subfigure}{0.475\linewidth}
			\centering
			\includegraphics[width=\linewidth]{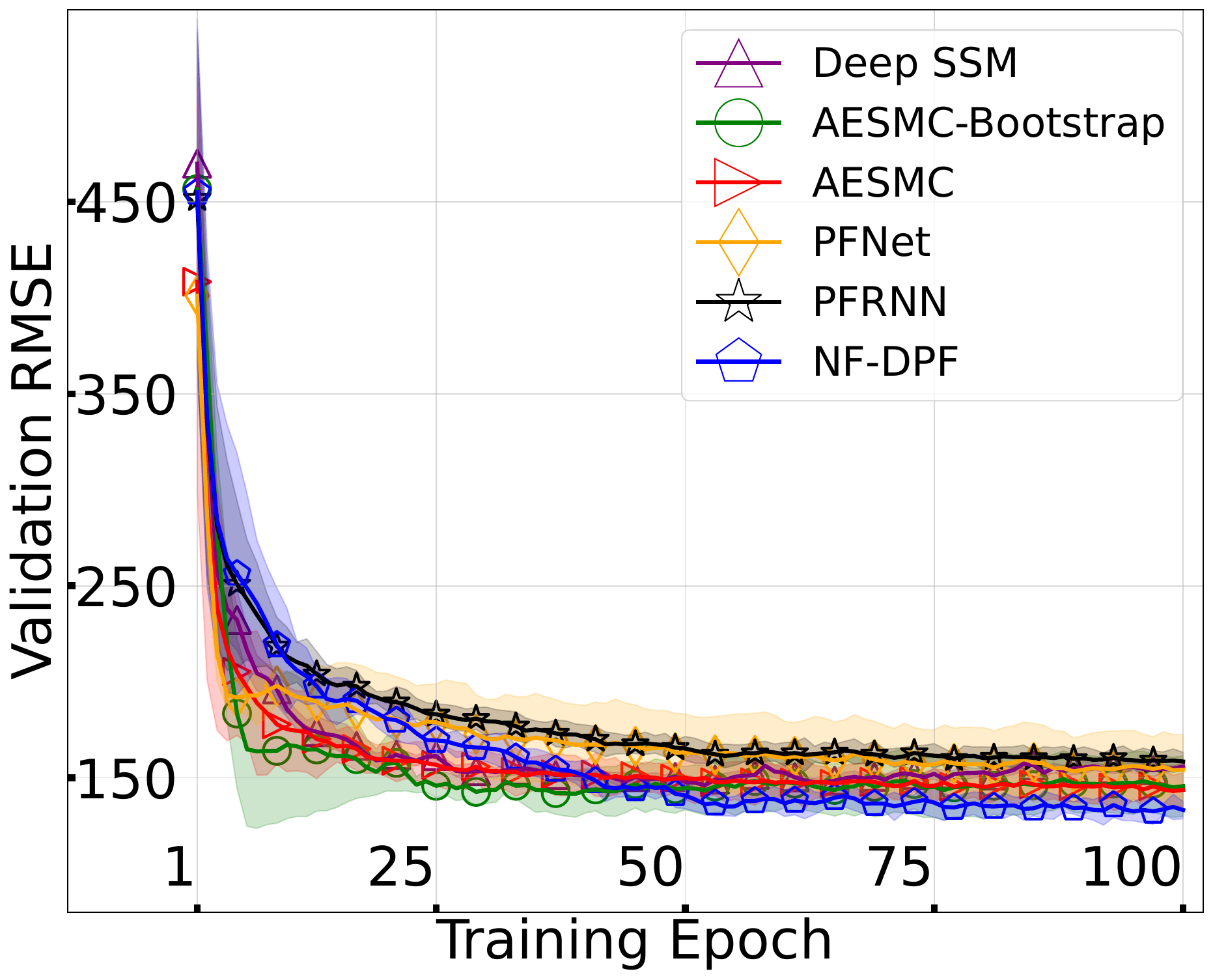}
			\caption{}
			\label{fig:maze_results_2}
		\end{subfigure}
		\begin{subfigure}{0.475\linewidth}
			\centering
			\includegraphics[width=\linewidth]{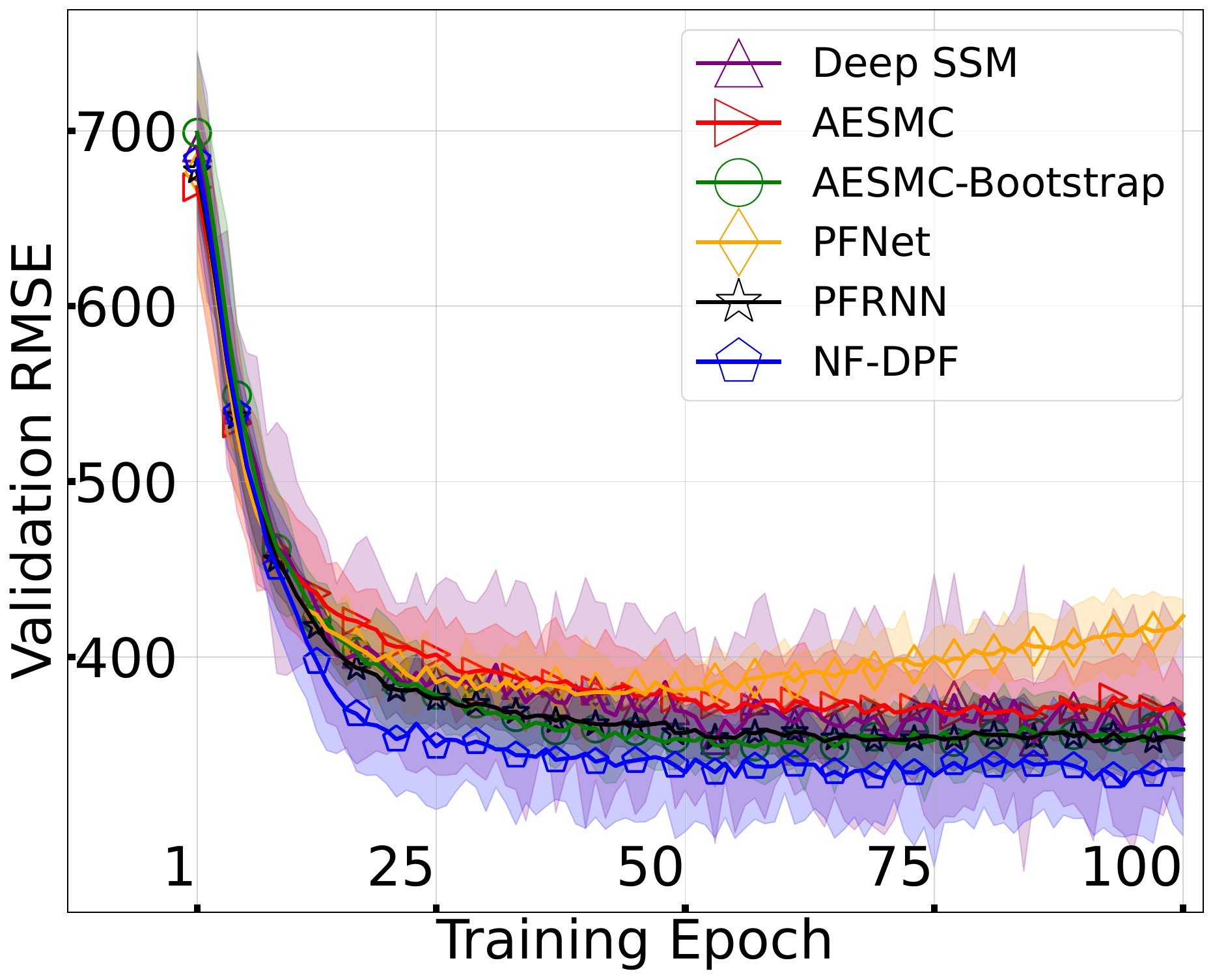}
			\caption{}
			\label{fig:maze_results_3}
		\end{subfigure}
		\caption{(A) An example of observation images in maze environments. RMSE of different DPFs on the validation set during training in (B) Maze 1. (C) Maze 2. (D) Maze 3. The shaded area represents the standard deviation of the presented evaluation metrics among 5 random simulations.}
		\label{fig:maze_results}
	\end{figure}
	
	\begin{figure}[htbp]
		\centering
		\begin{subfigure}{0.475\linewidth}
			\centering
			\includegraphics[width=\linewidth]{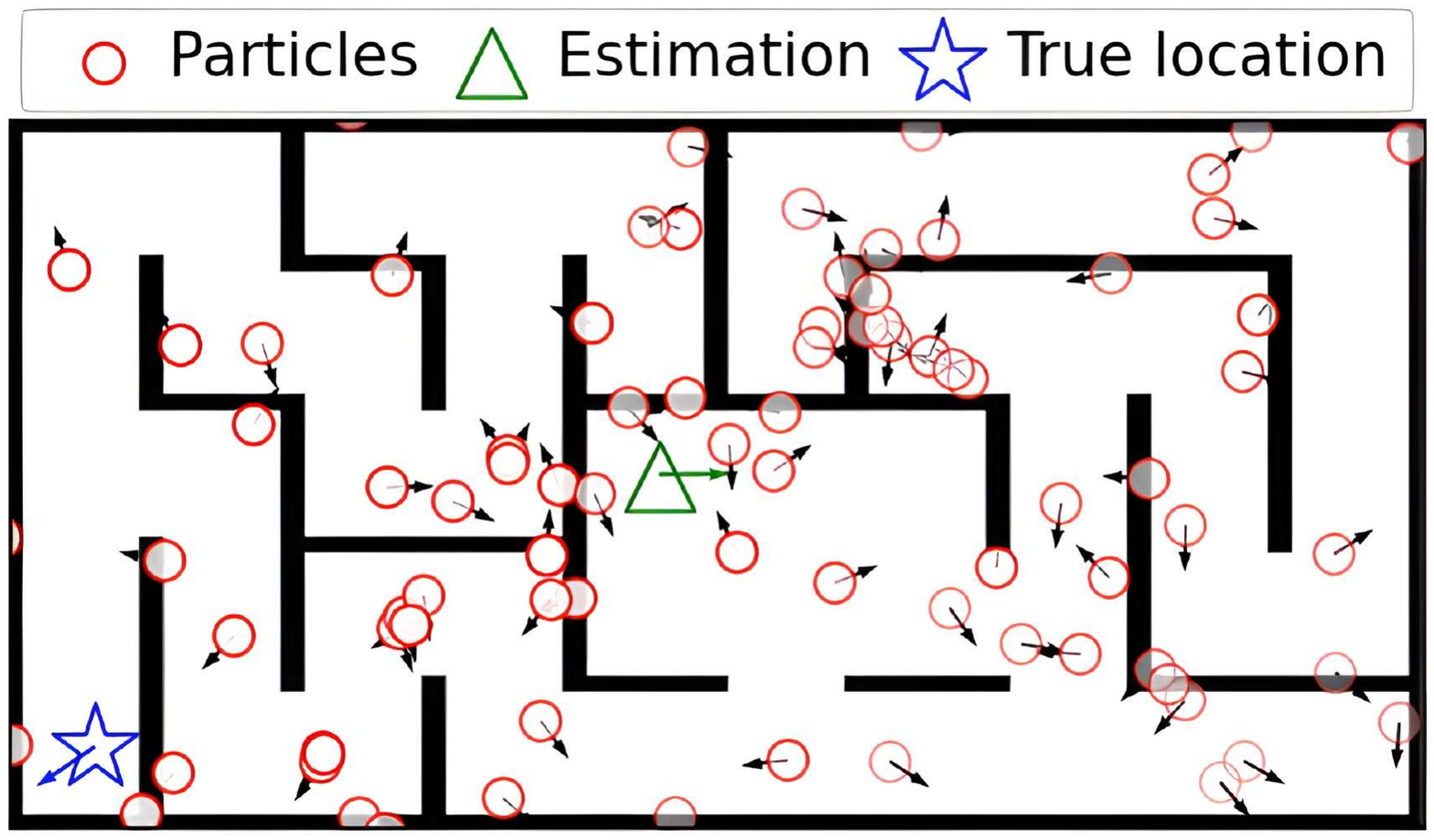}
			\caption{Time step 0}
			\label{fig:maze_vis_1}
		\end{subfigure}
		\begin{subfigure}{0.475\linewidth}
			\centering
			\includegraphics[width=\linewidth]{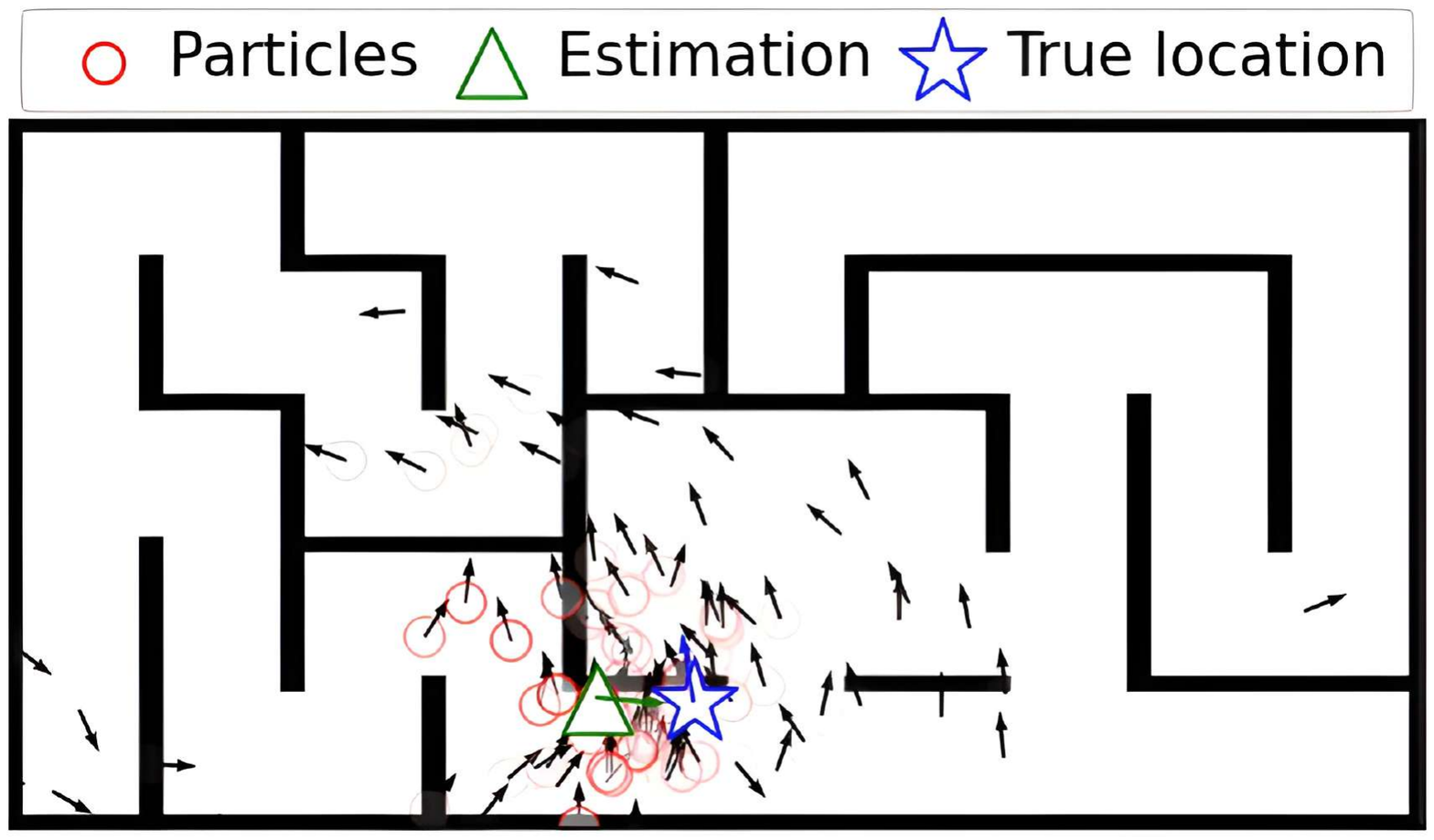}
			\caption{Time step 25}
			\label{fig:maze_vis_25}
		\end{subfigure}
		\begin{subfigure}{0.475\linewidth}
			\centering
			\includegraphics[width=\linewidth]{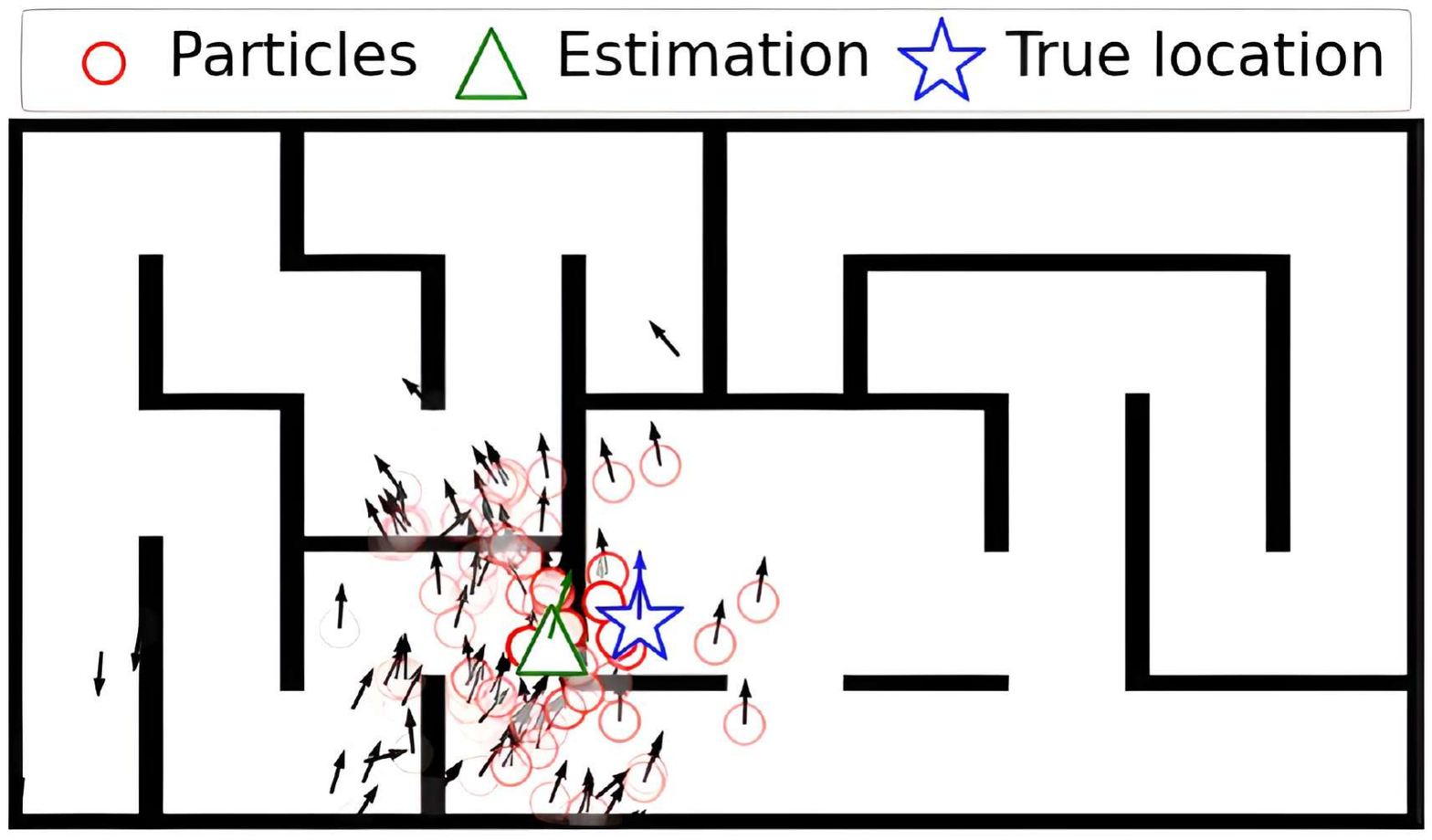}
			\caption{Time step 50}
			\label{fig:maze_vis_50}
		\end{subfigure}
		\begin{subfigure}{0.475\linewidth}
			\centering
			\includegraphics[width=\linewidth]{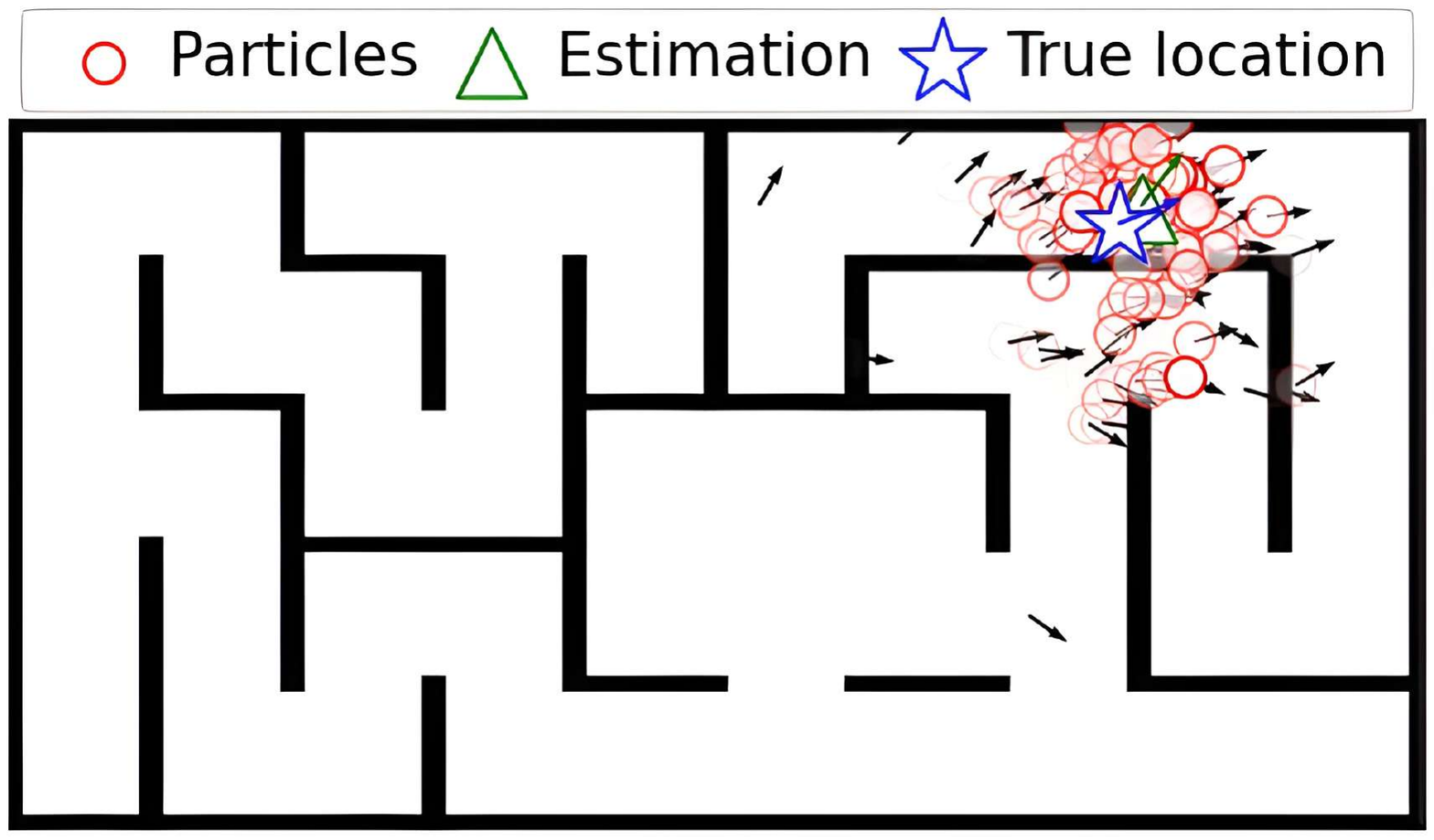}
			\caption{Time step 100}
			\label{fig:maze_vis_100}
		\end{subfigure}
		\caption{A visualization of the localization results of the NF-DPF at different time steps. Arrows represent the orientation $\varrho_t$ of particles and robots. The deeper the color of a particle, the higher its importance weight.}
		\label{fig:maze_vis}
	\end{figure}
	We first show validation RMSEs of tested methods in Figs.~\ref{fig:maze_results_1},~\ref{fig:maze_results_2}, and~\ref{fig:maze_results_3}. Since the size of environments varies from maze to maze (Maze 1: 500$\times$1000, Maze 2: 900$\times$1500, Maze 3: 1300$\times$2000), the reported RMSEs increase as the maze gets larger. As expected, the proposed NF-DPF outperformed the other differentiable particle filtering frameworks regarding validation RMSEs in all three maze environments when the training has finished. 
    In Table~\ref{tab:maze}, we report the RMSE of different methods at the last time step $t=100$ on the test set. Experimental results shown in Table~\ref{tab:maze} illustrate the superior performance of the proposed NF-DPF compared with the baseline methods. Specifically, the NF-DPF has the lowest RMSEs at the last time step in all three maze environments, implying that the NF-DPF can better localize the object for longer sequences.

	\begin{scriptsize}
 \setlength{\tabcolsep}{3pt}
		\begin{table}[htbp]
			\caption{Test RMSE of different DPFs in maze environments. The reported RMSE is computed at the last step $t=100$ for 100 trajectories in the test set. Standard deviations are computed with 5 simulation runs with different random seeds.}
			\label{tab:maze}
			\centering
			\begin{tabular}{?l?c?c?c?c?c?c?}
				\Xhline{3\arrayrulewidth}\rule{-2mm}{0mm}
				\multirow{1}{0.01\linewidth}{\rule{1.5mm}{0mm}\multirow{2}{0.01\linewidth}{Method}}&\multicolumn{2}{c?}{Maze 1}&\multicolumn{2}{c?}{Maze 2}&\multicolumn{2}{c?}{Maze 3}\\\Xcline{2-7}{3\arrayrulewidth}                    &RMSE&s/it&RMSE&s/it&RMSE&s/it\\ \Xcline{1-7}{3\arrayrulewidth}  
				\rule{-1.5mm}{0mm} Deep SSM& 61.0$\pm$10.8&0.571& 114.0$\pm$8.8&0.557& 203.0$\pm$10.1&0.580\\\Xcline{1-7}{3\arrayrulewidth}
				\rule{-1.5mm}{0mm} \makecell{AESMC\\ Bootstrap}&56.5$\pm$11.5&0.567&115.6$\pm$6.8&0.561&220.6$\pm$11.1&0.552\\\Xcline{1-7}{3\arrayrulewidth}
				\rule{-1mm}{0mm} AESMC &52.1$\pm$7.5&0.602&109.2$\pm$11.7&0.612&201.3$\pm$14.7&0.607\\\Xcline{1-7}{3\arrayrulewidth}
				\rule{-1mm}{0mm} PFNet &51.4$\pm$8.7&0.581&120.3$\pm$8.7&0.597&212.1$\pm$15.3&0.601\\\Xcline{1-7}{3\arrayrulewidth}
				\rule{-1mm}{0mm} PFRNN &54.1$\pm$8.9&0.644&125.1$\pm$8.2&0.631&210.5$\pm$10.8&0.650\\\Xcline{1-7}{3\arrayrulewidth}
				\rule{-1mm}{0mm} NF-DPF &\textbf{46.1$\pm$6.9}&0.667&\textbf{103.2$\pm$10.8}&0.661&\textbf{182.2$\pm$19.9}&0.672\\\Xcline{1-7}{3\arrayrulewidth}
			\end{tabular}
		\end{table}
	\end{scriptsize}
 
    We provide a visualization of the localization results in Fig.~\ref{fig:maze_vis}, where the particles, true robot locations, and estimated robot locations at different time steps are visualized. Fig.~\ref{fig:maze_vis_1} shows the localization result at the initialization step $t=0$. We can see that in Fig.~\ref{fig:maze_vis_1}, the estimation is located around the center of the maze and is far from the true location, which is expected because the particles are uniformly initialized at time step $t=0$. In Fig.~\ref{fig:maze_vis_25}, Fig.~\ref{fig:maze_vis_50}, and Fig.~\ref{fig:maze_vis_100}, it can be found that in later time steps, the NF-DPF can produce estimated locations that are close to the ground-truth locations, with particles centered at the ground-truth locations. In addition, we also observe that the learned measurement model can capture the relationship between observation images $y_t$ and robot locations $x_t$. In Fig.~\ref{fig:maze_vis}, especially Fig.~\ref{fig:maze_vis_100}, it is obvious that particles that are close to the true robot location are assigned greater importance weights, and vice versa.

\subsection{Robot tracking in Michigan NCLT dataset}
\label{subsec:exp_nclt}
\subsubsection{Experiment setup}
In this experiment, we evaluate the performance of baseline methods and the proposed NF-DPF in a real-world dataset, the Michigan North Campus Long-Term (NCLT) dataset~\cite{carlevaris2016university}. The NCLT dataset consists of robot trajectories collected at 26 sessions. The trajectories are labeled, meaning that the ground-truth location of the robot is accessible. The latent state is the robot location and its velocity  $x_{t}=[l_{t}^{(1)}, l_{t}^{(2)}, \Delta l_{t}^{(1)}, \Delta l_t^{(2)}]^\intercal$. Observations in this dataset are noisy sensor readings such as GPS and odometer. In this experiment, we only use the odometer reading as the observation, following the settings in~\cite{revach2022kalmannet}. The goal in this experiment is to track the location of the moving robot, and we initialize the latent state $x_{t}$ with the ground-truth latent state.
Following the setup in~\cite{revach2022kalmannet}, we use the trajectory collected on 2012-01-22 and sample the trajectory with a frequency of 1 Hz, resulting in 5,850 time steps. We split the 5,850 time steps into training, validation, and test datasets. The training, validation, and test datasets respectively comprise 20 trajectories of length 234, 3 trajectories of length 195, and 3 trajectories of length 195. We use the RMSE between the ground-truth location of the robot and the estimated location of the robot as both the training objective and the evaluation metric. Since observations in this experiment are 2-dimensional vectors and there is no need to encode the observations into feature vectors, the autoencoder reconstruction loss is not included in the loss function for this experiment.

From the experimental results shown in Table~\ref{tab:nclt}, we observed that the PFNet achieved the lowest RMSE, and our method produced slightly higher RMSE and lower standard deviation than the PFNet in this experiment. We speculate that the reason why proposed NF-DPF did not produce the best tracking performance is that the latent dynamics of the robot are relatively simple and the observation in this experiment is uninformative, so that using a very expressive model like NF-DPF only brings very marginal benefit and may result in overfitting.

\setlength{\extrarowheight}{1pt}
\setlength\tabcolsep{2pt}
\begin{scriptsize}
    \begin{table}[htbp]
        \caption{Robot tracking RMSE of different differentiable particle filters in the NCLT dataset. The reported RMSE is averaged over 195 time steps for 3 trajectories in the test set, and the standard deviation is computed with 5 simulation runs with different random seeds.}
        \label{tab:nclt}
        \centering
        \begin{tabular}{?l?c?c?c?c?c?c?}
            \Xhline{3\arrayrulewidth}
            Method&Deep SSM&\makecell{AESMC\\Bootstrap}&\makecell{AESMC}&\makecell{PFRNN}&\makecell{PFNet}&\makecell{NF-DPF} \\\Xcline{1-7}{3\arrayrulewidth}
            RMSE&56.3$\pm$4.5&60.3$\pm$3.8&57.5$\pm$5.2 &59.2$\pm$5.1&\textbf{53.5$\pm$5.7}&54.8$\pm$4.8\\\Xcline{1-7}{3\arrayrulewidth}
        \end{tabular}
    \end{table}
\end{scriptsize}
	
	\section{Conclusion}
	\label{sec:conclusion}
	This paper introduced a novel variant of differentiable particle filters (DPFs), the normalizing flow-based differentiable particle filter (NF-DPF), which is built based on normalizing flows and conditional normalizing flows. The proposed NF-DPF first provides a general mechanism to construct data-adaptive dynamic models, proposal distributions, and measurement models, three of the core components of particle filters. The theoretical analysis shows the consistency of the proposed NF-DPF and derives an upper bound for its estimation error. We empirically showed the superior performance of the NF-DPF over the other DPF methods on a wide range of simulated tasks, including parameter learning and posterior approximation in linear Gaussian state-space models, image-based disk localization, and robot localization in maze environments. Experimental results show that the NF-DPF can achieve the lowest tracking and localization errors in all considered experiments, indicating that (conditional) normalizing flows can indeed improve the performance of DPFs in various settings.
 Interesting future research directions include the development of differentiable resampling techniques with better statistical properties. 
	
	\bibliographystyle{IEEEtran}
	\bibliography{ref.bib}

	\newpage
	
	
    \begin{appendices}

    \section{Optimal Transport Background}
	\label{appendix:ot}
	In this section, we provide a brief review of concepts related to the proposed work, including the definition of Wasserstein distances, the optimal transport plan (coupling), and the optimal transport map. 
	Note that notations introduced in Section~\ref{sec:theoretical_results} also apply in this section.
	
	\noindent\textbf{Wasserstein Distance:} 
	
	Let $P_p(\mathcal{X})$ be a set of Borel probability measures with a finite $p$-th moment on a Polish metric space $(\mathcal{X}, d)$~\cite{villani2008optimal}. Given two probability measures $\alpha \in P_p(\mathcal{X})$, $\beta \in P_p(\mathcal{X}')$, the Wasserstein distance of order $p\in[1, +\infty)$ between $\alpha$ and $\beta$ is defined as: 
	\begin{gather}
		\label{equation_WD}
		\wass_p(\alpha, \beta)=\bigg(\underset{\mathcal{P}\in\mathcal{U}(\alpha,\beta)}{\inf} \int_{\mathcal{X}\times\mathcal{X}'} d(x,x')^p \mathcal{P}(\rd x,\rd x')\bigg)^{\frac{1}{p}}\,,
	\end{gather}
	where $d(\cdot,\cdot)^p$ is the cost function, $\mathcal{U}(\alpha,\beta)$ represents the set of all transportation plans $\mathcal{P}(\rd x,\rd x')$, i.e. joint distributions whose marginals are $\alpha$ and $\beta$, respectively. Formally, the set of all transportation plans between $\alpha$ and $\beta$ is defined as:
    \begin{gather}
        \mathcal{U}(\alpha,\beta):=\Big\{\mathcal{P}\in P(\mathcal{X}\times\mathcal{X}'):\mathcal{X}_{\#}\mathcal{P}=\alpha,\, \mathcal{X}'_{\#}\mathcal{P}=\beta\Big\}\,,
    \end{gather}
    where $\mathcal{X}_{\#}\mathcal{P}$ and $\mathcal{X}'_{\#}\mathcal{P}$ respectively represent the projection of a joint measure $\mathcal{P}\in P(\mathcal{X}\times\mathcal{X}')$ on $\mathcal{X}$ and $\mathcal{X}'$. Every transport plan $\mathcal{P}(\rd x,\rd x')$ corresponds to a transport map, also called the barycentric projection map, which is defined as:
    \begin{gather}
        \mathbf{T}(x)=\int x' \mathcal{P}(\rd x'|x)\,,\\
        \mathcal{P}(\rd x'|x)=\frac{\mathcal{P}(\mathrm{d}x, \mathrm{d}x')}{\int_{\mathcal{X}'}\mathcal{P}(\mathrm{d}x, \mathrm{d}x')}\,.
    \end{gather}
	
	\noindent\textbf{Optimal Transport Notations:}
	
	Solving the original optimal transport problem is computationally expensive and non-differentiable, an alternative is to rely on entropy-regularized optimal transport~\cite{peyre2019computational,cuturi2013sinkhorn}. In the DPF setting, $\epsilon$ denotes the regularization coefficient in the entropy-regularized optimal transport problem, $\mathcal{P}_{N,\epsilon}^{\mathrm{OT}}$ denotes the regularized transport plan between $\alpha^{(t)}_{N}$ and $\beta^{(t)}_{N}$, $\mathbf{T}(\cdot):\mathcal{X}\rightarrow\mathcal{X}$ denotes the optimal transport map between $\alpha^{(t)}$ and $\beta^{(t)}$, and $\mathbf{T}_{N,\epsilon}(\cdot):\mathcal{X}\rightarrow\mathcal{X}$, $\mathbf{T}_{N,\epsilon}(x)\coloneqq\int {x'} \mathcal{P}^{{\mathrm{OT}}}_{N,\epsilon}(\mathrm{d}{x}'|x)$ the transport map induced by the transport plan $\mathcal{P}_{N,\epsilon}^{\mathrm{OT}}$.

    \section{Proof of Proposition~\ref{prop:consistency}}
    \label{appendix:proof}
    \phantom{123}
    \newline
    We use notations below for the following proofs:
	\begin{gather}
		\label{eq:notations}
		\alpha^{(t)}:=p(x_t|y_{0:t-1};\theta)\,,\; \beta^{(t)}:=p(x_t|y_{0:t};\theta)\,,\\
		\label{eq:notations_empirical}
		\alpha^{(t)}_{N}(\psi):=\frac{1}{N}\sum_{i=1}^{N}\psi(x_t^i)\,,\; \beta^{(t)}_{N}(\psi):=\sum_{i=1}^{N}\tilde{w}_t^i \psi(x_t^i)\,,\;\\
		\tilde{\beta}^{(t)}_{N}(\psi):=\frac{1}{N}\sum_{i=1}^{N} \psi(\tilde{x}_t^i)\,,\;
		\omega^{(t)}(x_t)=p(y_t|x_t;\theta)\,,
		\label{eq:proof_weight_function}
	\end{gather}
	with $\alpha^{(0)}:=\pi(x_0;\theta)$, $\psi(\cdot): \mathcal{X}\rightarrow \mathbb{R}$ is a function defined on $\mathcal{X}$, $\alpha^{(t)}_{N}$ is an approximation of the predictive distribution $\alpha^{t}$ with $N$ uniformly weighted particles, and $\beta^{(t)}_{N}$ is an approximation of the posterior distribution $\beta^{(t)}$ with $N$ particles weighted by $\tilde{w}_t^i$. $\tilde{\beta}^{(t)}_{N}$ is an approximation of the posterior distribution $\beta^{(t)}$ with uniformly weighted particles $\tilde{x}_t^i$ obtained by applying the entropy-regularized optimal transport resampler in~\cite{corenflos2021differentiable} as shown in the line 12 of Algorithm~\ref{alg:nf_dpfs}. For a measure $\alpha$ defined on $\mathcal{X}$ we use $\alpha(\psi)=\int_{\mathcal{X}}\psi(x)\alpha(\mathrm{d}x)$ to denote the expectation of $\psi(\cdot)$ w.r.t. $\alpha$.

	To prove Proposition~\ref{prop:consistency}, we introduce the following Assumptions~\ref{ass:compact}-\ref{ass:weight_lipcontraction}:
	\begin{manualtheorem}{V.1}\label{ass:compact_appendices}
		$\mathcal{X}$ is a compact subset of $\mathbb{R}^{d_\mathcal X}$ with diameter $\mathfrak{d}:=\underset{x,x'\in\mathcal{X}}{\sup}{||x-x'||_2}$, where $||\cdot||_2$ denotes the Euclidean distance.
	\end{manualtheorem}
        Assumption V.1 requires that the domain where the latent state $x_t$ is defined on is a compact subset of the $d_\mathcal X$-dimensional Euclidean space $\mathbb{R}^{d_\mathcal X}$. In particular, a subset of Euclidean space is called compact if it is closed and bounded. For a compact subset of $\mathbb{R}^{d_\mathcal X}$, we can always find a diameter $\mathfrak{d}$ such that for all $x$ and $x'$ in $\mathcal{X}$, we have $||x, x'||_2 \leq \mathfrak{d}$. A simple example of $\mathcal{X}$ that satisfies this assumption is a $\mathbb{R}^{d_\mathcal X}$-dimensional hypersphere: $\mathcal{X}:=\{x\in\mathbb{R}^{d_\mathcal X}: ||x||_2\leq\mathfrak{d}/2\}$ 
        
	\begin{manualtheorem}{V.2}\label{ass:transport_plan_appendices}
		For $\forall t \geq 0$, there exists a unique optimal transport plan between $\alpha^{(t)}$ and $\beta^{(t)}$ featured by a deterministic transport map $\textbf{T}_t(\cdot):\mathcal{X}\rightarrow\mathcal{X}$, and the transport map $\textbf{T}_t(\cdot)$ is $\lambda$-Lipschitz for $\forall t \geq 0$ with $\lambda>0$. 
	\end{manualtheorem}
It was proved in~\cite{brenier1991polar} that on Euclidean space with Euclidean distance as the transportation cost function, there
is always a unique optimal transportation map if $\mu$ and $\rho$ are absolutely continuous with respect to the Lebesgue measure. Therefore, the first part of Assumption V.2 is satisfied as long as the predictive distribution $\alpha^{(t)}$ and the posterior distribution $\beta^{(t)}$ are continuous. A transport map $\textbf{T}_t(\cdot)$ is said to be
$\lambda$-Lipschitz if $||\textbf{T}_t(\cdot)(x), \textbf{T}_t(\cdot)(x')||_2 \leq \lambda||x, x'||_2$ for all $x$, $x'$ in $\mathcal{X}$, so the second part of Assumption V.2 requires that the transport map $\textbf{T}_t(\cdot)$ does not change discretely and the gradient of $\textbf{T}_t(\cdot)$ is bounded. A simple example that satisfies Assumption V.2 is the identity transport map $\textbf{T}_t(x)=x$.

	\begin{manualtheorem}{V.3}\label{ass:transition_lipcontraction_appendices}
		Denote by $f(\cdot)$ the transition kernel $p( {x}_t| {x}_{t-1};\theta)$ of NF-DPFs defined in Eq.~\eqref{eq:dynamic_density} and $\psi(\cdot):\mathcal{X}\rightarrow\mathbb{R}$ the considered bounded $k$-Lipschitz function, there exists an $\eta\in \mathbb{R}$ such that for any two probability measures $\mu, \rho$ on $\mathcal{X}$ 
		$$|\mu f(\psi) - \rho f(\psi)| \leq \eta|\mu(\psi) - \rho(\psi)|\,, s.t.\; \mu(\psi) \neq \rho(\psi)\,.$$
	\end{manualtheorem}
	Intuitively, this assumption requires the transition kernel $f(\cdot)$ not to change the measures $\mu$ and $\rho$ too much so that the absolute difference between the expectation of $\psi(\cdot)$ w.r.t. $\mu f$ and $\rho f$ is smaller than $\eta$ times the absolute difference between the expectation of $\psi(\cdot)$ w.r.t. $\mu$ and $\rho$. 

	\begin{manualtheorem}{V.4}\label{ass:weight_lipcontraction_appendices}
		There exists a constant $\zeta\in \mathbb{R}$ such that for any continuous probability measure $\mu$ on $\mathcal{X}$ and its empirical approximation $\mu_N$, for weighted probability measures $\mu_{\omega_t}={\omega_t\mu}/{\mu(\omega_t)}$ and $\mu_{N,\omega_t}={\omega\mu_N}/{\mu_N(\omega_t)}$, we have
		$$\wass_2(\mu_{N,\omega_t}, \mu_{\omega_t}) \leq \zeta\wass_2(\mu_N, \mu)\,,$$
		where $\omega_t(\cdot):\mathcal{X}\rightarrow\mathbb{R}$ is defined in Eq.~\eqref{eq:weight_function}, and $\wass_2(\cdot, \cdot)$ refers to the $2$-Wasserstein distance~\cite{villani2008optimal,peyre2019computational}.
	\end{manualtheorem}
    A sufficient condition for Assumption V.4 to hold is that for a finite $N$, there is a lower bound $\mathcal{B}(N)>0$ on $\wass_2(\mu_N, \mu)$: $\wass_2(\mu_N, \mu)\geq\mathcal{B}(N)$. Specifically, with the lower bound $\mathcal{B}(N)$, we can choose a constant $\zeta\geq\frac{N\mathfrak{d}}{\mathcal{B}(N)}$ such that $\wass_2(\mu_{N,\omega_t}, \mu_{\omega_t}) \leq \zeta\wass_2(\mu_N, \mu)$ holds for any function $\omega_t(\cdot)$:
    \begin{align}
        \wass_2(\mu_{N,\omega_t}, \mu_{\omega_t})&\leq N\mathfrak{d}
        \leq \frac{N\mathfrak{d}}{\mathcal{B}(N)}\mathcal{B}(N)
        \leq \zeta \wass_2(\mu_N, \mu)\,.\nonumber
    \end{align}

    We first present five Lemmas~\ref{lemma:wass_exp},~\ref{lemma:transport_map_distance},~\ref{lem:linochetto},~\ref{lemma:predictive_bound},~\ref{lemma:posterior_weight_approximation} and Proposition~\ref{prop:my_CVDETDetailed}. Lemma~\ref{lem:linochetto} is borrowed from ~\cite{corenflos2021differentiable} (Lemma C.2). The proof of Proposition~\ref{prop:consistency} is based on proof by induction, which is inspired by the proof of Proposition 11.3 of~\cite{chopin2020introduction}.
	
	\begin{lemma}
		\label{lemma:wass_exp}
		For all bounded $k$-Lipschitz function $\psi(\cdot):\mathcal{X}\rightarrow\mathbb{R}$ and any two probability measures $\mu, \rho$ on $\mathcal{X}$, we have:
		\begin{align}
			|\mu(\psi)-\rho(\psi)|\leq k \wass_1(\mu, \rho)\,.
		\end{align}
	\end{lemma}
	\begin{proof}
		Denote by $$\mathcal{P}^*(\rd x, \rd x'):=\underset{\mathcal{P}\in\mathcal{U}(\mu, \rho)}{\operatorname{argmin}}{\int_{\mathcal{X}^2}||x-x'||_2\; \mathcal{P}(\rd x, \rd x')}$$ the optimal transport plan between $\mu$ and $\rho$ w.r.t. the Euclidean distance, we have:
		\begin{align}
			&|\mu(\psi)-\rho(\psi)|\\
                =&\left|\int_{\mathcal{X}} \psi(x)\mu(\rd x) - \int_{\mathcal{X}} \psi(x')\rho(\rd x')\right|\nonumber\\
			=&\left|\int_{\mathcal{X}^2} \psi(x)\mathcal{P}^*(\rd x, \rd x') - \int_{\mathcal{X}^2} \psi(x')\mathcal{P}^*(\rd x, \rd x')\right|\nonumber\\
			\leq&\int_{\mathcal{X}^2} \left|\psi(x)-\psi(x')\right|\mathcal{P}^*(\rd x, \rd x')\nonumber\\
			\leq&\int_{\mathcal{X}^2} k||x-x'||_2\mathcal{P}^*(\rd x, \rd x')\nonumber\\
			=& k\wass_1(\mu, \rho)\,.
		\end{align}
	\end{proof}
	
	\begin{lemma}
		\label{lemma:transport_map_distance}
		For probability measures $\mu$ and $\rho$ defined on $\mathcal{X}$, denote by $\mathbf{T}(\cdot):\mathcal{X}\rightarrow\mathcal{X}$ the optimal transport map between them. Let $\mu_N=\sum_{i=1}^{N}a_i\delta_{x_i'}$ and $\rho_N=\sum_{j=1}^{M}b_j\delta_{{x_j}}$ be approximations of $\mu$ and $\rho$, where $x_i'\in\mathcal{X}$ and ${x_j'} \in\mathcal{X}$ for $\forall i,j\in\{1, \cdots, N\}$. Denote by $\mathcal{P}_N(\rd x', \rd {x})\in\mathcal{U}(\mu_N,\rho_N)$ a transport plan between $\mu_N$ and $\rho_N$, and $\mathbf{T}_N(\cdot):\mathcal{X}\rightarrow\mathcal{X}$ the transport map induced by $\mathcal{P}_N(\rd x', \rd {x})$, namely $\mathbf{T}_N(x_i')=\frac{1}{a_i}\sum_{j=1}^{M}p_{i,j}{x_j}$ with $p_{i,j}=\mathcal{P}_{N,i,j}$ the element at the intersection of $\mathcal{P}_{N}$'s $i$-th row and $j$-th column. The following inequality holds: 
		\begin{align}
			&\int_{\mathcal{X}^2}||\mathbf{T}(x)-\mathbf{T}_N(x)||^2\mathcal{P}_{N}(\rd x', \rd {x})\nonumber\\
                \leq&\int_{\mathcal{X}^2}||\mathbf{T}(x')-x||^2 \mathcal{P}_{N}(\rd x', \rd {x})
		\end{align}
	\end{lemma}
	\begin{proof}
		Firstly, denote by $\langle\cdot, \cdot\rangle$ the inner product operation, we have that:
		\begin{align}
			\sum_{i=1}^{N}\sum_{j=1}^{M}p_{i,j}\Big\langle\mathbf{T}(x_i'), {x_j}\Big\rangle&=\sum_{i=1}^{N}a_i\Big\langle\mathbf{T}(x_i'), \mathbf{T}_N(x_i')\Big\rangle\nonumber\\
			&=\sum_{i=1}^{N}\sum_{j=1}^{M}p_{i,j}\Big\langle\mathbf{T}(x_i'), \mathbf{T}_N(x_i')\Big\rangle\,,\\
			\sum_{i=1}^{N}\sum_{j=1}^{M}p_{i,j}\Big\langle\mathbf{T}_N(x_i'), {x_j}\Big\rangle&=\sum_{i=1}^{N}a_i\Big\langle\mathbf{T}_N(x_i'), \mathbf{T}_N(x_i')\Big\rangle\nonumber\\
			&=\sum_{i=1}^{N}\sum_{j=1}^{M}p_{i,j}\Big\langle\mathbf{T}_N(x_i'), \mathbf{T}_N(x_i')\Big\rangle\,.
		\end{align}
		The above equation leads to:
		\begin{align}
			&\sum_{i=1}^{N}\sum_{j=1}^{M}p_{i,j}||\mathbf{T}(x_i')-{x_j}||^2\\=&\sum_{i=1}^{N}\sum_{j=1}^{M}p_{i,j}\Big\langle\big(\mathbf{T}(x_i')-{x_j}\big), \big(\mathbf{T}(x_i')-{x_j}\big)\Big\rangle\\
			=&\sum_{i=1}^{N}\sum_{j=1}^{M}p_{i,j}\bigg(\Big\langle\mathbf{T}(x_i'), \mathbf{T}(x_i')\Big\rangle+\Big\langle\mathbf{T}_N(x_i'), \mathbf{T}_N(x_i')\Big\rangle\\
            &-2\Big\langle\mathbf{T}(x_i'), \mathbf{T}_N(x_i')\Big\rangle\nonumber+\Big\langle\mathbf{T}_N(x_i'), \mathbf{T}_N(x_i')\Big\rangle\\
            &+\Big\langle {x_j}, {x_j}\Big\rangle-2\Big\langle\mathbf{T}_N(x_i'), {x_j}\Big\rangle\bigg)\\
		=&\sum_{i=1}^{N}\sum_{j=1}^{M}p_{i,j}\Big(||\mathbf{T}(x_i')-\mathbf{T}_N(x_i')||^2+||\mathbf{T}_N(x_i')-{x_j}||^2\Big)\\
			\geq&\sum_{i=1}^{N}\sum_{j=1}^{M}p_{i,j}\Big(||\mathbf{T}(x_i')-\mathbf{T}_N(x_i')||^2\Big)\,.
		\end{align}
		Therefore the stated result is obtained:
		\begin{align}
			&\int_{\mathcal{X}^2}||\mathbf{T}(x')-\mathbf{T}_N(x')||^2\mathcal{P}_{N}(\rd x', \rd {x})\\	=&\sum_{i=1}^{N}\sum_{j=1}^{M}p_{i,j}||\mathbf{T}(x_i')-\mathbf{T}_N(x_i')||^2\\
			\leq&\sum_{i=1}^{N}\sum_{j=1}^{M}p_{i,j}||\mathbf{T}(x_i')-{x_j}||^2\\
			=&\int_{\mathcal{X}^2}||\mathbf{T}(x')-x||^2 \mathcal{P}_{N}(\rd x', \rd {x})\,.
		\end{align}
	\end{proof}

        \begin{lemma}{\textup{(Lemma C.2 in~\cite{corenflos2021differentiable})}}
        \label{lem:linochetto}
            Let $\mathcal{X}\subset \mathbb{R}^d$ be compact with diameter $\mathfrak{d}>0$. 
            Suppose we are given two probability measures $\alpha, \beta$ on $\mathcal{X}$ with a unique deterministic, $\lambda$-Lipschitz optimal transport map $\mathbf{T}$ while $\alpha_N=\sum_{i=1}^N a_i \delta_{x'^i}$ with $a_i>0$ and $\beta_N=\sum_{i=1}^N b_i \delta_{x^i}$. We write $\mathcal{P}^{\mathrm{OT},N}$, resp. $\mathcal{P}^{\mathrm{OT},N}_\epsilon$, for an optimal coupling between $\alpha_N$ and $\beta_N$, resp. the $\epsilon$-regularized optimal transport plan, between $\alpha_N$ and $\beta_N$.
            Then 
            \begin{align*}
                &\left[\int ||x - \mathbf{T}(x')||^2 \mathcal{P}^{\mathrm{OT}}_{N,\epsilon}(\rd x', \rd x) \right]^{\frac{1}{2}} 
                \leq 2\lambda^{1/2} 
                \mathcal{E}^{1/2}\left[\mathfrak{d}+ \mathcal{E}\right]^{1/2}&
                \\&+ \max\{\lambda,  1\}\left[ \wass_2(\alpha_N, \alpha) + \wass_2(\beta_N, \beta)\right],
            \end{align*}
            where
            $$\mathcal{E}\coloneqq \mathcal{E}(N, \epsilon, \alpha, \beta)\coloneqq \wass_2(\alpha_N, \alpha)+\wass_2(\beta_N, \beta) + \sqrt{2\epsilon \log (N)}.$$
        \end{lemma}

	\begin{proposition}\label{prop:my_CVDETDetailed}
		Consider atomic probability measures $\alpha_N=\sum_{i=1}^N a_i \delta_{x_i'}$ with $a_i>0$ and $\beta_N=\sum_{i=1}^N b_i \delta_{x_i}$, with support $\mathcal{X}\subset \mathbb{R}^d$. Denote by $\mathcal{P}^{\mathrm{OT}}_{N,\epsilon,i,j}$ the element at the intersection of $i$-th row and $j$-th column of the $\epsilon$-regularized optimal transport coupling $\mathcal{P}^{\mathrm{OT}}_{N,\epsilon}$ between $\alpha_N$ and $\beta_N$, and define $\tilde{\beta}_N=\sum_{i=1}^N a_i \delta_{\tilde{x}_{i,N,\epsilon}}$, where $\tilde{x}_{i,N,\epsilon}=\frac{1}{a_i}\sum_{j=1}^{N}\mathcal{P}_{N,\epsilon,i,j}^{\mathrm{OT}}x_j$. Let $\alpha, \beta$ be two other probability measures, also supported on $\mathcal{X}$, such that there exists a unique $\lambda$-Lipschitz optimal transport map $\mathbf{T}(\cdot):\mathcal{X}\rightarrow\mathcal{X}$ between them. Then for any bounded ${k}$-Lipschitz function $\psi(\cdot):\mathcal{X}\rightarrow \mathbb{R}$, we have
		\begin{align}\label{eq:my_UpperBoundonDET}
			\left|\beta_N(\psi)-\tilde{\beta}_N(\psi)\right|
			\leq &\sqrt{2}k\Big(2\lambda^{1/2} 
			\mathcal{E}^{1/2}\left[\mathfrak{d}+ \mathcal{E}\right]^{1/2}\nonumber\\
			&+ \max\{\lambda,  1\}\left[\wass_2(\alpha_N, \alpha)+  \wass_2(\beta_N, \beta)\right]\Big)\,,
		\end{align}
		where $\mathfrak{d}\coloneqq \underset{x,x'\in\mathcal{X}}{\sup}{||x-x'||_2}$ and $\mathcal{E}= \wass_2(\alpha_N, \alpha)+  \wass_2(\beta_N, \beta) + \sqrt{2\epsilon \log N}$.
	\end{proposition} 
	\begin{proof}
		By definition, we have $\tilde{\beta}_N(\mathrm{d}\tilde{x})=\int \alpha_N(\mathrm{d}x')  \delta_{\mathbf{T}_{N,\epsilon}(x')}(\mathrm{d}\tilde{x})$ with $\mathbf{T}_{N,\epsilon}(x')\coloneqq\int {x} \mathcal{P}^{{\mathrm{OT}}}_{N,\epsilon}(\mathrm{d}{x}|x')$ while, as $\mathcal{P}^{\mathrm{OT}}_{N,\epsilon}$ belongs to $\mathcal{U}(\alpha_N,\beta_N)$, we also have $\beta_N(\mathrm{d}{x})=\int \alpha_N(\mathrm{d}x') \mathcal{P}^{{\mathrm{OT}}}_{N,\epsilon}(\mathrm{d}{x}|x')$. We then have for any 1-Lipschitz function
		\begin{align*}
			&\left|\beta_N(\psi)-\tilde{\beta}_N(\psi)\right|\\
			=&\left|\int \left[\int (\psi({x})-\psi(\mathbf{T}_{N,\epsilon}(x'))) \mathcal{P}^{{\mathrm{OT}}}_{N,\epsilon}(\mathrm{d}x|x')\right] \alpha_N(\mathrm{d}x') \right|\\
			\leq& \iint \left|\psi({x})-\psi(\mathbf{T}_{N,\epsilon}(x'))\right|\alpha_N(\mathrm{d}x') \mathcal{P}^{{\mathrm{OT}}}_{N,\epsilon}(\mathrm{d}{x}|x') \\
			\leq& k\iint ||{x}-\mathbf{T}_{N,\epsilon}(x')||  \mathcal{P}^{{\mathrm{OT}}}_{N,\epsilon}(\mathrm{d}x',\mathrm{d}x) \\
			\leq& k\left(\iint ||{x}-\mathbf{T}_{N, \epsilon}(x')||^2\mathcal{P}^{{\mathrm{OT}}}_{N,\epsilon}(\mathrm{d}x',\mathrm{d}{x})\right)^{\frac 1 2}\\
			\leq& k \left(\iint\Big(||{x}-\mathbf{T}(x')||^2+||\mathbf{T}(x')-\mathbf{T}_{N,\epsilon}(x')||^2\Big)\mathcal{P}^{{\mathrm{OT}}}_{N,\epsilon}(\mathrm{d}x',\mathrm{d}{x})\right)^{\frac 1 2}\\
			\leq&\sqrt{2}k\left(\iint||{x}-\mathbf{T}(x')||^2\mathcal{P}^{{\mathrm{OT}}}_{N,\epsilon}(\mathrm{d}x',\mathrm{d}{x})\right)^{\frac 1 2}\,,
		\end{align*}
		where the final inequality follows from Lemma~\ref{lemma:transport_map_distance}.
		The stated result is then obtained using Lemma~\ref{lem:linochetto}.
	\end{proof}
	
	\begin{lemma}
		\label{lemma:predictive_bound}
		For all bounded $k$-Lipschitz function $\psi(\cdot):\mathcal{X}\rightarrow\mathbb{R}$, when the entropy-regularization hyperparameter $\epsilon_N=o(1/\log{N})$, for $\alpha_N^{(t)}$, $\beta_N^{(t-1)}$ defined as in Eq.~\eqref{eq:notations_empirical}, and transition kernel $f(\cdot)$ defined by $p(x_t|x_{t-1};\theta)$ in Eq.~\eqref{eq:dynamic_density}, the expectations $\alpha_N^{(t)}(\psi)$ and $\beta_N^{(t-1)}f(\psi)$ satisfy:
		\begin{align}
			\mathbb{E}\bigg[\bigg(\alpha_N^{(t)}(\psi)-\beta_{N}^{(t-1)}f(\psi)\bigg)^2\bigg]\leq \mathcal{C}||\psi||_\infty^2\,,
		\end{align}
		where  
		\begin{align}
			\label{eq:constant_mathcal_C}
			\mathcal{C}:=&\mathcal{C}(\lambda, k, \eta, \zeta, \tau, N, p, q, d_\mathcal{X})\\
            =&4k\sqrt{\mathcal{Q}}+4\sqrt{2}k\eta\bigg(2\sqrt{3\lambda(1+\zeta)\tau\sqrt{\mathcal{Q}}}\\
            &+\max\{\lambda,  1\}(1+\zeta)\sqrt{\mathcal{Q}}\bigg)\,,
		\end{align}
		is a constant depending on $\lambda, k, \eta, \zeta, \tau, N, p, q, d_\mathcal{X}$. The function $\mathcal{Q}$ is defined as $\mathcal{Q}:=\mathcal{Q}(\tau, N,p ,d_\mathcal{X}, q):={C_1\tau^{p}\mathcal{H}(N, p,d_\mathcal{X},q)}$, $C_1$ is a constant depending only on $p$, $d_\mathcal{X}$, $q$, and $\mathcal{H}(N, p,d_\mathcal{X},q)$ is defined as 
		\begin{align}
			&\mathcal{H}(N, p,d_\mathcal{X},q)\nonumber\\
            &=\begin{cases}N^{-1 / 2}+N^{-(q-p) / q} & \text { if } p>d_\mathcal{X} / 2 \\&\text { and }  q \neq 2 p\,, \\ N^{-1 / 2} \log (1+N)+N^{-(q-p) / q} & \text { if } p=d_\mathcal{X} / 2 \\ &\text { and } q \neq 2 p\,,\\ N^{-p / d_\mathcal{X}}+N^{-(q-p) / q} & \text { if } p \in(0, d_\mathcal{X} / 2)  \\&\text { and } q \neq d_\mathcal{X} /(d_\mathcal{X}-p)\,,\end{cases}
		\end{align}
		where $q>p$ is a constant satisfying $\int_\mathcal{X}|x|^q{\alpha'}_N^{(t)}(\rd x)< \infty$ and $\int_\mathcal{X}|x|^q{\alpha'}_N^{(t-1)}(\rd x)< \infty$, $p=2$ is the order of Wasserstein distances as detailed in the proof, and $\tau=\frac{\mathfrak{d}}{||\psi||_{\infty}}$. Besides, for large enough $N$ and $d_\mathcal{X}$ such that $\mathcal{Q}\leq\sqrt{\mathcal{Q}}\leq\sqrt[4]{\mathcal{Q}}$ and $\mathcal{H}(N, p,d_\mathcal{X},q)=N^{-p / d_\mathcal{X}}+N^{-(q-p) / q}\leq2N^{-p / d_\mathcal{X}}$, we also have that:
		\begin{align}
			\mathbb{E}\bigg[\bigg(\alpha_N^{(t)}(\psi)-\beta_{N}^{(t-1)}f(\psi)\bigg)^2\bigg]
			&\leq \frac{\tilde{\mathcal{C}}||\psi||_\infty^2}{N^{{1}/{2d_\mathcal{X}}}}\,,
		\end{align}
		where 
		\begin{align}
			\label{eq:constant_tilde_mathcal_C}
			\tilde{\mathcal{C}}:=&\tilde{\mathcal{C}}(\lambda, k, \eta, \mathfrak{d}, \tau, N, p, q, d_\mathcal{X})\\
            :=&4k\sqrt[4]{2C_1\mathfrak{d}^2}\bigg(1+\sqrt{2}\eta\Big(2\sqrt{3\lambda(1+\zeta)\tau}\\
            &+\max\{\lambda,  1\}(1+\zeta)\Big)\bigg)
		\end{align}
		is a constant.
	\end{lemma}
	
	\begin{proof}

		We first decompose $\alpha_N^{(t)}(\psi)-\beta_{N}^{(t-1)}f(\psi)$ into two terms:
		\begin{align}
			\label{eq:predictive_error_decomp}
			&\alpha_N^{(t)}(\psi)-\beta_{N}^{(t-1)}f(\psi)\\=&\bigg(\alpha_N^{(t)}(\psi)-{\alpha'}^{(t)}_N(\psi)\bigg) + \bigg({\alpha'}_N^{(t)}(\psi)-\beta_{N}^{(t-1)}f(\psi)\bigg)\,,
		\end{align}
		where ${\alpha'}^{(t)}_N$ is defined as ${\alpha'}^{(t)}_N\coloneqq\tilde{\beta}_{N}^{(t-1)}f$. The first term in Eq.~\eqref{eq:predictive_error_decomp} can be bounded by applying Lemma~\ref{lemma:wass_exp} to probability measures $\alpha_N^{(t)}(\psi)$ and ${\alpha'}_N^{(t)}(\psi)$:
		\begin{align}
			\bigg|\alpha_N^{(t)}(\psi)-{\alpha'}^{(t)}_N(\psi)\bigg|\leq k\wass_1({\alpha}^{(t)}_N, {\alpha'}^{(t)}_N)\leq k\wass_2({\alpha}^{(t)}_N, {\alpha'}^{(t)}_N)\,.
		\end{align}
		We denote by $M_q(\rho)$ the $q$-th moment $\int_\mathcal{X}|x|^q\rho(\rd x)$ of a probability measure $\rho$ defined on $\mathcal{X}$, and assume that $\mathcal{X}$ contains the origin $\textbf{0}_{d_x}$, otherwise we can add a constant to the diameter $\mathfrak{d}$, such that
		\begin{align}
			|x|^q&\leq\mathfrak{d}^q= \bigg(\tau||\psi||_{\infty}\bigg)^q\,,
		\end{align}
		where $\tau=\frac{\mathfrak{d}}{||\psi||_{\infty}}$.	Assume $M_{q}({\alpha'}_N^{(t)})<\infty$ for some $q>p=2$, following Theorem 1 of~\cite{fournier2015rate}  and notice that $\big|\alpha_N^{(t)}(\psi)-{\alpha'}^{(t)}_N(\psi)\big| \leq 2||\psi||_\infty$, we have that for all $N\geq 1$:
		\begin{align}
			&\mathbb{E}\bigg[\big|\alpha_N^{(t)}(\psi)-{\alpha'}^{(t)}_N(\psi)\big|^2\bigg]\\
                \leq& 2k||\psi||_\infty\mathbb{E}\bigg[\wass_2({\alpha}^{(t)}_N, {\alpha'}^{(t)}_N)\bigg] \\
			\leq& 2k||\psi||_\infty\sqrt{C_1M^{p/q}_{q}({\alpha'}^{(t)}_N)\mathcal{H}(N, p,d_\mathcal{X},q)}\\
			\leq& 2k||\psi||_\infty\sqrt{{C}_1\tau^p||\psi||_{\infty}^{p}\mathcal{H}(N, p,d_\mathcal{X},q)}\\
			\leq& 2k\sqrt{\mathcal{Q}}||\psi||_\infty^2\\
			\leq& \mathcal{C}_1||\psi||_\infty^2
			\label{eq:decompose_predictive_term_1}
		\end{align}
		where $\mathcal{C}_1:=\mathcal{C}_1(k,\tau, N, p, d_\mathcal{X}, q):=2k\sqrt{\mathcal{Q}(\tau, N,p ,d_\mathcal{X}, q)}$, $p$ is the order of the Wasserstein distance ($p=2$ in this case),  $C_1$ is a constant depending only on $p$, $d$, $q$, $\mathcal{Q}:=\mathcal{Q}(\tau, N,p ,d_\mathcal{X}, q):={C_1\tau^{p}\mathcal{H}(N, p,d_\mathcal{X},q)}$, and $\mathcal{H}(N, p,d_\mathcal{X},q)$ is defined as 
		\begin{align}
			&\mathcal{H}(N, p,d_\mathcal{X},q)\nonumber\\
            &=\begin{cases}N^{-1 / 2}+N^{-(q-p) / q} & \text { if } p>d_\mathcal{X} / 2 \\&\text { and }  q \neq 2 p\,, \\ N^{-1 / 2} \log (1+N)+N^{-(q-p) / q} & \text { if } p=d_\mathcal{X} / 2 \\ &\text { and } q \neq 2 p\,,\\ N^{-p / d_\mathcal{X}}+N^{-(q-p) / q} & \text { if } p \in(0, d_\mathcal{X} / 2)  \\&\text { and } q \neq d_\mathcal{X} /(d_\mathcal{X}-p)\,,\end{cases}
		\end{align}
		For the second term in Eq.~\eqref{eq:predictive_error_decomp}, by Assumption~\ref{ass:transition_lipcontraction},~\ref{ass:weight_lipcontraction} and Proposition~\ref{prop:my_CVDETDetailed}:
		\begin{align}
			&\big|{\alpha'}_N^{(t)}(\psi)-\beta_{N}^{(t-1)}f(\psi)\big|\\
                \leq& \eta\big|\tilde{\beta}_{N,\epsilon}^{(t-1)}(\psi)-\beta_{N}^{(t-1)}(\psi)\big|\\
			\leq& \sqrt{2}\eta k\Bigg(2\lambda^{1/2} 
			\mathcal{E}^{1/2}\left[\mathfrak{d}+ \mathcal{E}\right]^{1/2} \nonumber\\ &+ \max\{\lambda,  1\}\bigg[\wass_2\bigg(\alpha_N^{(t-1)}, {\alpha'}^{(t-1)}_N\bigg) \nonumber\\&+  \wass_2\bigg(\beta_N^{(t-1)}, \frac{\omega{\alpha'}^{(t-1)}_N}{{\alpha'}^{(t-1)}_N(\omega)}\bigg)\bigg]\Bigg)\\
			\leq& \sqrt{2}\eta k\Bigg(2\lambda^{1/2} \mathcal{E}^{1/2}\left[\mathfrak{d}+ \mathcal{E}\right]^{1/2}\nonumber\\&+ \max\{\lambda,  1\}\left[(1+\zeta)\wass_2\bigg(\alpha_N^{(t-1)}, {\alpha'}^{(t-1)}_N\bigg)\right]\Bigg)\,,
		\end{align}
		where $\lambda$ is the Lipschitz constant of the optimal transport map $\mathbf{T}(\cdot):\mathcal{X}\rightarrow\mathcal{X}$ between $ {\alpha'}^{(t-1)}_N$ and $\frac{\omega{\alpha'}^{(t-1)}_N}{{\alpha'}^{(t-1)}_N(\omega)}$, and $\mathcal{E}:=\wass_2\bigg(\alpha_N^{(t-1)}, {\alpha'}^{(t-1)}_N\bigg)+  \wass_2\bigg(\beta_N^{(t-1)}, \frac{\omega{\alpha'}^{(t-1)}_N}{{\alpha'}^{(t-1)}_N(\omega)}\bigg)+\sqrt{2\epsilon\log(N)}$. Again following Theorem 1 of~\cite{fournier2015rate}, assume $M_{q}({\alpha'}_N^{(t-1)})<\infty$ for $q>p=2$, we have that for all $N\geq1$:
		\begin{align}
			&\mathbb{E}[\mathcal{E}]\\:=&\mathbb{E}\Big[\wass_2\bigg(\alpha_N^{(t-1)}, {\alpha'}^{(t-1)}_N\bigg)+  \wass_2\bigg(\beta_N^{(t-1)}, \frac{\omega{\alpha'}^{(t-1)}_N}{{\alpha'}^{(t-1)}_N(\omega)}\bigg)\nonumber\\&+\sqrt{2\epsilon\log(N)}\Big]\\
			\leq& \mathbb{E}\Big[(1+\zeta)\wass_2\bigg(\alpha_N^{(t-1)}, {\alpha'}^{(t-1)}_N\bigg)+\sqrt{2\epsilon\log(N)}\Big]\\
			\leq& (1+\zeta)\mathbb{E}\Bigg[\wass_2\bigg(\alpha_N^{(t-1)}, {\alpha'}^{(t-1)}_N\bigg)\Bigg]+\mathbb{E}\Big[\sqrt{2\epsilon\log(N)}\Big]\\
			\leq& (1+\zeta)||\psi||_{\infty}\sqrt{C_1\tau^{p}\mathcal{H}(N, p,d_\mathcal{X},q)}+\mathbb{E}\Big[\sqrt{2\epsilon\log(N)}\Big]\,,
		\end{align}
		and
		\begin{align}
			\mathbb{E}[\mathcal{E}^2]&\leq \mathbb{E}\Big[\mathcal{E}\big(2\mathfrak{d}+\sqrt{2\epsilon\log(N)}\big)\Big]\\
			&=2\mathfrak{d}\mathbb{E}[\mathcal{E}]+\mathbb{E}\Big[\mathcal{E}\big(\sqrt{2\epsilon\log(N)}\big)\Big]\\
			&\leq2\mathfrak{d}\mathbb{E}[\mathcal{E}]+\mathbb{E}\Big[2\mathfrak{d}\sqrt{2\epsilon\log(N)}+{2\epsilon\log(N)}\Big]
		\end{align}
		
		Let $\epsilon_N=o(1/\log N)$ such that $\mathbb{E}\Big[\sqrt{2\epsilon\log(N)}\Big]=0$ and $\mathbb{E}\Big[{2\epsilon\log(N)}\Big]=0$, we now have:
		\begin{align}
			\mathbb{E}[\mathcal{E}]&\leq(1+\zeta)||\psi||_{\infty}\sqrt{\mathcal{Q}}\,,\\
			\mathbb{E}[\mathcal{E}^2]&\leq 2\mathfrak{d}(1+\zeta)||\psi||_{\infty}\sqrt{\mathcal{Q}}\,.
		\end{align}
		
		Therefore we have that:
		\begin{align}
			&\quad\mathbb{E}\Big[\big|{\alpha'}_N^{(t)}(\psi)-\beta_{N}^{(t-1)}f(\psi)\big|\Big]\\
			\leq& \mathbb{E}\Bigg[\sqrt{2}\eta k\Bigg(2\lambda^{1/2} \mathcal{E}^{1/2}\left[\mathfrak{d}+ \mathcal{E}\right]^{1/2}\nonumber\\ &+ \max\{\lambda,  1\}\left[(1+\zeta)\wass_2\bigg(\alpha_N^{(t-1)}, {\alpha'}^{(t-1)}_N\bigg)\right]\Bigg)\Bigg]\\
			\leq& \sqrt{2}\eta k\mathbb{E}\Big[2\lambda^{1/2} \left(\mathfrak{d}\mathcal{E}+ \mathcal{E}^2\right)^{1/2}\nonumber\\&+ \max\{\lambda,  1\}(1+\zeta)||\psi||_{\infty}\sqrt{C_1\tau^{p}\mathcal{H}(N, p,d_\mathcal{X},q)}\Big]\\
			\leq& \sqrt{2}\eta k\Bigg(2\lambda^{1/2}\sqrt{\mathbb{E}\Big[\mathfrak{d}\mathcal{E}+ \mathcal{E}^2\Big]}\nonumber\\&+ \max\{\lambda,  1\}(1+\zeta)||\psi||_{\infty}\sqrt{C_1\tau^{p}\mathcal{H}(N, p,d_\mathcal{X},q)}\Bigg)\\
			\leq& \sqrt{2}\eta k||\psi||_{\infty}\bigg(2\sqrt{3\lambda(1+\zeta)\tau\sqrt{\mathcal{Q}}}\nonumber\\&+\max\{\lambda,  1\}(1+\zeta)\sqrt{\mathcal{Q}}\bigg)
		\end{align}
		
		In addition, notice that 
		$
		\big|{\alpha'}_N^{(t)}(\psi)-\beta_{N}^{(t-1)}f(\psi)\big|\leq 2||\psi||_\infty
		$,
		therefore,
		\begin{align}
			&\mathbb{E}\Big[\big|{\alpha'}_N^{(t)}(\psi)-\beta_{N}^{(t-1)}f(\psi)\big|^2\Big]\\\leq& 2||\psi||_\infty\mathbb{E}\Big[\big|{\alpha'}_N^{(t)}(\psi)-\beta_{N}^{(t-1)}f(\psi)\big|\Big]\\
			\leq& 2\sqrt{2}k\eta\bigg(2\sqrt{3\lambda(1+\zeta)\tau\sqrt{\mathcal{Q}}}+\max\{\lambda,  1\}(1+\zeta)\sqrt{\mathcal{Q}}\bigg)||\psi||_\infty^2\\
			=&\mathcal{C}_2||\psi||_\infty^2\,,
			\label{eq:decompose_predictive_term_2}
		\end{align}
		where $\mathcal{C}_2:=\mathcal{C}_2(\lambda, k, \eta, \zeta, \tau, N, p, q, d_\mathcal{X}):= 2\sqrt{2}k\eta\bigg(2\sqrt{3\lambda(1+\zeta)\tau\sqrt{\mathcal{Q}}}+\max\{\lambda,  1\}(1+\zeta)\sqrt{\mathcal{Q}}\bigg)$.
		
		From Eq.~\eqref{eq:decompose_predictive_term_1} and Eq.~\eqref{eq:decompose_predictive_term_2}, we have that
		\begin{align}
			&\mathbb{E}\bigg[\bigg(\alpha_N^{(t)}(\psi)-\beta_{N}^{(t-1)}f(\psi)\bigg)^2\bigg]\\
                \leq&\mathbb{E}\bigg[\bigg(\alpha_N^{(t)}(\psi)-{\alpha'}^{(t)}_N(\psi)\bigg)^2 + \bigg({\alpha'}_N^{(t)}(\psi)-\beta_{N}^{(t-1)}f(\psi)\bigg)^2\bigg]\\
			\leq& 2\mathcal{C}_1||\psi||_\infty+2\mathcal{C}_2||\psi||_\infty\\
			\leq& \mathcal{C}||\psi||_\infty^2\,,
		\end{align}
		where $\mathcal{C}:=\mathcal{C}(\lambda, k, \eta, \zeta, \tau, N, p, q, d_\mathcal{X})=4k\sqrt{\mathcal{Q}}+4\sqrt{2}k\eta\bigg(2\sqrt{3\lambda(1+\zeta)\tau\sqrt{\mathcal{Q}}}+\max\{\lambda,  1\}(1+\zeta)\sqrt{\mathcal{Q}}\bigg)$.
		Besides, for large enough $N$ and $d_\mathcal{X}$ such that $\mathcal{Q}\leq\sqrt{\mathcal{Q}}\leq\sqrt[4]{\mathcal{Q}}$ and $\mathcal{H}(N, p,d_\mathcal{X},q)=N^{-p / d_\mathcal{X}}+N^{-(q-p) / q}\leq2N^{-p / d_\mathcal{X}}$ with $p=2$, we also have that:
		\begin{align}
			&\mathbb{E}\bigg[\bigg(\alpha_N^{(t)}(\psi)-\beta_{N}^{(t-1)}f(\psi)\bigg)^2\bigg]\\
			\leq& ||\psi||_\infty^2\Bigg[4k\sqrt{\mathcal{Q}}+4\sqrt{2}k\eta\bigg(2\sqrt{3\lambda(1+\zeta)\tau\sqrt{\mathcal{Q}}}\\&+\max\{\lambda,  1\}(1+\zeta)\sqrt{\mathcal{Q}}\bigg)\Bigg]\\
			\leq& ||\psi||_\infty^2\Bigg[4k\sqrt[4]{\mathcal{Q}}+4\sqrt{2}k\eta\sqrt[4]{\mathcal{Q}}\bigg(2\sqrt{3\lambda(1+\zeta)\tau}\\
                &+\max\{\lambda,  1\}(1+\zeta)\bigg)\Bigg]\\
			\leq& \frac{\tilde{\mathcal{C}}||\psi||_\infty^2}{N^{{1}/{2d_\mathcal{X}}}}\,,
		\end{align}
		where $\tilde{\mathcal{C}}:=\tilde{\mathcal{C}}(\lambda, k, \eta, \mathfrak{d}, \tau, N, p, q, d_\mathcal{X}):=4k\sqrt[4]{2C_1\mathfrak{d}^2}\bigg(1+\sqrt{2}\eta\Big(2\sqrt{3\lambda(1+\zeta)\tau}+\max\{\lambda,  1\}(1+\zeta)\Big)\bigg)$.
	\end{proof}
	
	\begin{lemma}
		\label{lemma:posterior_weight_approximation}
		Provided the weight function $\omega(\cdot):\mathcal{X}\rightarrow\mathbb{R}$ is upper bounded, for all measurable and bounded function $\psi(\cdot):\mathcal{X}\rightarrow\mathbb{R}$,
		\begin{align}
			\mathbb{E}\Bigl[\big|\beta_N^{(t)}(\psi)-{\beta'}_N^{(t)}(\psi)\big|^2\Bigr]\leq||\psi||_\infty^2\mathbb{E}\Big[(\alpha_N^{(t)}(\bar{\omega})-1)^2\Big]\,,
		\end{align}
		where $\bar{\omega}(\cdot):\mathcal{X}\rightarrow\mathbb{R}$ is defined as $\bar{\omega}(x)=\frac{\omega(x)}{\alpha^{(t)}(\omega)}$, and ${\beta'}_N^{(t)}=\frac{\omega\alpha_N^{(t)}}{\alpha^{(t)}(\omega)}$.
	\end{lemma}
	\begin{proof}
		Notice that $$\beta_N^{(t)}(\psi)=\frac{{\beta'}^{(t)}_N(\psi)}{\alpha_N^{(t)}(\bar{\omega})}\,,$$
		therefore, we have that
		\begin{align}
			\beta_N^{(t)}(\psi)-{\beta'}_N^{(t)}(\psi)&=\beta_N^{(t)}(\psi)\Big(1-\alpha_N^{(t)}(\bar{\omega})\Big)\\
			&\leq ||\psi||_\infty\Big(1-\alpha_N^{(t)}(\bar{\omega})\Big)\,.
		\end{align}
		So, we can conclude that 
		\begin{align}
			\mathbb{E}\Bigl[\big|\beta_N^{(t)}(\psi)-{\beta'}_N^{(t)}(\psi)\big|^2\Bigr]\leq||\psi||_\infty^2\mathbb{E}\Big[(\alpha_N^{(t)}(\bar{\omega})-1)^2\Big]\,.
		\end{align}
	\end{proof}
	
	\begin{customthm}{V.1}
		\label{prop:consistency_appendix}
		For a bounded weight function $\omega_t(x_t)=p(y_t|x_t;\theta):\mathcal{X}\rightarrow\mathbb{R}$ and a measurable bounded $k$-Lipschitz function $\psi(\cdot):\mathcal{X}\rightarrow\mathbb{R}$, when the regularization coefficient in entropy-regularized optimal transport resampler $\epsilon_N=o(1/\log{N})$, there exist constants $c_t$ and ${c'}_t$ such that for $t\geq0$
		\begin{align}
			\label{eq:predictive_true_exp_appendix}
			\mathbb{E}\Bigg[\bigg({\alpha}_N^{(t)}(\psi)-{\beta}^{(t-1)}f(\psi)\bigg)^2\Bigg]\leq c_t\frac{||\psi||_\infty^2}{N^{1/2d_\mathcal{X}}}
		\end{align}
		(replacing ${\beta}^{(t-1)}f$ by the initial distribution $\pi(x_0,\theta)$ at time $t=0$ defined in Eq.~\eqref{eq:initial_dist}) and
		\begin{align}
			\label{eq:posterior_true_exp_appendix}
			\mathbb{E}\Biggl[\bigg(\beta_{N}^{(t)}(\psi)-\beta^{(t)}(\psi)\bigg)^2\Biggr]\leq {c'}_t\frac{||\psi||_\infty^2}{N^{1/2d_\mathcal{X}}}\,,
		\end{align}
		where ${\beta}^{(t)}$ and ${\alpha}_N^{(t)}$ are respectively defined by~Eqs.~\eqref{eq:notations} and~\eqref{eq:notations_empirical}, and $f(\cdot)$ is a transition kernel defined by $p(x_t|x_{t-1};\theta)$ in Eq.~\eqref{eq:dynamic_density}.    
	\end{customthm}
	
	\begin{proof}
		We prove the above statement by induction. Firstly, Eq.~\eqref{eq:predictive_true_exp_appendix} holds at time $t=0$ with $c_1=4$:
		\begin{align}
			&\mathbb{E}_{x_0^i\overset{\text{i.i.d}}{\sim}\pi}\Bigg[\bigg(\frac{1}{N}\sum_{i=1}^{N}\psi(x_0^i)-\pi(\psi)\bigg)^2\Bigg]\\
                =&\text{Var}\Bigg(\frac{1}{N}\sum_{i=1}^{N}\psi(x_0^i)\Bigg)\\
			=&\frac{1}{N}\text{Var}\Big(\psi(x)\Big)\\
			=&\frac{1}{N}\mathbb{E}_{X\sim\mu}\bigg[\Big(\psi(x)-\mu(\psi)\Big)^2\bigg]\\
			\leq& \frac{1}{N}\big(2||\psi||_\infty\big)^2\\
			\leq& \frac{4||\psi||_\infty^2}{N^{1/2d_\mathcal{X}}}\,.
		\end{align}
		Assume Eq.~\eqref{eq:predictive_true_exp_appendix} holds at time $t\geq0$, we have:
		\begin{align}
			\label{eq:posterior_true_decompose_appendix}
			&\beta_{N}^{(t)}(\psi)-\beta^{(t)}(\psi)\\=&\Big(\beta_N^{(t)}(\psi)-{\beta'}_N^{(t)}(\psi)\Big)\nonumber+\Big({\beta'}_N^{(t)}(\psi)-\beta^{(t)}(\psi)\Big)\,.
		\end{align}
		The MSE of the second term in the r.h.s of Eq.~\eqref{eq:posterior_true_decompose_appendix} can be bounded by applying Eq.~\eqref{eq:predictive_true_exp_appendix} to function $\bar{\omega}\psi$. The MSE of the first term in the r.h.s of Eq.~\eqref{eq:posterior_true_decompose_appendix} can be bounded by applying Lemma~\ref{lemma:posterior_weight_approximation}, then Eq.~\eqref{eq:predictive_true_exp_appendix} to function $\bar{\omega}$ (using the fact that ${\beta}^{(t-1)}f(\bar{\omega})=1$). Therefore
		\begin{align}
			&\mathbb{E}\Biggl[\bigg(\beta_{N}^{(t)}(\psi)-\beta^{(t)}(\psi)\bigg)^2\Biggr]\\
                \leq& 2\mathbb{E}\Biggl[\bigg(\beta_N^{(t)}(\psi)-{\beta'}_N^{(t)}(\psi)\bigg)^2\Biggr] +2\mathbb{E}\Biggl[\bigg({\beta'}_N^{(t)}(\psi)-\beta^{(t)}(\psi)\bigg)^2\Biggr]\\
			\leq& 4c_t\frac{||\psi||_\infty^2||\bar{\omega}||_\infty^2}{N^{1/2d_\mathcal{X}}}\,,
		\end{align}
		where we can obtain Eq.~\eqref{eq:posterior_true_exp_appendix} with $ {c'}_t=4c_t\frac{||\bar{\omega}||_\infty^2}{N^{1/2d_\mathcal{X}}}$.
		
		We then prove that Eq.~\eqref{eq:posterior_true_exp_appendix} at time $t-1$ implies Eq.~\eqref{eq:predictive_true_exp_appendix} at time $t$. Firstly, we have that
		\begin{align}
			\label{eq:predictive_true_decompose_appendix}
			&{\alpha}_N^{(t)}(\psi)-{\beta}^{(t-1)}f(\psi) \\=& \Big(\alpha_N^{(t)}(\psi)-\beta_{N}^{(t-1)}f(\psi)\Big) + \Big(\beta_{N}^{(t-1)}f(\psi) - {\beta}^{(t-1)}f(\psi)\Big)\,.
		\end{align}
		The MSE of the first term in the r.h.s of Eq.~\eqref{eq:predictive_true_decompose_appendix} can be bounded by applying Lemma~\ref{lemma:predictive_bound}, and the MSE of the second term in the r.h.s of Eq.~\eqref{eq:predictive_true_decompose_appendix} can by bounded by applying Eq.~\eqref{eq:posterior_true_exp_appendix} (at time $t-1$, replacing $\psi$ by $f\psi$), therefore we have that Eq.~\eqref{eq:predictive_true_exp_appendix} holds:
		\begin{align}
			&\mathbb{E}\Bigg[\bigg({\alpha}_N^{(t)}(\psi)-{\beta}^{(t-1)}f(\psi)\bigg)^2\Bigg]\\\leq& 2\mathbb{E}\Bigg[\Big(\alpha_N^{(t)}(\psi)-\beta_{N}^{(t-1)}f(\psi)\Big)^2\Bigg]\\&+2\mathbb{E}\Bigg[\Big(\beta_{N}^{(t-1)}f(\psi) - {\beta}^{(t-1)}f(\psi)\Big)^2\Bigg]\\
			\leq& 2\bigg[\big(\tilde{\mathcal{C}}+{c'}_t\big)\frac{||\psi||_\infty^2}{N^{1/2d_\mathcal{X}}}\bigg]\,.
		\end{align}
		So, Eq.~\eqref{eq:posterior_true_exp_appendix} at time $t-1$ leads to Eq.~\eqref{eq:predictive_true_exp_appendix} at time step $t$ with $c_t=2\big(\tilde{\mathcal{C}}+{c'}_t\big)$, where $\tilde{\mathcal{C}}$ is defined by Eq.~\eqref{eq:constant_tilde_mathcal_C}.
		
		To summarize, we have Eq.~\eqref{eq:predictive_true_exp_appendix} to hold at time 0, and Eq.~\eqref{eq:predictive_true_exp_appendix} at time step $t$ implies Eq.~\eqref{eq:posterior_true_exp_appendix} to hold at time step $t$. Then Eq.~\eqref{eq:posterior_true_exp_appendix} at time step $t$ lead to Eq.~\eqref{eq:predictive_true_exp_appendix} at time step $t+1$, therefore we can conclude that Eq.~\eqref{eq:predictive_true_exp_appendix} and Eq.~\eqref{eq:posterior_true_exp_appendix} hold for $\forall t \geq 0$.
	\end{proof}

 \section{Detailed Experiment Setups}
 \subsection{One-dimensional Linear Gaussian State-space Models}
 In this experiment, we use conditional planar flow for the construction of proposal distributions in NF-DPFs~\cite{rezende2015variational}. Normalizing flows are not used to construct the dynamic model and the measurement model in this experiment because the functional form of them are assumed to be known and we aim to approximate their parameters, following the setup in~\cite{le2018auto,maddison2017filtering,naesseth2018variational}. 
 
 At time step $t+1$, given a particle $x_{t}^i$, the following equations shows how new particles $x_{t+1}^i$ are generated from the proposal distribution of NF-DPFs in this experiment:
 \begin{gather}
     \hat{x}_{t+1}^i\sim \mathcal{N}(\theta_1 x_{t}, 1) \text{ for } t\geq1\,,,\\
     {x}_{t+1}^i=\mathcal{F}_\phi(\hat{x}_{t+1}^i;y_{t+1})\sim q( {x}_{t+1}| {x}_{t},  {y}_{t+1};\phi)\,.
 \end{gather}
 where $\mathcal{F}_\phi(\hat{x}_{t+1}^i;y_{t+1})$ is a planar flow that is defined as:
\begin{gather}
    \label{eq:appendices_conditional_planar_flow}
    \mathcal{F}_\phi(\hat{x}_{t+1}^i;y_{t+1})=\hat{x}_{t+1}^i+v{h}(w \hat{x}_{t+1}^i+b y_{t+1})\,,
\end{gather}
where $v\in\mathbb{R}$, $w\in\mathbb{R}$, and $b\in\mathbb{R}$ are learnable parameters, we use a tanh function as the smooth non-linear function $h(\cdot):\mathbb{R}\rightarrow \mathbb{R}$.

We employed the Adam optimizer to update the parameters, with a fixed learning rate of 0.002 and optimize the model for 500 iterations. We set number of particles as $N=100$ for both training, validation, and testing stages. The particles are resampled when effective sample size is smaller than 50. 

 \subsection{Multivariate Linear Gaussian State-space Models}
 \label{appendices:subsec_mlgssm}
We use conditional Real-NVP models to construct the proposal distributions of NF-DPFs in this experiment~\cite{winkler2019learning}. Following the setup in~\cite{corenflos2021differentiable}, as the functional forms of the dynamic model and the measurement model are given and the we aim to learn their parameters, normalizing flows are not employed to construct the dynamic model and the measurement model.

In this experiment, the conditional normalizing flow $\mathcal{F}_\phi(\cdot;\cdot):\mathbb{R}^{d_\mathcal{X}+d_\mathcal{Y}}\rightarrow \mathbb{R}^{d_\mathcal{X}}$ used in the construction of NF-DPF's proposal distributions is defined as a composition of $K$ conditional normalizing flows  $\mathcal{F}_{\phi_k}(\cdot;\cdot):\mathbb{R}^{d_\mathcal{X}+d_\mathcal{Y}}\rightarrow \mathbb{R}^{d_\mathcal{X}}$. Denote by $z\in \mathbb{R}^{d_\mathcal{X}}$, $u\in\mathbb{R}^{d_\mathcal{Y}}$ the input and $s\in\mathbb{R}^{d_\mathcal{X}}$ the output of $\mathcal{F}_{\phi_k}(\cdot;\cdot)$, a single component of $\mathcal{F}_\phi(\cdot;\cdot)$ can be formulated as: 
\begin{gather}
    \label{eq:appendices_conditional_coupling_layers_1}
    \underset{1:d}{s'}=\underset{1:d}{z}\,,\\
    \label{eq:appendices_conditional_coupling_layers_2}
    \underset{d+1:d_\mathcal{X}}{s'}=\underset{d+1:d_\mathcal{X}}{z}\odot e^{\tilde{\gamma'}_{\phi_k}(\underset{1:d}{z},u)}+\tilde{\eta'}_{\phi_k}(\underset{1:d}{z},u)\,,\\
    \label{eq:appendices_conditional_coupling_layers_3}
    \underset{1:d}{s}=\underset{1:d}{s'}\odot e^{\tilde{\gamma}_{\phi_k}(\underset{d+1:d_\mathcal{X}}{s'},u)}+\tilde{\eta}_{\phi_k}(\underset{d+1:d_\mathcal{X}}{s'},u)\,,\\
    \label{eq:appendices_conditional_coupling_layers_4}
    \underset{d+1:d_\mathcal{X}}{s}=\underset{d+1:d_\mathcal{X}}{s'}\,,
\end{gather}
where $d_\mathcal{X}\in\{2,5,10,25,50,100\}$ is the dimension of $x_t^i$, $d=\lfloor d_\mathcal{X}/2\rfloor$, $\tilde{\gamma}_{\phi_k}(\cdot):\mathbb{R}^{d+d_\mathcal{Y}}\rightarrow\mathbb{R}^{d_\mathcal{X}-d}$, $\tilde{\eta}_{\phi_k}(\cdot):\mathbb{R}^{d+d_\mathcal{Y}}\rightarrow\mathbb{R}^{d_\mathcal{X}-d}$, $\tilde{\gamma'}_{\phi_k}(\cdot):\mathbb{R}^{d_\mathcal{X}+d_\mathcal{Y}-d}\rightarrow\mathbb{R}^{d}$, and $\tilde{\eta'}_{\phi_k}(\cdot):\mathbb{R}^{d_\mathcal{X}+d_\mathcal{Y}-d} \rightarrow\mathbb{R}^{d}$ are constructed with two-layer fully-connected neural networks using tanh as their activation functions. 
Given a particle $x_{t}^i$ at time step $t$, new particles $x_{t+1}^i$ at time step $t+1$ are generated by:
\begin{align} 
    &\hat{x}_{t+1}^i\sim\mathcal{N}({\theta}_1x_{t}^i,  \mathbf{I}_{d_\mathcal{X}})\,,\\
    &x_{t+1}^i = \mathcal{F}_\phi(\hat{x}_{t+1}^i;y_{t+1})\\ 
    &= \mathcal{F}_{\phi_K}(\cdots\mathcal{F}_{\phi_2}(\mathcal{F}_{\phi_1}(\hat{x}_{t+1}^i;y_{t+1});y_{t+1})\cdots;y_{t+1})\,.
\end{align}
Specifically in this experiment, we set $K=1$.

We employed the Adam optimizer to update the parameters, with a fixed learning rate of 0.002 and optimize the model for 500 iterations. We set number of particles as $N=100$ for both training, validation, and testing stages. The particles are resampled when effective sample size is smaller than 50.
 \subsection{Disk Localization}

With a slight abuse of notations, in this experiment, we denote by ${x}_t$ the location of the target disk at the $t$-th time step, and ${a}_t$ the action. The velocity and the location of the target disk at the $t$-th time step can be described as follows~\cite{kloss2021train}:
\begin{align}
    \hat{{a}}_t &= {a}_t + \hat{\varsigma}_t,\;\;\;\;\hat{\varsigma}_t\overset{\text{i.i.d}}{\sim}\mathcal{N}(\textbf{0},\;\; \sigma_{\hat{\varsigma}}^2\mathbf{I})\,\,,\\ 
    {x}_{t+1} &= {x}_t + \hat{{a}}_t + \varsigma_t,\;\;\;\;\varsigma_t\overset{\text{i.i.d}}{\sim}\mathcal{N}(\textbf{0},\;\; \sigma_\varsigma^2\mathbf{I})\,\,,
    \label{eq:dynamic1}
\end{align}
where $\mathbf{I}$ is the identity matrix, $\hat{{a}}_t$ is the noisy action obtained by adding random action noise $\hat{\varsigma}_t$, $\sigma_{\hat{\varsigma}}=4$ is the standard deviation of the action noise, and $\varsigma_t$ is the dynamic noise whose standard deviation is $\sigma_\varsigma=2$. The distractors follow the same dynamic presented above.

In this experiment, we assume that the action $a_t$ is given, but neither the transition of latent state nor the relation between observations and latent state are known. Therefore, the dynamic model, the measurement model, and the proposal distribution of the proposed NF-DPF are constructed with normalizing flows. Specifically, given a particle $x_{t}^i$ at time step $t$, the dynamic model is constructed with a base distribution $g(\cdot|x_{t}^i; \theta)$ and a normalizing flow $\mathcal{T}_\theta(\cdot): \mathbb{R}^{d_\mathcal{X}}\rightarrow \mathbb{R}^{d_\mathcal{X}}$, and one can draw samples from the dynamic model for time step $t+1$ as below:
\begin{gather}
\dot{x}_{t+1}^i = x_{t}^i + a_t + \varepsilon_t^i \sim g(\dot{x}_{t+1}|x_{t}; \theta)\,,\\
{x}_{t+1}^i=\mathcal{T}_\theta(\dot{x}_{t+1}^i)\sim p({x}_{t+1}|x_{t};\theta)\,.
\end{gather}

The normalizing flow $\mathcal{T}_\theta(\cdot)$ is defined as a composition of $K$ conditional normalizing flows  $\mathcal{T}_{\theta_k}(\cdot):\mathbb{R}^{d_\mathcal{X}}\rightarrow \mathbb{R}^{d_\mathcal{X}}$. Denote by $z\in \mathbb{R}^{d_\mathcal{X}}$ the input and $s\in\mathbb{R}^{d_\mathcal{X}}$ the output of $\mathcal{T}_{\theta_k}(\cdot)$, a single component of $\mathcal{T}_{\theta}(\cdot)$ can be formulated as: 
\begin{gather}
    \label{eq:appendices_coupling_layers_1}
    \underset{1:d}{s'}=\underset{1:d}{z}\,,\\
    \label{eq:appendices_coupling_layers_2}
    \underset{d+1:d_\mathcal{X}}{s'}=\underset{d+1:d_\mathcal{X}}{z}\odot e^{\tilde{\gamma'}_{\theta_k}(\underset{1:d}{z})}+\tilde{\eta'}_{\theta_k}(\underset{1:d}{z})\,,\\
    \label{eq:appendices_coupling_layers_3}
    \underset{1:d}{s}=\underset{1:d}{s'}\odot e^{\tilde{\gamma}_{\theta_k}(\underset{d+1:d_\mathcal{X}}{s'})}+\tilde{\eta}_{\theta_k}(\underset{d+1:d_\mathcal{X}}{s'})\,,\\
    \label{eq:appendices_coupling_layers_4}
    \underset{d+1:d_\mathcal{X}}{s}=\underset{d+1:d_\mathcal{X}}{s'}\,,
\end{gather}
where $d_\mathcal{X}$ is the dimension of $x_t^i$, $d=\lfloor d_\mathcal{X}/2\rfloor$, $\tilde{\gamma}_{\theta_k}(\cdot):\mathbb{R}^{d+d_\mathcal{Y}}\rightarrow\mathbb{R}^{d_\mathcal{X}-d}$, $\tilde{\eta}_{\theta_k}(\cdot):\mathbb{R}^{d+d_\mathcal{Y}}\rightarrow\mathbb{R}^{d_\mathcal{X}-d}$, $\tilde{\gamma'}_{\theta_k}(\cdot):\mathbb{R}^{d_\mathcal{X}+d_\mathcal{Y}-d}\rightarrow\mathbb{R}^{d}$, and $\tilde{\eta'}_{\theta_k}(\cdot):\mathbb{R}^{d_\mathcal{X}+d_\mathcal{Y}-d} \rightarrow\mathbb{R}^{d}$ are constructed with two-layer fully-connected neural networks using tanh as their activation functions. 
Given a particle $x_{t}^i$ at time step $t$, new particles $x_{t+1}^i$ at time step $t$ are generated by applying $\mathcal{T}_{\theta}(\cdot)$ to $\dot{x}_{t+1}^i$:
\begin{gather} 
    x_{t+1}^i = \mathcal{T}_\theta(\dot{x}_{t+1}^i) = \mathcal{T}_{\theta_K}(\cdots\mathcal{T}_{\theta_2}(\mathcal{T}_{\theta_1}(\dot{x}_{t+1}^i)\cdots)\,.
\end{gather}
Specifically in this experiment, we set $K=2$ and use Real-NVP models to construct $\mathcal{T}_{\theta}(\cdot)$.

For this experiment, the proposed NF-DPF approximates the conditional likelihood $p(y_t|x_t;\theta)$ of an observation $y_t$ conditioned on a latent state $x_t$ as follows:
\begin{align}
        \label{eq:appendices_obs_feature_disk}
    {e}_t&=U_\theta({y}_t)\,,\\
    \label{eq:appendices_feature_base}
    z_t&=\mathcal{G}^{-1}_\theta(e_t; x_t)\,,\\
    p(y_t|x_t;\theta)&\approx p(e_t|x_t;\theta)\\
    &=p_Z(z_t)\bigg|\operatorname{det} J_{\mathcal{G}_\theta}(z_t; x_t)\bigg|^{-1}\,,
    \label{eq:appendices_measurement_cnf}
\end{align}
where $U_\theta(\cdot):\mathbb{R}^{\mathcal{Y}}\rightarrow\mathbb{R}^{d_e}$ is the encoder function constructed with convolutional neural networks with that transforms $128\times128$ RGB images to $32$-dimensional feature vectors, $\mathcal{G}^{-1}_\theta(\cdot):\mathbb{R}^{d_e}\times\mathcal{X}:\rightarrow \mathbb{R}^{d_e}$ is a conditional normalizing flow built by stacking two conditional Real-NVP models~\cite{winkler2019learning}, and $p_Z(\cdot)$ is the PDF of a standard Gaussian distribution.

Similar to the way we constructed the proposal distribution of NF-DPFs in the multivariate Gaussian experiment we presented in Appendix~\ref{appendices:subsec_mlgssm}, the proposal distribution of NF-DPF in this experiment is constructed with a conditional normalizing flow, the conditional Real-NVP model~\cite{winkler2019learning}, while we set the number of flows $K=2$ for this experiment to build a more expressive proposal distribution, as the dimension of observations in this experiment is higher than that of the multivariate Gaussian experiment. Given a particle $x_{t}^i$ at time step $t$, new particles $x_{t+1}^i$ from the proposal distribution at time step $t+1$ are generated by:
\begin{align} 
    &\hat{x}_{t+1}^i = x_{t}^i + a_t + \varepsilon_t^i \sim g(\hat{x}_{t+1}|x_{t}; \theta)\\
    &x_{t+1}^i = \mathcal{F}_\phi(\hat{x}_{t+1}^i;y_{t+1})\\ 
    &= \mathcal{F}_{\phi_K}(\cdots\mathcal{F}_{\phi_2}(\mathcal{F}_{\phi_1}(\hat{x}_{t+1}^i;y_{t+1});y_{t+1})\cdots;y_{t+1})\,.
\end{align}
The Adam optimizer~\cite{kingma2015adam} with a learning rate of 0.001 is used in this experiment to minimize the loss function $\mathcal{L}(\theta,\phi)$ defined in Eq.~\eqref{eq:loss_disk}. 

\subsection{Robot Localization in Maze Environments}
The dynamic model, measurement model, and proposal distribution of NF-DPFs in this experiment are constructed with normalizing flows. In particular, one can draw samples from the dynamic model of NF-DPFs by first drawing samples $\dot{x}_{t+1}^i$ with Eq.~\eqref{eq:robot_dynamic} and then applying a normalizing flow $\mathcal{T}_\theta(\cdot): \mathbb{R}^{d_\mathcal{X}}\rightarrow \mathbb{R}^{d_\mathcal{X}}$ to $\dot{x}_{t+1}^i$:
\begin{gather}{x}_{t+1}^i=\mathcal{T}_\theta(\dot{x}_{t+1}^i)\sim p({x}_{t+1}|x_{t};\theta)\,,
\end{gather}
where $\mathcal{T}_\theta(\cdot)$ is defined as a composition of $K$ Real-NVP models defined in Eqs.~\eqref{eq:appendices_coupling_layers_1}-\eqref{eq:appendices_coupling_layers_4}. 

Given a particle $x_{t}^i$ at time step $t$, at time step $t+1$, new particles $x_{t+1}^i$ from the proposal distribution of the NF-DPF are generated by first drawing samples $\hat{x}_{t+1}^i$ with Eq.~\eqref{eq:robot_dynamic} and then applying a conditional normalizing flow $\mathcal{F}_\phi(\cdot;\cdot):\mathbb{R}^{d_\mathcal{X}+d_\mathcal{Y}}\rightarrow \mathbb{R}^{d_\mathcal{X}}$:
\begin{gather}
    x_{t+1}^i = \mathcal{F}_\phi(\hat{x}_{t+1}^i;y_{t+1})\sim q( {x}_{t+1}| {x}_{t},  {y}_{t+1};\phi)\,,
\end{gather}
where $\mathcal{F}_\phi(\cdot;\cdot)$ is defined as a composition of $K$ conditional Real-NVP models defined in Eqs.~\eqref{eq:appendices_conditional_coupling_layers_1}-\eqref{eq:appendices_conditional_coupling_layers_4}.

For this experiment, the proposed NF-DPF approximates the conditional likelihood $p(y_t|x_t;\theta)$ of an observation $y_t$ conditioned on a latent state $x_t$ as follows:
\begin{align}
        \label{eq:appendices_obs_feature}
    {e}_t&=U_\theta({y}_t)\,,\\
    \label{eq:appendices_feature_base_2}
    z_t&=\mathcal{G}^{-1}_\theta(e_t; x_t)\,,\\
    p(y_t|x_t;\theta)&\approx p(e_t|x_t;\theta)\\
    &=p_Z(z_t)\bigg|\operatorname{det} J_{\mathcal{G}_\theta}(z_t; x_t)\bigg|^{-1}\,,
    \label{eq:appendices_measurement_cnf_2}
\end{align}
where $U_\theta(\cdot):\mathbb{R}^{\mathcal{Y}}\rightarrow\mathbb{R}^{d_e}$ is the encoder function constructed with convolutional neural networks with that transforms $32\times32$ RGB images to $32$-dimensional feature vectors, $\mathcal{G}^{-1}_\theta(\cdot):\mathbb{R}^{d_e}\times\mathcal{X}:\rightarrow \mathbb{R}^{d_e}$ is a conditional normalizing flow built by stacking two conditional Real-NVP models~\cite{winkler2019learning}, and $p_Z(\cdot)$ is the PDF of a standard Gaussian distribution.

In this experiment, we set $K=2$, i.e. the (conditional) normalizing flows are constructed by stacking two (conditional) Real-NVP models~\cite{dinh2017density,winkler2019learning}. We set the learning rate to be 0.001, and use the Adam optimizer to train DPFs by minimizing the loss function $\mathcal{L}(\theta,\phi)$ defined in Eq.~\eqref{eq:loss_disk}.

\subsection{Robot tracking in Michigan NCLT dataset}

In this experiment, the dynamics of the moving robot is modeled as follows:
\begin{align}
    x_{t+1}:&=\left[\begin{array}{c}
        l_{t+1}^{(1)} \\
        l_{t+1}^{(2)} \\
        \Delta l_{t+1}^{(1)}\\
        \Delta l_{t+1}^{(2)}
    \end{array}\right]
    =\left[\begin{array}{c}
        l_t^{(1)}+\Delta l_t^{(1)} \\
        l_t^{(2)}+\Delta l_t^{(2)}\\
        \Delta l_t^{(1)}\\
        \Delta l_t^{(2)}
    \end{array}\right]+\varsigma_t\,,
 \label{eq:nclt_dynamic}
\end{align}
where $\varsigma_t\sim\mathcal{N}(\textbf{0},\Sigma^2)$ is the dynamic noise, and $\Sigma:=\text{diag}(\sigma_l, \sigma_l, \sigma_{\Delta l}, \sigma_{\Delta l})$ with $\sigma_l=1.0$ and $\sigma_{\Delta l}=1.0$. 

The dynamic model, measurement model, and proposal
distribution of NF-DPFs in this experiment are constructed
with normalizing flows. To draw samples from the dynamic model of NF-DPFs, one can first draw samples $\dot{x}_{t+1}^i$ with Eq.~\eqref{eq:nclt_dynamic} and then apply a normalizing flow $\mathcal{T}_\theta(\cdot): \mathbb{R}^{d_\mathcal{X}}\rightarrow \mathbb{R}^{d_\mathcal{X}}$ to $\dot{x}_{t+1}^i$:
\begin{gather}{x}_{t+1}^i=\mathcal{T}_\theta(\dot{x}_{t+1}^i)\sim p({x}_{t+1}|x_{t};\theta)\,,
\end{gather}
where $\mathcal{T}_\theta(\cdot)$ is defined as a composition of $K$ Real-NVP models defined in Eqs.~\eqref{eq:appendices_coupling_layers_1}-\eqref{eq:appendices_coupling_layers_4}. 
To draw samples from the proposal distribution of the NF-DPF, we first draw samples $\hat{x}_{t+1}^i$ with Eq.~\eqref{eq:nclt_dynamic} and then apply a conditional normalizing flow $\mathcal{F}_\phi(\cdot;\cdot):\mathbb{R}^{d_\mathcal{X}+d_\mathcal{Y}}\rightarrow \mathbb{R}^{d_\mathcal{X}}$ to $\hat{x}_{t+1}^i$:
\begin{gather}
    x_{t+1}^i = \mathcal{F}_\phi(\hat{x}_{t+1}^i;y_{t+1})\sim q( {x}_{t+1}| {x}_{t},  {y}_{t+1};\phi)\,,
\end{gather}
where $\mathcal{F}_\phi(\cdot;\cdot)$ is defined as a composition of $K$ conditional Real-NVP models defined in Eqs.~\eqref{eq:appendices_conditional_coupling_layers_1}-\eqref{eq:appendices_conditional_coupling_layers_4}.

For the measurement model in the NF-DPF, the conditional likelihood $p(y_t|x_t;\theta)$ is approximated as below:
\begin{align}
        \label{eq:appendices_obs_feature_nclt}
    z_t&=\mathcal{G}^{-1}_\theta(y_t; x_t)\,,\\
    p(y_t|x_t;\theta)&\approx p(y_t|x_t;\theta)\\
    &=p_Z(z_t)\bigg|\operatorname{det} J_{\mathcal{G}_\theta}(z_t; x_t)\bigg|^{-1}\,,
    \label{eq:appendices_measurement_cnf_2_nclt}
\end{align}
where $\mathcal{G}^{-1}_\theta(\cdot):\mathbb{R}^{d_e}\times\mathcal{X}:\rightarrow \mathbb{R}^{d_e}$ is a conditional normalizing flow built by stacking two conditional Real-NVP models~\cite{winkler2019learning}, the observation $y_t$ is noisy odometry reading of the robot, and $p_Z(\cdot)$ is the PDF of a standard Gaussian distribution. 

In this experiment, we set $K=2$, i.e. the (conditional) normalizing flows are constructed by stacking two (conditional) Real-NVP models~\cite{dinh2017density,winkler2019learning}. We set the learning rate to be 0.001, and use the Adam optimizer to train DPFs by minimizing the RMSE loss function $\mathcal{L}_{\textnormal{RMSE}}(\theta,\phi)$ defined in Eq.~\eqref{eq:loss_rmse}.
	
	\vfill
        
    \end{appendices}	
\end{document}